\documentclass[11pt,reqno]{amsart}

\RequirePackage{fix-cm}

\usepackage[foot]{amsaddr}

\usepackage{etoolbox}

\usepackage[english]{babel}

\usepackage{pifont}
\usepackage{subcaption}
\usepackage{environ}
\usepackage{tikz}
\usepackage{fullpage}
\usepackage{amsfonts,epsfig,graphicx}
\usepackage{amsmath,amssymb,amsthm}
\usepackage{graphics}
\usepackage{lipsum}
\usepackage{color}
\usepackage[ruled]{algorithm2e}
\usepackage{epstopdf}
\usepackage{algorithmic}
\usepackage{enumerate}
\usepackage{epigraph}
\usepackage[bookmarks=true,colorlinks,citecolor=black,linkcolor=red,urlcolor=black,backref]{hyperref}
\usepackage[utf8]{inputenc}
\usepackage{enumitem}
\setlist[enumerate]{itemsep=2pt,parsep=2pt,before={\parskip=2pt}}
\usepackage{calc}

\usepackage{natbib}
\usepackage{marginnote}

\usepackage{cleveref}

\usepackage{appendix}

\newtheorem{theorem}{Theorem}
\newtheorem{lemma}[theorem]{Lemma}
\newtheorem{corollary}[theorem]{Corollary}
\theoremstyle{definition}
\newtheorem{example}{Example}

\newtheorem*{i:cor:fixedstepsize}{Corollary~\ref{cor:fixedstepsize}}
\newtheorem*{i:thm:cogen}{Theorem~\ref{thm:cogen}}

\newcommand{\mz}{\phantom{-}0}
\newcommand{\mo}{\phantom{-}1}
\newcommand{\mf}{\hspace{7pt}\frac{1}{4}}
\newcommand{\mn}{-1}
\newcommand{\bb}[1]{\mathbb{#1}}

\DeclareMathOperator{\ralpha}{\mathnormal{\alpha_m / \alpha^{\perp}_{m}}}
\DeclareMathOperator*{\E}{\mathbb{E}}
\DeclareMathOperator*{\V}{Var}

\DeclareMathOperator*{\gap}{gap}
\DeclareMathOperator*{\stab}{stab}

\newcommand{\thoughtsep}{\phantom{M}}
\newcommand{\newbold}[1]{\phantom{M}\\ \noindent {\bf {#1}}}

\newcommand{\boldcenter}[1]{
    \phantom{M}\\ 
    \begin{center}
    \parbox{0.9\textwidth}{{\bf #1}}
    \end{center}
    \phantom{M}\\
}

\begin{document}

\citeindextrue

\title{On the Generalization Mystery in Deep Learning}
\begin{center}
\end{center}

\author{Satrajit Chatterjee and Piotr Zielinski}
\address{Email: satrajit@gmail.com, zielinski@google.com}

\maketitle

\begin{abstract}
The generalization mystery in deep learning is the following: Why do over-parameterized neural networks trained with gradient descent (GD) generalize well on real datasets even though they are capable of fitting random datasets of comparable size? Furthermore, from among all solutions that fit the training data, how does GD find one that generalizes well (when such a well-generalizing solution exists)?

We argue that the answer to both questions lies in the interaction of the gradients of different examples during training. Intuitively, if the per-example gradient vectors are well-aligned, that is, if they are {\em coherent}, then one may expect GD to be (algorithmically) stable, and hence generalize well. We formalize this argument with an easy to compute and interpretable metric for coherence, and show that the metric takes on very different values on real and random datasets for several common vision networks. The theory also explains a number of other phenomena in deep learning, such as why some examples are reliably learned earlier than others, why early stopping works, and why it is possible to learn from noisy labels. Moreover, since the theory provides a causal explanation of how GD finds a well-generalizing solution when one exists, it motivates a class of simple modifications to GD based on robust averaging of per-example gradients that attenuate memorization and improve generalization.

Generalization in deep learning is an extremely broad phenomenon, and therefore, it requires an equally general explanation. We conclude with a survey of alternative lines of attack on this problem, and argue that the proposed approach is the most viable one on this basis.
\end{abstract}

\setcounter{tocdepth}{1} 
\tableofcontents

\newcommand{\cmark}{\ding{51}}
\newcommand{\xmark}{\ding{55}}
\tikzstyle{place}=[circle,draw=blue!50,fill=blue!20,thick, inner sep=0pt,minimum size=10mm]
\tikzstyle{weight}=[circle,fill=white,inner sep=0pt,minimum size=10mm]
\tikzstyle{post}=[->,shorten <=1pt,>=stealth,very thick]
\pgfmathsetmacro{\xsp}{1.8}
\pgfmathsetmacro{\ysp}{4}

\addtocontents{toc}{\phantom{M}}
\section{Introduction}
\label{sec:intro}

In spite of the tremendous practical success of deep learning, we do not yet have a good understanding of why it works. Deep neural networks used in practice are over-parameterized, that is, they have many more parameters than the number of examples that are used to train them, and conventional wisdom holds that such over-parameterized models should not generalize well,
but yet they do.%
\footnote{No less an authority than von Neumann is reported to have said, ``With four parameters I can fit an elephant, and with five I can make him wiggle his trunk''~\citep{Mayer10}.}

Although this gap in our understanding has been long known (see, for example, \citet{Bartlett96} and \citet{Neyshabur14}), the problem was sharpened in 
an influential paper by \cite{Zhang17} who showed that typical neural networks trained with stochastic gradient descent, which is the usual training method, can easily memorize a random dataset of the same size as the (real) dataset that they were designed for. 
They argued that this simple experimental observation poses a challenge to all known explanations of generalization in deep learning, and called for a ``rethinking" of our approach.
This led to a large effort in the community to better understand why neural networks generalize. However, although our understanding of deep learning has greatly improved as a result of this effort, to date, there does not appear to be an satisfactory explanation (see \citet{Zhang21} for a detailed review).

Our goal in this work is to answer the 
broad question raised by the observations of \citeauthor{Zhang17}, namely,
\boldcenter{
Why do neural networks generalize well in practice when they have sufficient capacity to memorize their training set?
}
Specifically, we want to answer the following questions:

\begin{enumerate}

\item [{\bf Q1.}] Since by simply changing the dataset in an over-parameterized setting (for example, from real labels to random labels), we can obtain very different generalization performance, what property of the dataset controls the generalization gap (assuming, of course, that architecture, learning rate, size of training set, etc. are held fixed)? We stress that our interest is in the {\em gap}, that is the difference between training and test loss (and not in the test loss {\em per se}).

\item [{\bf Q2.}] Why does gradient descent not simply memorize real training data as it does random training data? That is, from among all the models that fit the training set perfectly in an over-parameterized setting, how does gradient descent find one that generalizes well to unseen data when such a model exists? This property is often called the {\em implicit bias} of gradient descent.%
\footnote{
It is an implicit bias because even in the absence of explicit regularizers (such as weight decay or drop out), 
gradient descent still generalizes well (where possible).}

\end{enumerate}

Most previous attempts to make progress on these questions have been based on the 
notion of uniform convergence~\citep{Vapnik71}, the primary theoretical tool in classical learning theory.
However, a recent paper by \citet{Nagarajan19b} provides a strong argument for why any method based on uniform convergence is unlikely to provide an explanation to the mystery.

In this work, we study the generalization mystery from the less-explored, but arguably more general, perspective of algorithmic stability~\citep{Devroye79, Bousquet02, Hardt16}. We propose that the answer to both the questions lies in the interaction of the gradients of the examples during training. 
(The interaction of per-example gradients has not received much attention in the literature, with notable exceptions being \citet{Yin18, Fort19, Sankararaman20, Liu20, He20, Mehta21} which we discuss later.) 
Specifically, we argue,

\begin{enumerate}

\item [{\bf A1.}] {\bf Gradient descent in an over-parameterized setting generalizes well when the gradients of different examples (during training) are similar, that is, when there is {\em coherence}.}

\item [{\bf A2.}] {\bf When there is coherence, the dynamics of gradient descent leads to models that are {\em stable}, that is, to models that do not depend too much on any one training example, and, as is well known, stable models generalize well.}

\end{enumerate}

\noindent In this paper, we advance a theory along these lines, which we call the theory of {\bf Coherent Gradients}.%
\footnote{Some preliminary results appeared in \citet{Chatterjee20, Zielinski20, Chatterjee20c}.}

\section{The Theory, Informally}
\label{sec:overview:informal}

The motivation for our theory comes from the observation that the ability to memorize random datasets, and yet generalize well on real datasets is not unique to deep neural networks, but also seen, for example, with decision trees (and random forests). But there is no generalization mystery there: a typical tree construction algorithm splits the training set recursively into similar subsets based on input features. If no similarity is found, eventually, each example is put into its own leaf to achieve good training accuracy, but, of course, at the cost of poor generalization. Thus, trees that achieve good training accuracy on a randomized dataset are larger than those on a real dataset (for example, see Expt. 5 in ~\citep{Chatterjee20b}).

We propose that something similar happens with neural networks trained with gradient descent (GD):
\boldcenter{
Gradient descent exploits patterns common across training examples during the fitting process, and if there are no common patterns to exploit, then the examples are fitted on a ``case-by-case" basis.
}
Intuitively, this provides a {\em uniform} explanation of memorization and generalization: If a dataset is such that examples are fitted on a case-by-case basis, then we expect poor generalization (it corresponds to memorization), whereas, if there are common patterns that can be exploited to fit the data, then we should expect good generalization.

So how does gradient descent exploit common patterns to fit the training data (when such common patterns exist)? Since the only interaction between examples in gradient descent is in the parameter update step, the mechanism for commonality exploitation---if it exists---must be there.
Let $\eta$ be the learning rate and let $g_i(w_t)$ be the gradient for the $i$th training example at the point $w_t$ in parameter space. Furthermore, for now, assume that the $g_i(w_t)$ all have the same scale (we discuss the consequences of relaxing this shortly).
Now, consider the parameter update at step $t$ of gradient descent:%
\footnote{For simplicity of exposition, here we only consider the full-batch case and no batch normalization. For the stochastic case, we have to consider the expected gradient, but the argument carries over, as we shall see shortly in the formal development.}
\[
w_{t + 1} \equiv w_t - \eta\,\frac{1}{m} \sum_{i=1}^{m} g_i(w_t)
\]
If $g(w_t)$ denotes the average gradient, that is, $g(w_t) \equiv \frac{1}{m} \sum_{i = 1}^{m} g_i(w_t)$, we observe the following:

\begin{enumerate}

\item  The average gradient $g(w_t)$ is stronger in directions (components) where the per-example gradients are ``similar," and reinforce each other; and weaker in other directions where they are different, and do not add up (or perhaps cancel each other).

\item Since network parameters $w_{t+1}$ are updated proportionally to gradients, those parameters that correspond to stronger gradient directions change more.

\end{enumerate}

\noindent Therefore, the parameter changes of the network are biased towards those that benefit multiple examples.

When per-example gradients reinforce each other, we say they are ``coherent," and use the term {\em coherence} to informally refer to the similarity of per-example gradients (either in aggregate across entire gradients or across specific components of the gradients). 

Note that it is possible that there are {\em no} directions or components where different per-example gradients add up, that is, there is no coherence. That case would correspond to fitting each example independently, that is, on a case-by-case basis. In other words, reducing loss on one example fails to reduce the loss on any other example. Intuitively, one might expect this to be the case for random datasets, and only possible when the learning problem is over-parameterized: the dimensionality of the tangent space (that is, the dimension of the gradient vectors) is greater than the number of training examples.

\thoughtsep

We can reason about the generalization of the above process through the notion of algorithmic stability~\citep{Bousquet02, Devroye79}. 
A learning process is {\em stable} if a replacing one example in the training sample by another example (from the example distribution) does not change the learned model too much. It can be shown that models obtained through stable processes generalize well. 

Strong directions in the average gradient are stable since in those directions multiple examples support or reinforce each other.
In particular, the absence or presence of a single example in the training set does not impact descent down a strong direction since other examples contribute to it anyway. Therefore, by stability theory, the corresponding parameter updates should generalize well, that is, they would lead to lower loss on unseen examples as well.

On the other hand, weak directions in the average gradient are unstable, since they are due to a few or even single examples.
In the latter case, for example, the absence of the corresponding example in the training set would prevent descent down that direction, and therefore, the corresponding parameter updates would not generalize well. 

With this observation, we can reason inductively about the stability of GD: since the initial values
of the parameters do not depend on the training data, the initial function mapping examples to their
gradients is stable. Now, if all parameter updates are due to strong gradient directions, then stability
is preserved. However, if some parameter updates are due to weak gradient directions, then stability
is diminished. 
Now, since stability (suitably formalized) is equivalent to generalization~\citep{ShalevShwartz10}, this allows us to see how generalization may degrade as training progresses.%
\footnote{Based on this insight, we shall see later how a simple modification to GD to suppress 
the weak gradient directions can dramatically reduce overfitting.} %
In summary:
\boldcenter{
When there is coherence, the dynamics of gradient descent, particularly the use of the average gradient, leads to models that are stable, that is, to models that do not depend too much on any one training example.
}

Of course, in general, there can be {\em no} dataset-independent guarantee of stability.
We can understand the connection between coherence of the dataset and stability intuitively, by considering two extreme cases. If, on the one hand, the gradients of all the examples are pointing in the same direction (that is, we have perfect coherence), then the replacing one training example by another does not matter, and the loss on the replaced example should still decrease, even though it was not used for training. 
On the other hand, if the gradients of all the examples are all orthogonal to each other (we have low coherence), then replacing an example with another eliminates the descent down the gradient direction of the replaced example, and we should not expect loss to reduce on that example. In other words, in the presence of high coherence, training on the original training set, and a perturbed version should not differ too much.

\thoughtsep

This insight suggests a new modification to gradient descent to improve generalization, and also explains some existing regularization techniques:

\begin{itemize}

    \item We can make gradient descent more stable by eliminating or attenuating weak directions by combining per-example gradients using robust mean estimation techniques (such as median) instead of simple averages.
    
    \item We can view $\ell^2$ regularization (weight decay) as an attempt to eliminate movement in weak gradient directions by having a ``background force" that pushes each parameter to a data-independent default value (zero). It is only in the case of strong directions that this background force is overcome, and parameters updated in a data-dependent manner.
    
    \item Earlier we assumed that all the $g_i(w_t)$ all have the same scale. But that is not always true. For example, during training, some examples get fitted earlier than others, and so their gradients become negligible. Now, fewer examples dominate the average gradient $g(w_t)$, and this leads to overfitting since stability is degraded. This may be viewed as a justification for early stopping as a regularizer.

\end{itemize}

\thoughtsep

The properties of gradient descent as an optimizer do not play an important role in our generalization theory.
We model gradient descent as a simple, ``combinatorial" optimizer, a kind of greedy search with some hill-climbing thrown in (due to sampling and finite step size).
Therefore, we are concerned less about the quality of solutions reached at the end of gradient descent, but more about staying ``feasible" at all times during the search. In our context, feasibility means being able to generalize; and this naturally leads us to look at the transition dynamics of gradient descent to see if that preserves generalizability.

A notable feature of this approach is that it does not rely on any special property of the final solution obtained by gradient descent. 
Therefore, (1) it allows us to reason in an {\em uniform} manner about generalization with respect to early stopping and generalization at convergence,%
\footnote{Since stopping early typically leads to smaller generalization gap, the nature of the solutions of GD (for example, stationary points, the limit cycles of SGD at equilibrium) cannot be the ultimate explanation for generalization. In fact, think of the extreme case where no GD step is taken, and we have zero generalization gap!} and (2) it applies {\em uniformly} to convex and non-convex problems (and uniformly across architectures without requiring simplifying assumptions such as homogeneous activations, infinite widths, etc.).

\section{An Illustrative Example}
\label{sec:overview:example}

The main ideas of the theory described above can be illustrated with a simple thought experiment where we fit an over-parameterized linear model to two toy datasets, one of which is ``real" and the other is ``random." 

\begin{example}[``Real" and ``Random"]
\label{ex:two-datasets}

Consider the task of fitting the linear model
\[
y = w \cdot x = \sum_{i = 1}^{6} w^{(i)} x^{(i)}
\]
under the usual square loss  $\ell(w) \equiv \frac{1}{2} (y - w \cdot x)^2$ to each of the following datasets:
\begin{center}

\phantom{M}

\begin{tabular}{r|c|c}
\multicolumn{3}{c}{{\bf L} (``real")} \\
$i$ & $x_i$ & $y_i$ \\
\hline
\hline
$1$ & $\langle\,\textcolor{pink}{1}, \textcolor{pink}{\mz}, \textcolor{pink}{\mz}, \textcolor{pink}{\mz}, \textcolor{pink}{\mz}, \textcolor{blue}{\mo}\,\rangle$ & $\mo$ \\
$2$ & $\langle\,\textcolor{pink}{0}, \textcolor{pink}{\mn}, \textcolor{pink}{\mz}, \textcolor{pink}{\mz}, \textcolor{pink}{\mz}, \textcolor{blue}{\mn}\,\rangle$ & $\mn$ \\
$3$ & $\langle\,\textcolor{pink}{0}, \textcolor{pink}{\mz}, \textcolor{pink}{\mn}, \textcolor{pink}{\mz}, \textcolor{pink}{\mz}, \textcolor{blue}{\mn}\,\rangle$ & $\mn$ \\
$4$ & $\langle\,\textcolor{pink}{0}, \textcolor{pink}{\mz}, \textcolor{pink}{\mz}, \textcolor{pink}{\mo}, \textcolor{pink}{\mz}, \textcolor{blue}{\mo}\,\rangle$ & $\mo$ \\
\hline
$5$ & $\langle\,\textcolor{pink}{0}, \textcolor{pink}{\mz}, \textcolor{pink}{\mz}, \textcolor{pink}{\mz}, \textcolor{pink}{\mn}, \textcolor{blue}{\mn}\,\rangle$ & $\mn$ \\
\end{tabular}
\quad \quad \quad
\begin{tabular}{r|c|c}
\multicolumn{3}{c}{{\bf M} (``random")} \\
$i$ & $x_i$ & $y_i$ \\
\hline
\hline
$1$ & $\langle\,\textcolor{pink}{1}, \textcolor{pink}{\mz}, \textcolor{pink}{\mz}, \textcolor{pink}{\mz}, \textcolor{pink}{\mz}, \mz\,\rangle$ & $\mo$ \\
$2$ & $\langle\,\textcolor{pink}{0}, \textcolor{pink}{\mn}, \textcolor{pink}{\mz}, \textcolor{pink}{\mz}, \textcolor{pink}{\mz}, \mz\,\rangle$ & $\mn$ \\
$3$ & $\langle\,\textcolor{pink}{0}, \textcolor{pink}{\mz}, \textcolor{pink}{\mn}, \textcolor{pink}{\mz}, \textcolor{pink}{\mz}, \mz\,\rangle$ & $\mn$ \\
$4$ & $\langle\,\textcolor{pink}{0}, \textcolor{pink}{\mz}, \textcolor{pink}{\mz}, \textcolor{pink}{\mo}, \textcolor{pink}{\mz}, \mz\,\rangle$ & $\mo$ \\
\hline
$5$ & $\langle\,\textcolor{pink}{0}, \textcolor{pink}{\mz}, \textcolor{pink}{\mz}, \textcolor{pink}{\mz}, \textcolor{pink}{\mn}, \mz\,\rangle$ & $\mn$ \\
\end{tabular}

\phantom{M}

\end{center}

\noindent In each dataset, the first four examples, that is, $(x_i, y_i)$ for $i \in [4]$ are used for training whereas the fifth example $(x_5, y_5)$ is held out for evaluating generalization. 
Observe that the first 5 input features (shown in \textcolor{pink}{pink}) are the same in both {\bf L} and {\bf M}, and furthermore, that they are each only predictive for one example (they are ``idiosyncratic"). 
The last feature however is different in the two datasets. In {\bf L} (shown in \textcolor{blue}{blue}), it is predictive across all examples (``common" feature), whereas in {\bf M} (shown in black), it is uninformative.
We can think of {\bf L} as a ``real'' dataset where there is something that can be learned by a linear model, whereas {\bf M} is a ``random'' dataset that may at best be memorized.

Now, if we apply gradient descent (GD) to these problems (starting from $w = 0$), we see the following training and test curves:

\begin{center}
\begin{tabular}{cc}
{\bf L} & {\bf M} \\
\includegraphics[width=0.45\textwidth]{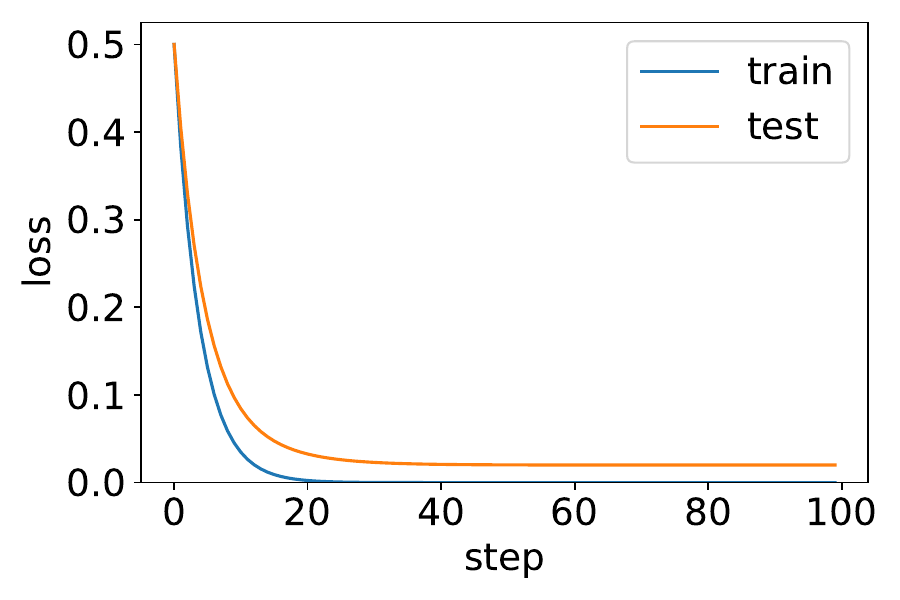} &
\includegraphics[width=0.45\textwidth]{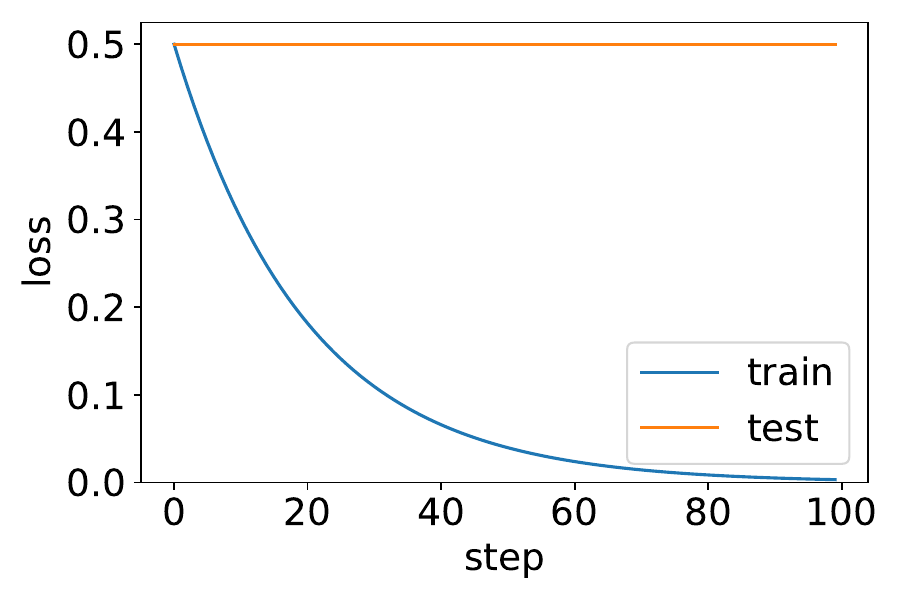} \\
\end{tabular}
\end{center}

\noindent In this simple setup, we see the same phenomena as in deep learning: (1) the model has sufficient capacity to memorize random data ({\bf M}), yet it generalizes on real data ({\bf L}), that is, the generalization is dataset dependent; and (2) from among all the models that fit {\bf L} (which includes the solution for {\bf M}), GD finds one that generalizes well (instead of just memorizing). 
However, in contrast to deep models, it is much easier to understand intuitively what is going on in this setup, particularly, from the point of view of the dynamics of GD.

First, consider the per-example gradients $g_i(w)$ and the average training gradient $g(w)$ at a point $w$ on the trajectory of gradient descent starting from 0:

\begin{center}

\phantom{M}

\begin{tabular}{c|c}
\multicolumn{2}{c}{{\bf L}} \\
$i$ & $g_i(w)$ \\
\hline
\hline
$1$ & $r(w)\, \langle\,\textcolor{pink}{1}, \textcolor{pink}{\mz}, \textcolor{pink}{\mz}, \textcolor{pink}{\mz}, \textcolor{pink}{\mz}, \textcolor{blue}{\mo}\,\rangle$ \\
$2$ & $r(w)\, \langle\,\textcolor{pink}{0}, \textcolor{pink}{\mo}, \textcolor{pink}{\mz}, \textcolor{pink}{\mz}, \textcolor{pink}{\mz}, \textcolor{blue}{\mo}\,\rangle$ \\
$3$ & $r(w)\, \langle\,\textcolor{pink}{0}, \textcolor{pink}{\mz}, \textcolor{pink}{\mo}, \textcolor{pink}{\mz}, \textcolor{pink}{\mz}, \textcolor{blue}{\mo}\,\rangle$ \\
$4$ & $r(w)\, \langle\,\textcolor{pink}{0}, \textcolor{pink}{\mz}, \textcolor{pink}{\mz}, \textcolor{pink}{\mo} , \textcolor{pink}{\mz}, \textcolor{blue}{\mo}\,\rangle$ \\
\hline
\hline
\rule{0pt}{2.5ex} $g(w)$ & $r(w)\, \langle\,\textcolor{pink}{\frac{1}{4}}, \textcolor{pink}{\mf}, \textcolor{pink}{\mf}, \textcolor{pink}{\mf}, \textcolor{pink}{\mz}, \textcolor{blue}{\mo}\,\rangle$ \\
\end{tabular}
\quad \quad \quad
\begin{tabular}{c|c}
\multicolumn{2}{c}{{\bf M}} \\
$i$ & $g_i(w)$ \\
\hline
\hline
$1$ & $r(w)\, \langle\,\textcolor{pink}{1}, \textcolor{pink}{\mz}, \textcolor{pink}{\mz}, \textcolor{pink}{\mz}, \textcolor{pink}{\mz}, \textcolor{black}{\mz}\,\rangle$ \\
$2$ & $r(w)\, \langle\,\textcolor{pink}{0}, \textcolor{pink}{\mo}, \textcolor{pink}{\mz}, \textcolor{pink}{\mz}, \textcolor{pink}{\mz}, \textcolor{black}{\mz}\,\rangle$ \\
$3$ & $r(w)\, \langle\,\textcolor{pink}{0}, \textcolor{pink}{\mz}, \textcolor{pink}{\mo}, \textcolor{pink}{\mz}, \textcolor{pink}{\mz}, \textcolor{black}{\mz}\,\rangle$ \\
$4$ & $r(w)\, \langle\,\textcolor{pink}{0}, \textcolor{pink}{\mz}, \textcolor{pink}{\mz}, \textcolor{pink}{\mo}, \textcolor{pink}{\mz}, \textcolor{black}{\mz}\,\rangle$ \\
\hline
\hline
\rule{0pt}{2.5ex} $g(w)$ & $r(w)\, \langle\,\textcolor{pink}{\frac{1}{4}}, \textcolor{pink}{\mf}, \textcolor{pink}{\mf}, \textcolor{pink}{\mf}, \textcolor{pink}{\mz}, \textcolor{black}{\mz}\,\rangle$ \\
\end{tabular}

\phantom{M}

\end{center}

\noindent where $r(w)$ is a scalar.%
\footnote{$r(w)$ is equal to $|y_i - w \cdot x_i|$ for $i \in [4]$, a quantity that is independent of $i$ for $w$ in the trajectory of GD due to the symmetries in this problem.}
Observe that {\bf M} lacks coherence: the per-example gradients do not reinforce each other, that is, they do not add up in any component. Therefore, there are no strong directions in the average gradient.
In contrast, in {\bf L}, the per-example gradients are coherent: they share a {\em common} component (shown in \textcolor{blue}{blue}) which
``adds up" in the average gradient $g(w)$, which consequently, is {\bf 4 times} stronger in that component than in the other (idiosyncratic) components (shown in \textcolor{pink}{pink}). 

Finally, since parameter changes in GD are proportional to the average gradient $g(w)$, 
the parameter $w^{(6)}$ corresponding to the common input feature in {\bf L} changes much more than any of the ones for the idiosyncratic features (for example, $w^{(1)}$)%
\footnote{From symmetry, $w^{(2)}, w^{(3)},$ and $w^{(4)}$ are the same as $w^{(1)}$.} during training. In contrast, in {\bf M}, only the weights for the idiosyncratic features are used to fit the data. This is confirmed by looking at the trajectories of $w^{(1)}$ and $w^{(6)}$ during training:

\begin{center}
\begin{tabular}{cc}
{\bf L} & {\bf M} \\
\includegraphics[width=0.45\textwidth]{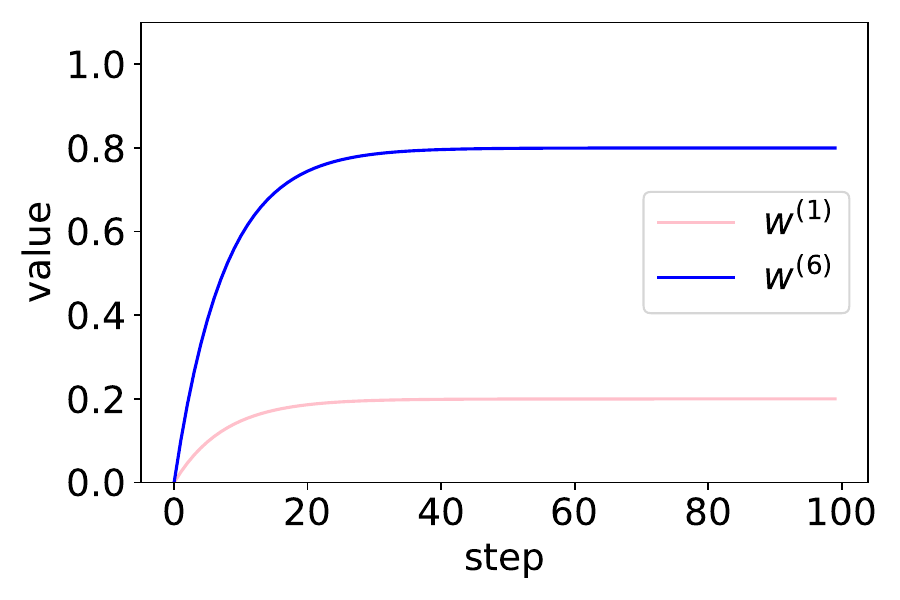} &
\includegraphics[width=0.45\textwidth]{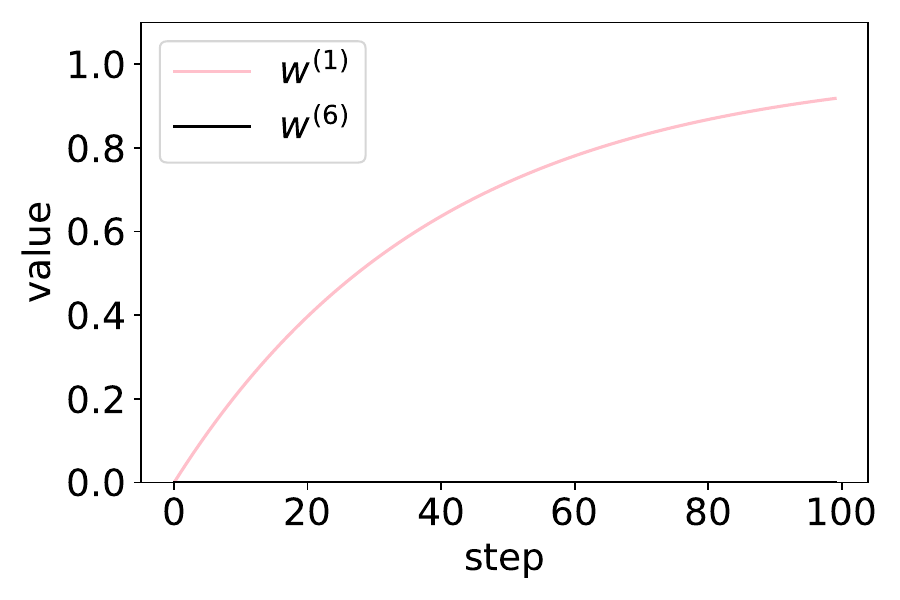} \\
\end{tabular}
\end{center}

\noindent To summarize, in this simple example of ``real" and ``random" datasets that replicates the generalization mystery of deep learning, we see that:

\begin{enumerate}

    \item The difference in the generalization gap between the two datasets can be understood in terms of the difference in the similarity of the per-example gradients, that is, in terms of the difference in coherence.
    
    \item In the case of the ``random" dataset, the training data was fit on a case-by-case basis (that is, memorized). Although the same would also have been an optimal solution for the problem of fitting the ``real" training data, gradient descent produced a different solution by exploiting the commonality between training examples, as expressed in their gradients.

\end{enumerate}

\thoughtsep

Finally, let us consider a modification to gradient descent where instead of simply averaging the per-example gradients, we take the component-wise {\bf median} gradient $\widetilde{g}(w)$:
\begin{center}
\phantom{M}
\begin{tabular}{c|c}
\multicolumn{2}{c}{{\bf L}} \\
$i$ & $g_i(w)$ \\
\hline
\hline
$1$ & $r(w)\, \langle\,\textcolor{pink}{1}, \textcolor{pink}{\mz}, \textcolor{pink}{\mz}, \textcolor{pink}{\mz}, \textcolor{pink}{\mz}, \textcolor{blue}{\mo}\,\rangle$ \\
$2$ & $r(w)\, \langle\,\textcolor{pink}{0}, \textcolor{pink}{\mo}, \textcolor{pink}{\mz}, \textcolor{pink}{\mz}, \textcolor{pink}{\mz}, \textcolor{blue}{\mo}\,\rangle$ \\
$3$ & $r(w)\, \langle\,\textcolor{pink}{0}, \textcolor{pink}{\mz}, \textcolor{pink}{\mo}, \textcolor{pink}{\mz}, \textcolor{pink}{\mz}, \textcolor{blue}{\mo}\,\rangle$ \\
$4$ & $r(w)\, \langle\,\textcolor{pink}{0}, \textcolor{pink}{\mz}, \textcolor{pink}{\mz}, \textcolor{pink}{\mo} , \textcolor{pink}{\mz}, \textcolor{blue}{\mo}\,\rangle$ \\
\hline
\hline
$\widetilde{g}(w)$ & $r(w)\, \langle\,\textcolor{pink}{0}, \textcolor{pink}{\mz}, \textcolor{pink}{\mz}, \textcolor{pink}{\mz}, \textcolor{pink}{\mz}, \textcolor{blue}{\mo}\,\rangle$ \\
\end{tabular}
\quad \quad \quad
\begin{tabular}{c|c}
\multicolumn{2}{c}{{\bf M}} \\
$i$ & $g_i(w)$ \\
\hline
\hline
$1$ & $r(w)\, \langle\,\textcolor{pink}{1}, \textcolor{pink}{\mz}, \textcolor{pink}{\mz}, \textcolor{pink}{\mz}, \textcolor{pink}{\mz}, \textcolor{black}{\mz}\,\rangle$ \\
$2$ & $r(w)\, \langle\,\textcolor{pink}{0}, \textcolor{pink}{\mo}, \textcolor{pink}{\mz}, \textcolor{pink}{\mz}, \textcolor{pink}{\mz}, \textcolor{black}{\mz}\,\rangle$ \\
$3$ & $r(w)\, \langle\,\textcolor{pink}{0}, \textcolor{pink}{\mz}, \textcolor{pink}{\mo}, \textcolor{pink}{\mz}, \textcolor{pink}{\mz}, \textcolor{black}{\mz}\,\rangle$ \\
$4$ & $r(w)\, \langle\,\textcolor{pink}{0}, \textcolor{pink}{\mz}, \textcolor{pink}{\mz}, \textcolor{pink}{\mo}, \textcolor{pink}{\mz}, \textcolor{black}{\mz}\,\rangle$ \\
\hline
\hline
$\widetilde{g}(w)$ & $r(w)\, \langle\,\textcolor{pink}{0}, \textcolor{pink}{\mz}, \textcolor{pink}{\mz}, \textcolor{pink}{\mz}, \textcolor{pink}{\mz}, \textcolor{black}{\mz}\,\rangle$ \\
\end{tabular}
\phantom{M}
\phantom{M}
\end{center}
It is easy to see that performing gradient descent with the median gradient $\widetilde{g}(w)$ increases the stability of gradient descent on both datasets by eliminating the weak gradient directions, and leads to zero generalization gap for both ``real" and ``random".\footnote{Note that in the case of ``random," this reduction in the gap is achieved by preventing the training set from being memorized.}

\qed
\end{example}

\thoughtsep

The conventional explanation of generalization in over-parameterized linear models depends on the observation that GD from {\bf 0} finds the solution with minimum $\ell^2$ norm (that is, the solution found by the pseudo-inverse method). 
In comparison, the explanation above, based on the alignment of per-example gradients and stability, is simpler and more direct. Crucially, it generalizes in a straightforward manner to deep networks where GD does not always find the minimum $\ell^2$ norm solution.
Thus, at a fundamental level, the Coherent Gradients approach allows us to decouple optimization and generalization; so even if we cannot say anything about the result of the optimization process after GD, we can say something about whether the solution generalizes or not by analyzing the per-example gradients along the way.

\section{Metrics to Quantify Coherence}
\label{sec:overview:metrics} 

So far we have argued, informally, that if the gradients of different training examples are similar and reinforce each other (that is, if they are coherent), then the model produced by gradient descent is expected to generalize well; and we have illustrated this argument with a simple thought experiment. 
But in order to test this explanation in practical settings, such as the experiments of \citet{Zhang17}, we need to quantify the notion of coherence.

An obvious metric to quantify the coherence of the per-example gradients is their average pairwise dot product. Since this has a nice connection to the loss function, we start by reviewing the connection, and also set up some notation in the process.
Formally, let $\mathcal{D}(z)$ denote the distribution
of examples $z$ from a finite set $Z$.%
\footnote{We assume finiteness for mathematical simplicity  since it does not affect generality for practical applications.} 
For a network with $d$ trainable parameters, 
let $\ell_z(w)$ be the loss for an example $z \sim \mathcal{D}$ for a parameter vector $w \in \mathbb{R}^d$. For the learning problem, we are interested in minimizing the expected loss $\ell(w) \equiv \E_{z \sim \mathcal{D}}[\ell_z(w)]$.
Let $g_z \equiv [\nabla \ell_z](w)$ denote the gradient of the loss on example $z$, and $g \equiv [\nabla \ell](w)$ denote the average gradient. From linearity, we have,
\[
g = \E_{z \sim \mathcal{D}}\ [\ g_z\ ]
\]
Now, suppose we take a small descent step $h = - \eta\,g$ (where $\eta > 0$ is the learning rate). From the Taylor expansion of $\ell$ around $w$, we have, 
\begin{equation}
\ell(w + h) - \ell(w) \approx g \cdot h 
= - \eta\  g \cdot g
= - \eta \E_{z \sim \mathcal{D}}\ [\ g_z\ ] \cdot \E_{z \sim \mathcal{D}}\ [\ g_z\ ]
= - \eta \E_{z \sim \mathcal{D}, z' \sim \mathcal{D}}\ [g_z \cdot g_{z'}]
\label{eq:overview:dloss}
\end{equation}
where the last equality can be checked with a direct computation.
Thus, the expected pairwise dot product is equal to 
the change in loss divided by the learning rate, that is, 
the ``instantaneous" change in loss.

\begin{example}[Perfect similarity v/s pairwise orthogonal]

Consider a sample with $m$ examples $z_i$ where $1 \leq i \leq m$. Let $g_i$ be the gradient of $z_i$ and further that $\|g_i\| = \|u\|$ for some $u$. Let $1 \leq j \leq m$.
If all the $g_i$ are the same, that is, if coherence or similarity is maximum, then $\E\,[g_i \cdot g_{j}] = \|u\|^2$. However, if they are pairwise orthogonal, i.e., $g_i \cdot g_j = 0$ for $i \neq j$, then $\E\,[g_i \cdot g_j] = (1/m) \|u\|^2$.

Observe that the loss reduces $m$ times faster in the case of maximum coherence than when the gradients are pairwise orthogonal.

\qed
\end{example}

As the above example illustrates, the average expected dot product can vary significantly depending on the similarity of the gradients, and so could be used as a measure of coherence. However, since it has no natural scale---just rescaling the loss changes its value---it is difficult to interpret in an experimental setting.
For example, it is not immediately clear if say, a value of 17.5 for the expected dot product indicates good or bad coherence.

Now, there is a natural scaling factor that can be used to normalize the expected dot product of per-example gradients.
Consider the Taylor expansion of each individual loss $\ell_z$ around $w$ when we take a small step $h_z$ down {\em its} gradient $g_z$:
\[
\ell_z(w + h_z) - \ell_z(w) \approx g_z \cdot h_z 
= - \eta\  g_z \cdot g_z
\]
Taking expectations over $z$ we get,
\begin{equation}
\E_{z \sim \mathcal{D}} [ \ell_z(w + h_z) - \ell_z(w) ] 
= - \eta\ \E_{z \sim \mathcal{D}} [ g_z \cdot g_z ]
\label{eq:overview:dloss_ez}
\end{equation}
The quantity in (\ref{eq:overview:dloss_ez}) has a simple interpretation: It is the expected reduction in the per-example loss $\ell_z$ if we could take different steps for different examples. In that sense, it is an {\em idealized reduction in loss} that real gradient descent cannot usually attain.
As might be expected intuitively, it is a bound on the quantity in  (\ref{eq:overview:dloss}) and is tight when all the per-example gradients are identical.
Thus, it serves as a natural scaling factor for the expected dot product, and we obtain a normalized metric for coherence, called $\alpha$, from (\ref{eq:overview:dloss}) and (\ref{eq:overview:dloss_ez}): 
\begin{equation}
\boxed{
\quad
\ 
\alpha \equiv 
\frac{\ell(w + h) - \ell(w)}{\displaystyle \E_{z \sim \mathcal{D}}\ [ \ell_z(w + h_z) - \ell_z(w) ]} =
\frac{\displaystyle \E_{z \sim \mathcal{D}, z' \sim \mathcal{D}}\ [g_z \cdot g_{z'}]}
{\displaystyle \E_{z \sim \mathcal{D}}\ [ g_z \cdot g_z ]}
=
\frac{\displaystyle \E_{z \sim \mathcal{D}}\ [\ g_z\ ] \cdot \E_{z \sim \mathcal{D}}\ [\ g_z\ ]}
{\displaystyle \E_{z \sim \mathcal{D}}\ [ g_z \cdot g_z ]}
\quad
\ 
}
\label{eq:overview:metric}
\end{equation}

\noindent From the discussion above, it is easy to see that 
\[
0 \le \alpha \le 1.
\]
Coherence is 1 when all per-example gradients are identical, and 0 when the expected gradient is {\bf 0}, that is, when training converges.%
\footnote{If all examples are fit at the end of training, the denominator vanishes. However, in that case, the numerator is also zero, and we define the coherence to be 0. This is also consistent with adding a small positive epsilon to the denominator when computing coherence to avoid division by zero. This choice can also be justified with a continuity argument (see \Cref{app:alphafacts}).}
This is formalized as \Cref{thm:basicprop} in \Cref{app:alphafacts}.

\begin{example}[Orthogonal Sample and Orthogonal Limit] 

If we have a sample of $m$ examples whose gradients $g_i$ ($1 \leq i \leq m$) are pairwise orthogonal (an ``orthogonal"
sample), then,
\[
\alpha 
= \frac
    {(1/m)\,{\displaystyle \E_{i}\,[g_i \cdot g_i]}}
    {\E_{i}\,{\displaystyle [g_i \cdot g_i]}}
= \frac{1}{m}
\]
\noindent Thus, for an orthogonal sample, $\alpha$ is independent of the actual gradients, and we call $1/m$ the {\em orthogonal limit} for a sample of size $m$, and denote it by $\alpha^{\perp}_{m}$. 

\qed
\end{example}

In our experiments, we measure coherence on a sample, as we typically do not have access to the underlying distribution $\mathcal{D}$. 
However, as one might expect, $\alpha$ as measured from a sample is a biased estimator of the true (that is, the distributional) $\alpha$. To distinguish between the two, we use $\alpha_S$ to denote the estimate of $\alpha$ obtained from a sample $S$. Since the sample $S$ is often clear from the context, we commonly also use $\alpha_m$ instead of $\alpha_S$ where $m$ is the size of $S$ (to preserve the distinction with $\alpha$ and to remind ourselves of the sample size dependence).%
\footnote{Typically, $\alpha_S$ overestimates $\alpha$. A sample $S$ of size 1 provides an extreme example of this, since in that case, $\alpha_S$ is $1$ regardless of $\alpha$. Another example is provided by a distribution $\mathcal{D}$ where all the per-example gradients have the same norm, and it is not too hard to see that $\E_{S \sim \mathcal{D}^m} [ \alpha_S ] \ge \alpha$. This is also confirmed by experiments (for example, see \Cref{fig:m_variation} in \Cref{app:measuring-coherence}).}

In the case of $\alpha$ estimated from a sample, the coherence of a (possibly fictitious%
\footnote{If the dimension $d$ of the tangent space is less than $m$, that is, we have more examples than parameters, then an orthogonal sample of size $m$ cannot actually be constructed.}%
) orthogonal sample provides a convenient yardstick to judge the coherence of any other sample of the same size.
So given a sample of size $m$, rather than show $\alpha_m$, we show the ratio $\alpha_m / \alpha^{\perp}_{m} = m\,\alpha$. This quantity has a simple intuitive interpretation:
\boldcenter{
The metric $\alpha_m / \alpha^{\perp}_{m}$ for coherence measures how many examples on average (including itself) a given example helps at a given step of gradient descent.
}
\noindent This intuition may be justified by considering different kinds of possible training samples:

\begin{itemize}
    \item  For an orthogonal sample, $\alpha_m = 1/m$ and, therefore,  $\alpha_m / \alpha^{\perp}_{m} = 1$.
    In this case, each example in the sample is fitted independently of the others: a step in a direction of any given example's gradient does not affect the loss on any other example. Each example ``helps" only itself and does not positively or negatively affect the other examples in the sample. 

    \item For a sample with perfect coherence (all the examples are identical), $\alpha_m = 1$, and we have, $\alpha / \alpha^{\perp}_{m} = m$. 
    Here, each example in the sample ``helps" all the other examples in the sample during gradient descent.
    
    \item At the end of training, when the expected gradient of the sample is {\bf 0}, either the loss on a given example cannot be improved, or improving it comes at the expense of worsening the loss on other examples in the sample. Thus, an example ``helps" no examples (including itself) on average, and $\alpha_m = \alpha_m / \alpha^{\perp}_{m} = 0$.
    
    \item For the $d \ll m$ (under-parameterized) case, consider $\alpha_{m} / \alpha_{m}^{\perp}$ for a sample that comprises $k \gg 1$ copies of $d$ orthogonal gradients living in a $d$-dimensional tangent space (that is, $m = kd \gg d$). Now, $\alpha_{m}$ of this replicated sample is also $1 / d$ (since replicating a sample does not change the empirical distribution) and thus, its $\alpha_{m} / \alpha_{m}^{\perp}$ is $(1/d) / (1/kd) = k$, which agrees with the intuition that each example in the replicated sample helps $k$ other examples.
    
\end{itemize}

\begin{example}[``Real" and ``Random"]

In  Example~\ref{ex:two-datasets}, consider the training sample for {\bf L} ($m = 4$). It is easy to check with a direct computation that $\alpha_m = 5/8$ and $\alpha_m / \alpha^{\perp}_{m} = 2.5$, that is, each example helps 2.5 examples.
Intuitively, 2.5 may be understood as an example helping reduce the loss on itself 100\%, and reducing the loss on 3 other examples by 50\% due to the common component (the other half  comes from their own idiosyncratic components). 

The training sample for {\bf M} is pairwise orthogonal, and therefore, $\alpha_m / \alpha^{\perp}_{m} = 1$.
\qed
\end{example}

\thoughtsep

One can imagine many different metrics for coherence. We have seen two so far: the expected pairwise dot product and $\alpha$.
As we will see in the next section, another metric that arises naturally is 
\begin{equation}
\label{eq:chap1:diffmetric}
\E_{z \sim \mathcal{D},\,z' \sim \mathcal{D}}\,[\|g_z - g_{z'}\|],
\end{equation}
and there are other metrics in the literature such as stiffness~\citep{Fort19} and GSNR~\citep{Liu20}. Furthermore, coherence metrics may be computed at the level of the network (that is, with entire per-example gradient vectors), or at the level of layers or even individual parameters (that is, with different projections of the per-example gradients).

Compared to the other metrics, $\alpha$ (and by extension, $\alpha_m / \alpha^{\perp}_{m}$) has certain advantages: It is simultaneously (1) easy to compute at scale, (2) more interpretable, and (3) has convenient theoretical properties that, among other things, allows us to provide a qualitative bound on the generalization gap.
However, $\alpha$ is not a perfect metric and has some significant limitations as we shall see.

\section{Bounding the Generalization Gap with \texorpdfstring{$\alpha$}{Alpha}}
\label{sec:overview:bound}

We can now formalize the argument in Section~\ref{sec:overview:informal} by using $\alpha$ as a metric of coherence.
We build on the work of \citet{Hardt16} and \citet{Kuzborskij18} who showed that (small-batch) {\em stochastic} gradient descent is stable since each training example is looked at so rarely, that it cannot have much influence on the final model if the training is not run for too long (or the learning rate is decayed quickly enough). 
Obviously, such a dataset-independent argument for stability is inapplicable to the generalization mystery since it rules out memorization.\footnote{In the light of this, the experiments of \citet{Zhang17} may be seen as demonstrating that in practice we run SGD long enough that is it is no longer (unconditionally) stable.} 
In contrast, we provide a dataset-dependent argument for stability may be summarized as follows:

\begin{center}
coherence $\implies$ stability $\implies$ generalization
\end{center}

\noindent In the reverse direction, although it is known that generalization implies stability, generalization or stability do not imply coherence. There may be good generalization in spite of low coherence simply by virtue of having many training examples (for example, if the learning problem is under-parameterized) or by training for a short time.

In this context, our main result 
is a bound on the expected generalization gap,
that is, the expected difference between training and test loss over all samples of size $m$ from $\mathcal{D}$ (denoted by $\gap(\mathcal{D}, m)$) 
in terms of $\alpha$. In its most general form, where we make no assumption on the learning rate schedule, it is as follows:

\begin{theorem}[Generalization Theorem]
\label{thm:cogen}

If (stochastic) gradient descent is run for $T$ steps on a training set consisting of $m$ examples drawn from a distribution $\mathcal{D}$, we have, 
\begin{align}
|\gap(\mathcal{D}, m)| 
\le 
    \frac{L^2}{m} \ \sum_{t = 1}^{T}
    [\eta_{k}\,\beta]_{k = t + 1}^{T}
    \cdot
    \eta_t
    \cdot
    \sqrt{2 \  (1 - \alpha(w_{t - 1}))}
\end{align}
\end{theorem}
\noindent where $\alpha(w)$ denotes the coherence at a point $w$ in the parameter space, 
$w_t$ the parameter values seen during gradient descent, $\eta_t$ is the step size at step $t$,
\[
[\eta_{k}\,\beta]_{k = t_0}^{t_1} \equiv \prod_{k = t_0}^{t_1} (1 + \eta_{k} \, \beta),
\]
and $L$ and $\beta$ are certain Lipschitz constants. 

\begin{proof}
Please see \Cref{app:generalization_theorem} for the full formal statement and the proof.
\end{proof}

Like \citet{Hardt16}, our bound depends on the length of training (shorter the training, better the generalization), and the size of the training set (more the examples, better the generalization).%
\footnote{In fact, when $m$ is large, even with random data, we can have good generalization since we are interested in the gap, and not just test loss---fitting the training set gets harder with larger $m$.
Although natural, this inverse dependence on $m$ does not usually hold for bounds based on uniform convergence as was pointed out by \citet{Nagarajan19b}.} 
However, unlike \citet{Hardt16}, the bound also depends on the coherence as measured by $\alpha$ during training (greater the coherence, better the generalization). 
Due to the presence of the ``expansion" term $[\eta_{k}\,\beta]_{k = t + 1}^{T}$,
an interesting qualitative aspect of this dependence is that high coherence early on in training is better than high coherence later on.
Also in contrast to \citet{Hardt16}, our bound applies in an uniform manner to the stochastic and the full-batch case in the general non-convex setting.\footnote{In fact, since our analysis is in expectation, the size of the minibatch drops out of the bound.}

We emphasize that as with most theoretical generalization bounds in deep learning, our bound is extremely loose, and is therefore only useful in a qualitative sense. This is not only due to the Lipschitz constants. The coherence term based on $\alpha$ only plays a strong role when it is close to 1, but as we shall see in the next section, $\alpha$ on real datasets is quite far from 1.
This is because $\alpha$, being an average over the entire network, is a rather blunt instrument, and conjecture that a tighter bound may be obtained in terms of layer-wise coherences, or perhaps better yet, in terms of ``minimum coherence cuts" of the network.

The proof of the theorem follows the iterative framework introduced in \citet{Hardt16}.
We analyze the evolution of two models under stochastic gradient descent, one trained on the original training set, and another trained on a perturbed training set where the $i$th training example ($z_i$) is replaced with a different one from the data distribution (call it $z_i'$). 
When a minibatch with the $i$th example is encountered, the two models diverge (further) due to the different gradients, and this divergence (in expectation) is bounded by the quantity in (\ref{eq:chap1:diffmetric}).
This quantity, in turn, is bounded by $\alpha$ as follows (see Lemma~\ref{lem:diffbound} in \Cref{app:alphafacts}):
\[
\E_{z \sim \mathcal{D},\,z' \sim \mathcal{D}}\,[\|g_z - g_{z'}\|]
\leq    
\sqrt{2 \  (1 - \alpha) \E_{z \sim \mathcal{D}} \  [g_z \cdot g_z]}.
\]

\section{Measuring \texorpdfstring{$\alpha$}{Alpha} on Real and Random Datasets}
\label{sec:overview:measurement}

\begin{figure*}[t]
\centering
\includegraphics[width=\textwidth]{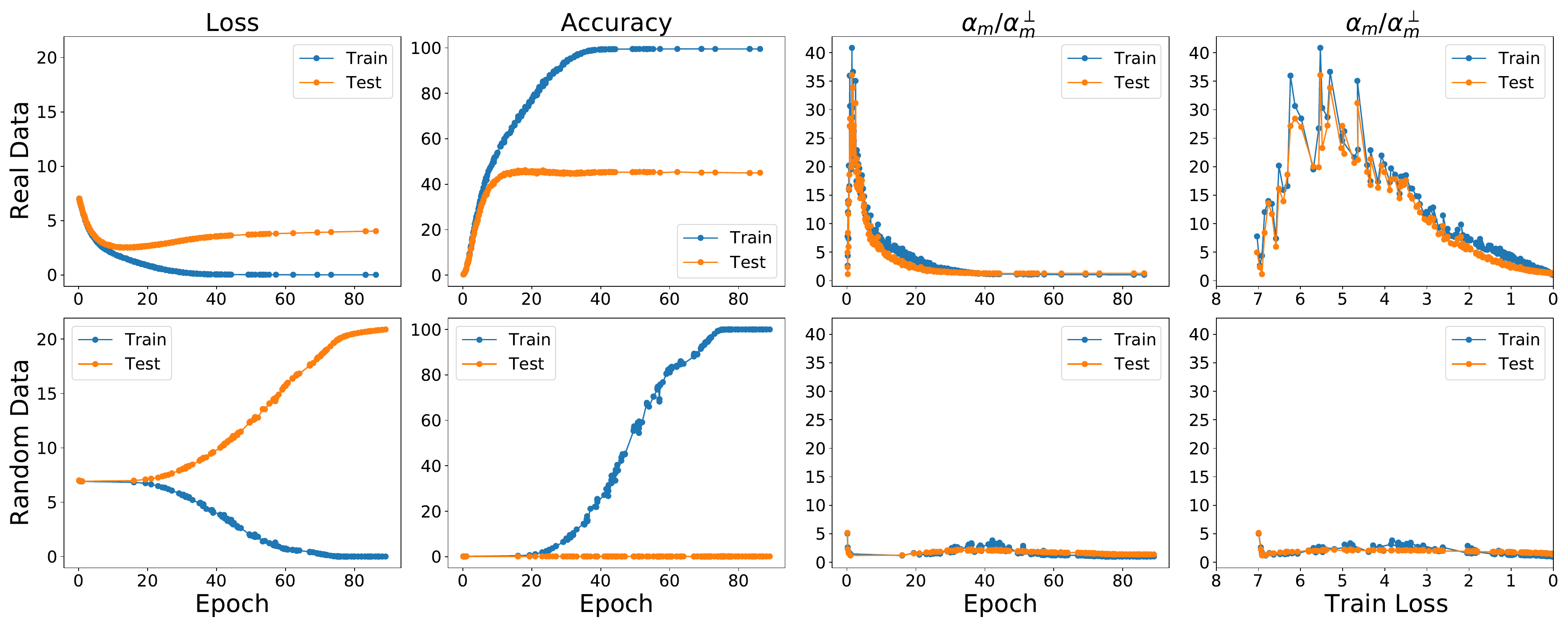}
\vspace{0.1cm}
\caption{%
An experiment in the spirit of~\citet{Zhang17} illustrating the generalization mystery in deep learning.
We train two ResNet-50 models, one on ImageNet with original labels (``real", top row), and another on ImageNet with images replaced by Gaussian noise (``random", bottom row) using vanilla SGD and no explicit regularization. 
As the loss and accuracy curves (first two  columns) show, the network has sufficient capacity to memorize the training data, yet, it generalizes in one case and not the other.  
We believe that the reason for this difference in behavior can be found by analyzing the similarity between per-example gradients during training, that is, {\em coherence}.
Using the $\alpha / \alpha_{m}^{\perp}$ metric for coherence (last two columns), we see that in the case of real data, the per-example gradients are much more similar, and each example helps reduce the loss on many other examples, as compared to the random case. }
\vspace{1cm}
\label{fig:chap1:ResNet-50imagenet}
\end{figure*}

\begin{figure*}[t]
\centering
\includegraphics[width=\textwidth]{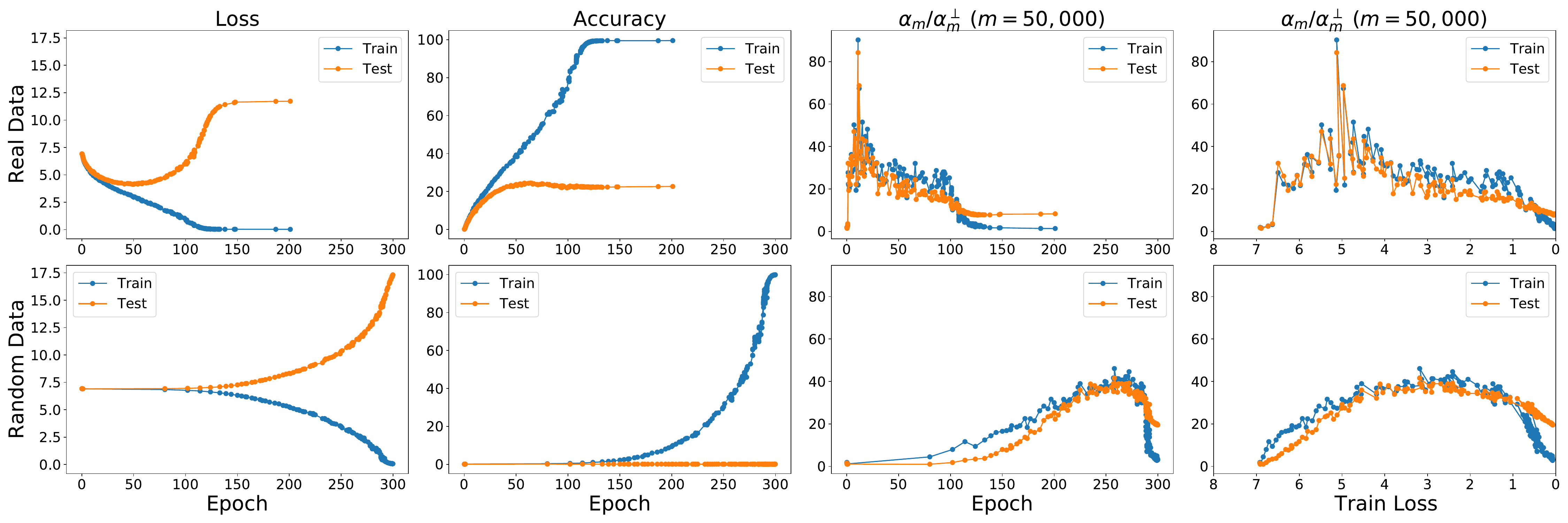}
\vspace{0.1cm}
\caption{Unlike the situation with ResNet-50 (\Cref{fig:chap1:ResNet-50imagenet}), with AlexNet we find that the peak coherence for random data (second row) as measured by $\ralpha$ can be surprisingly high, even though it happens much later in training, and is lower than that of real (first row). Although this appears to be a contradiction to the theory, it is not; it is a limitation of the metric. $\ralpha$ in this plot is a measure of coherence over the entire network (that is, over entire per-example gradients), and is therefore an average quantity. A closer look at the layer-by-layer values of $\ralpha$ as shown in \Cref{fig:chap1:alexnetimagenet_layer} reveals, once again, a significant difference between real and random data.}
\vspace{1cm}
\label{fig:chap1:alexnetimagenet_overall}
\end{figure*}

\begin{figure*}[t]
\centering
\includegraphics[width=\textwidth]{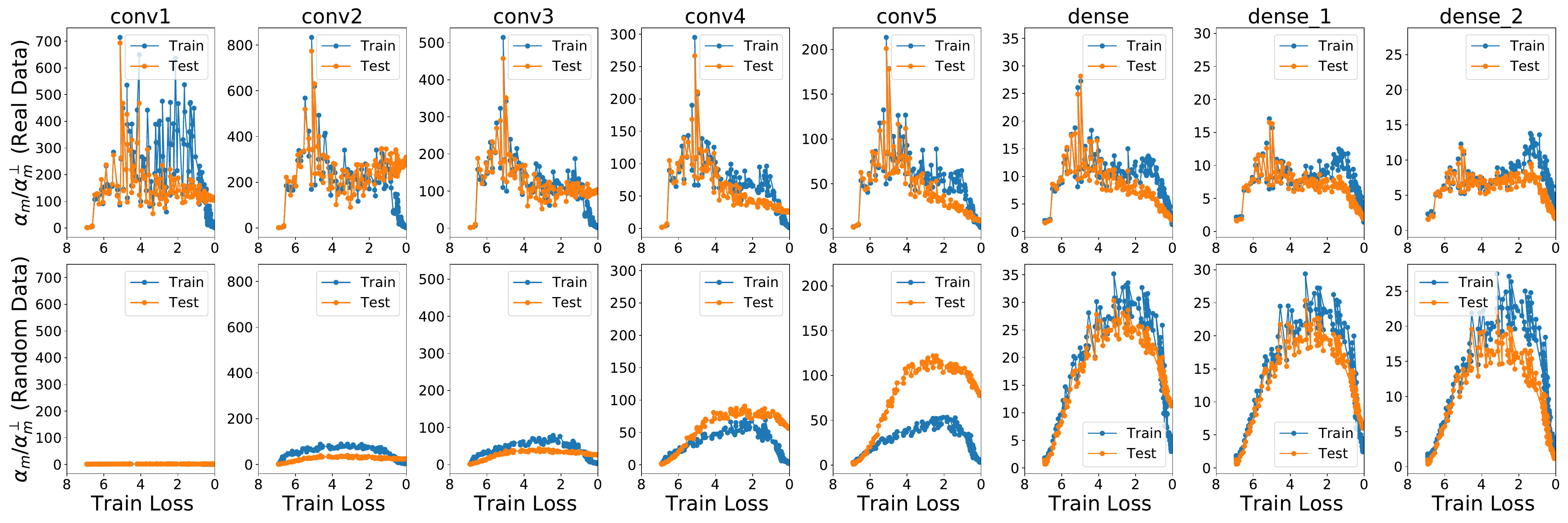}
\vspace{0.1cm}
\caption{A layer-by-layer breakdown of $\ralpha$ for AlexNet from \Cref{fig:chap1:alexnetimagenet_overall} shows that on random data (second row), $\ralpha$ is indeed close to 1 and much lower than that of real data (first row) for the first few layers. For the higher (dense) layers, coherence is comparable between real and random, though note the difference in scale of $\ralpha$ between the convolutional and dense layer plots.}
\label{fig:chap1:alexnetimagenet_layer}
\end{figure*}

In this section, we use the $\ralpha$ metric developed in Section~\ref{sec:overview:metrics} to investigate the generalization mystery.
We can measure $\alpha_m$ at any point in training by computing it directly using (\ref{eq:overview:metric}) on a given sample.%
\footnote{Observe that $\alpha$ can be computed efficiently in an online fashion by keeping a running sum for $\E[g_z]$, and another for $\E[g_z \cdot g_z]$.}
The sample could be either from the training set (giving us ``training" coherence), or the test set (``test" coherence).

One complication is the use of batch normalization~\citep{Ioffe15} in most practical networks, since in that case, per-example gradients are no longer well-defined. 
This is a problem not just with $\alpha$, but with any metric that is based on per-example gradients.
However, $\alpha$ has an interesting mathematical property that allows us to impute it using the coherence of {\em mini-batch} gradients.
We discuss this in greater detail in \Cref{app:measuring-coherence}. 
In what follows, we use this imputation method for networks with batch normalization.

In our main experiment, along the lines of \citet{Zhang17}, 
we train a ResNet-50 network on two datasets, 
one real (ImageNet) and the other random (``ImageNet'' with random input pixels).
Both networks are trained with vanilla SGD (that is, no momentum)
with a constant learning rate of 0.5 and batch size of 4096.
We do not use any explicit regularizers such as weight decay or dropout. 
We estimate test and train $\alpha$ during training using the test set (which has 50K samples) and a training sample (a random sample of 50K training examples). We take snapshots of the network during training each time the (minibatch) training loss falls by 1\% from the previous low watermark, and compute loss, accuracy, and $\ralpha$ on these snapshots using the test and training samples.

The resulting training curves are shown in the first two columns of Figure~\ref{fig:chap1:ResNet-50imagenet}.
As expected, both networks achieve zero training loss and near perfect training accuracy, but only the network trained on (real) ImageNet shows non-trivial generalization to the test set.

Coherence as measured by $\ralpha$ ($m$ = 50,000) is shown in the third and fourth columns of Figure~\ref{fig:chap1:ResNet-50imagenet}.
The third column shows coherence as a function of the number of epochs trained, whereas the fourth columns shows it in terms of the training loss.
Since the realized rate of learning is different for the two datasets (the real data is learned in fewer epochs), plotting coherence against loss allows us to compare across the two datasets more easily.
 
In the case of real data, we observe that $\ralpha$ starts off low (around 1) in early training and then increases to a maximum (about 40) within in the first few epochs and then returns to a low value again (around 1) at the end of training.\footnote{For now, we ignore the small differences in training and test coherence.} 
In the plot of $\ralpha$ against training loss, we see that when actual learning happens, that is, when the loss comes down, $\ralpha$ stays around 20.
In other words, when training with real labels, each training example in our set of 50K examples used to measure coherence helps many other examples.

In contrast, for random data, although the evolution of $\ralpha$ is similar to that of real data, the actual values, particularly, the peak is very different. $\ralpha$ starts off low (around 1), increases slightly (staying usually below 5), and then returns back to a low value (around 1).
Therefore, each training example in the case of random data, helps only one or two other examples during training, that is, the 50K random examples used to estimate coherence are more or less orthogonal to each other.%
\footnote{We note here that very early in training, that is, the first few steps (not shown in \Cref{fig:chap1:ResNet-50imagenet}, but presented in \Cref{fig:coherence_first_steps} instead), $\ralpha$ can be very high even for random data due to imperfect initialization. All the training examples are coordinated in moving the network to a more reasonable point in parameter space. As may be expected from our theory, this movement generalizes well: the test loss decreases in concert with training loss in this period. Rapid changes to the network early in training is well documented (see, for example, the need for learning rate warmup in \citet{He16} and \citet{Goyal17}).}

In summary, 
\boldcenter{
With a ResNet-50 model on real ImageNet data, in a sample of 50K examples, each example helps tens of other examples during training, whereas on random data, each example only helps one or two others.
}
This provides evidence that the difference in generalization between real and random stems from a difference in similarity between the per-example gradients in the two cases, that is, from a difference in coherence.

While experiments with other architectures and datasets also show similar differences between real and random datasets (see \Cref{app:coherence-additional}), there are cases when the coherence of random data {\em as measured by $\ralpha$ over the entire network} can be surprisingly high for an extended period during training.%
\footnote{As we discussed earlier, coherence even for random can be high for a short period early on in training due to imperfections in initialization. But the difference here is {\em sustained} high coherence.}
In our experiments, we found an extreme case of this when we replaced the ResNet-50 network in the previous experiment with an AlexNet network (learning rate of 0.01).
The training curves and measurements of $\ralpha$ in this case are shown in Figure~\ref{fig:chap1:alexnetimagenet_overall}.
As we can see, unlike the ResNet-50 case, $\ralpha$ reaches a value of 40 for $m = 50,000$. In other words, in a sample of 50K examples, at peak coherence, each random example helps 40 other examples!%
\footnote{
That said, note that (1) even in this case, at its peak $\ralpha$ for real is more than 2$\times$ the peak for random; and (2) the high coherence of random occurs much later in training than that of real which possibly indicates the importance of the ``expansion term" ($[\eta_{k}\,\beta]_{k = t + 1}^{T}$) in the bound of Theorem~\ref{thm:cogen} (see discussion in Section~\ref{sec:overview:bound}).}

What is going on? An examination of the per-layer values of $\ralpha$ provides some insight. These are shown in Figure~\ref{fig:chap1:alexnetimagenet_layer}.
We see that for the first convolution layer ({\sf conv1}) in the case of random---and {\em only} in that case---$\ralpha$ is approximately 1 indicating that the per-example gradients in that layer are pairwise orthogonal (at least over the sample used to measure coherence).%
\footnote{The difference in $\ralpha$ in the first layer between real and random is also seen when the entire training set is used to measure $\ralpha$ (\Cref{fig:alexnet_full_data_layers}).}
This indicates that the first layer plays an important role in ``memorizing" the random data since each example is pushing the parameters of the layer in a different direction (orthogonal to the rest). This is not surprising since the images are comprised of random pixels.%

Now, the overall (network) $\alpha$ is a convex combination of the per-layer $\alpha$s (see \Cref{thm:weightedmean} in \Cref{app:alphafacts}). Since the fully connected layers have high coherence, 
overall $\alpha$ (as shown in Figure~\ref{fig:chap1:alexnetimagenet_overall}) can be high even though there is a layer with very low $\alpha$ (at the orthogonal limit). 
In other words,
\boldcenter{
As a measure of coherence, $\ralpha$ over the whole network, being an average, is a blunt instrument, and therefore, a finer-grained analysis, for example, on a per-layer basis, is sometimes necessary.
}
An important open problem, therefore, is to devise a better metric for coherence that accounts for the structure of the network, and to use that metric to obtain a bound stronger than that in Theorem~\ref{thm:cogen}.
Please also see the discussion in \Cref{sec:overview:feedback}, and particularly, \Cref{ex:two-neurons-m} for a closer look at this in the context of a simple, illustrative example.

\newbold{Evolution of Coherence.}
Experiments across several architectures and datasets show a common pattern in how coherence as measured by $\ralpha$ (or equivalently, $\alpha$) changes during training.
Ignoring the initial transient in the first few steps of training, coherence follows a broad parabolic pattern:
It starts off at a low value, rises to a peak, and then comes back down to the orthogonal limit.%
\footnote{In rare cases, such as a fully connected network on {\sc mnist} where the signal is strong and easy to find, coherence starts off high.}
This happens regardless of whether the dataset is random or real, indicating that this is an optimization (as opposed to a generalization) effect.
We discuss the reasons behind this in \Cref{app:evolution}. 


\section{From Measurement to Control: Suppressing Weak Descent Directions}
\label{sec:overview:suppress}

Weak directions in the average gradient are supported by few examples or perhaps even one, and therefore, are less stable than strong directions which are supported by many examples.
Therefore, as discussed in \Cref{sec:overview:informal}, 
suppressing weak directions in the average gradient should lead to less overfitting
and better generalization.
Now, although existing regularization techniques such as weight decay, dropout, and early stopping may be viewed through this lens, 
the theory also suggests a new, more direct, regularization technique we call {\em winsorized gradient descent} (WGD).

In WGD, instead of updating each parameter with the average gradient as in gradient descent, we update it with a ``winsorized" average where the most extreme values (outliers) are clipped.
Formally, let $w_t^{(j)}$ represent the $j$th trainable parameter (that is, the $j$th component of the parameter vector $w_t$) at step $t$, and $g^{(j)}_i(w_t)$ the $j$th component of the gradient of the $i$th example at $w_t$.
In normal gradient descent, we update $w_t^{(j)}$ as follows:
\[
w_{t + 1}^{(j)} = w_t^{(j)} - \eta
\frac{1}{m} \sum_{i = 1}^{m} g_i^{(j)}(w_t).
\]

Now, let $c \in [0, 50]$ be a hyperparameter that controls the level of winsorization. 
Define $l^{(j)}$ to be the $c\,$th percentile of $g_i^{(j)}(w_t)$ (over the examples $i$). 
Likewise, let $u^{(j)}$ be the $(100 - c)\,$th percentile of $g_i^{(j)}(w_t)$.
The update rule with winsorized gradient descent is as follows:
\[
\boxed{
\quad
w_{t + 1}^{(j)} = w_t^{(j)} - \eta
\frac{1}{m} \sum_{i = 1}^{m} {\rm clip}(g_i^{(j)}(w_t),\,l^{(j)},\,u^{(j)})
\quad
}
\]
where ${\rm clip}(x, l, u) \equiv {\rm min}({\rm max}(x, l), u)$.

The modified update rule minimizes the effect of outliers in the per-example gradients on a per-coordinate basis.
The value of $c$ dictates what is an outlier. When $c = 0$, nothing is an outlier, and this corresponds to normal gradient descent, whereas when $c = 50$, all values other than the median are considered outliers. 
Thus, increasing $c$ reduces variance and increases bias.

\begin{example}[WGD applied to \Cref{ex:two-datasets}]

Recall that using the component-wise median gradient in \Cref{ex:two-datasets} instead of the average gradient reduced the generalization gap to zero for both ``real" and ``random."
We note that this corresponds to running WGD with $c = 50$.

\qed
\end{example}

Although the modification for WGD is a conceptually simple change, it greatly increases 
the computational expense due to
the need to compute and store per-example gradients for all the examples. 
The computational expense can be reduced by 
performing winsorized {\em stochastic} gradient descent (WSGD). This is a straight forward modification of SGD where the winsorization is only performed over the examples in the mini-batch rather than all examples in the training set.

\begin{figure*}[t]
\centering
\includegraphics[width=\textwidth]{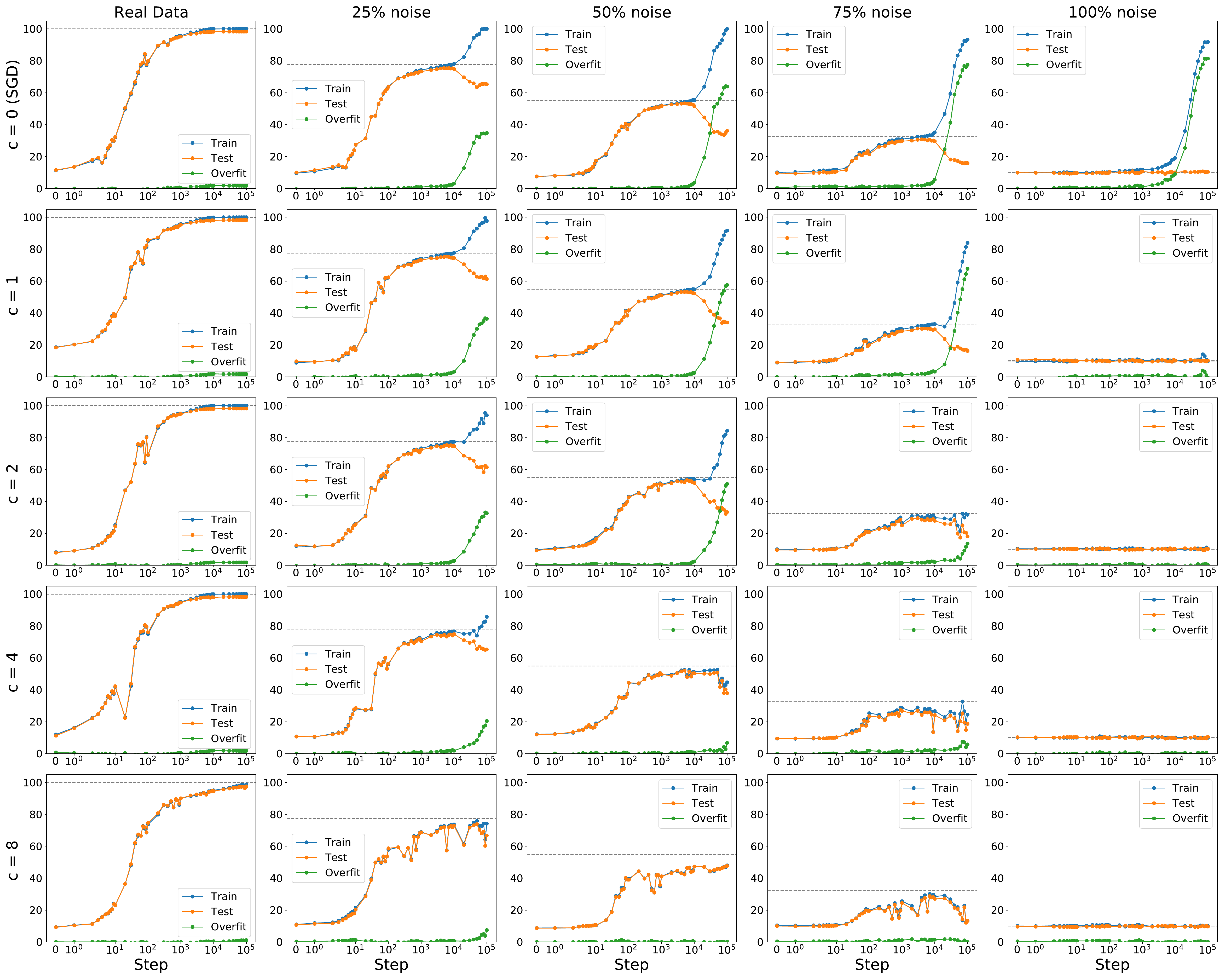}
\caption{
Generalization improves when weak directions in the average gradient are suppressed during gradient descent. Weak directions are suppressed by winsorization, that is, by clipping extreme per-examples gradients (independently for each coordinate of the gradient). The parameter $c$ controls the level of winsorization and $c = 0$ corresponds to using the (usual) average gradient. We train a fully connected network on {\sc mnist} with varying amounts of label noise (see \Cref{fig:overview:winsorization_gaussian} for a similar experiment with random pixels).
}
\vspace{1cm}
\label{fig:chap1:iclrwinsorizationmnist}
\end{figure*}

\newbold{WSGD on MNIST.} 
We train a small fully-connected ReLU network with 3 hidden layers of 256 neurons each for 60,000 steps (100 epochs) with a fixed learning rate of 0.1 on {\sc mnist} and 4 variants with different amounts of label noise $\epsilon$ ranging from 25\% to 100\%.%
\footnote{A label noise of 25\% means that the dataset is constructed by randomly assigning labels to 25\% of the training examples in {\sc mnist} that are chosen uniformly at random. The remaining 75\% have the correct labels as do the test examples. Since {\sc mnist} has 10 possible labels, this means that 75\% + (1/10) 25\% = 77.5\% of training examples have pristine (that is, correct) labels and 22.5\% have corrupted labels.}
We use WSGD with a batch size of 100 and vary $c$ in $\{0, 1, 2, 4, 8\}$.
Since we have 100 examples in each minibatch, the value of $c$ immediately
tells us how many outliers are clipped in each minibatch. For example, $c=2$
means the 2 largest and 2 lowest values of the per-example gradient in a batch are
clipped (independently for each trainable parameter in the network), and as always, $c = 0$ corresponds to unmodified SGD.

The resulting training and test curves shown in Figure~\ref{fig:chap1:iclrwinsorizationmnist}. The
columns correspond to different amounts of label noise and the rows to
different amounts of winsorization.
In addition to the training and test accuracies (${\sf ta}$ and ${\sf va}$,
respectively), we show the level of overfit which is defined as ${\sf ta} - {\sf va}$.

As expected, as $c$ increases, and the weaker directions are suppressed more, the extent of overfitting decreases. 
Furthermore, for larger values of $c$ (for example, $c = 8$) the ability to fit the corrupted labels is severely impacted. The training accuracy stays below the accuracy that would be obtained if only the uncorrupted labels were learned (shown by dashed gray lines in the plot).
The ability to memorize, that is, fit random labels (100\% noise) is impacted more than the ability to fit real labels (0\% noise): the effective learning rate (the rate at which training accuracy increases) is much lower for random than for real.

In summary,
\boldcenter{
If we suppress weak gradient directions by modifying SGD to use a robust average of per-example gradients that excludes outliers, generalization improves. This provides further direct evidence that weak directions are responsible for overfitting and memorization.
}

Finally, we notice that with a large amount of winsorization, there can be 
optimization instability (not to be confused with algorithmic stability which as we have seen 
is improved): training accuracy can fall after a certain point.
We are not sure what causes the instability but conjecture that this is because of a strengthening of positive feedback loops in strong directions. Positive feedback loops are discussed in  \Cref{sec:overview:feedback}.

\begin{figure*}[t]
\centering
\includegraphics[width=\textwidth]{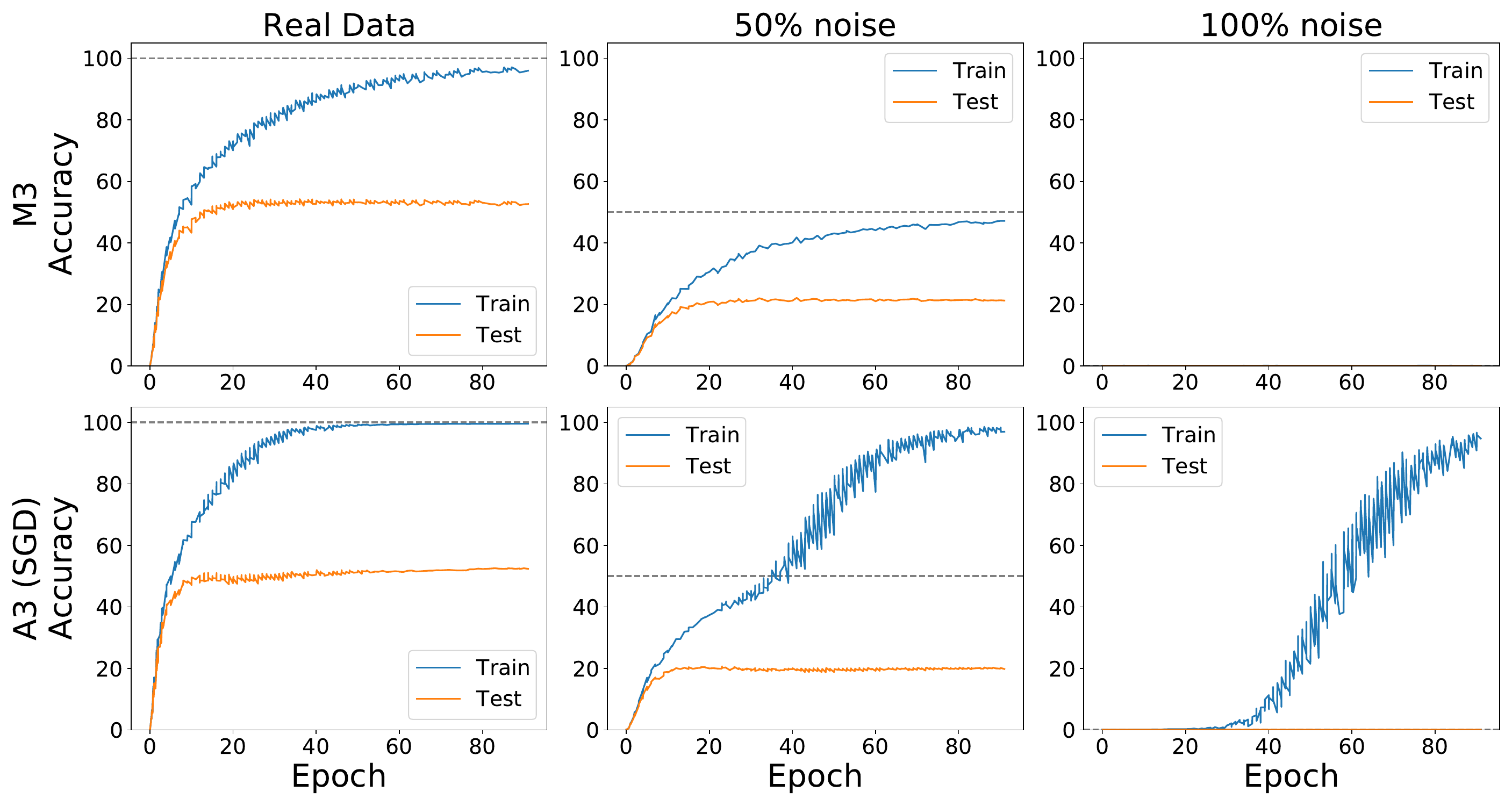}
\vspace{0.1cm}
\caption{Taking the coordinate-wise median of 3 micro-batches ({\sf M3}) suppresses weak gradient directions more than taking their coordinate-wise average ({\sf A3}) (the same as ordinary SGD), leading to better generalization. 
Here, we train a ResNet-50 on 3 datasets derived from ImageNet by replacing some fraction of the training images with Gaussian noise. 
In the case of random data (100\% noise), {\sf M3} prevents memorization, so the training and test curves both lie on the x-axis. In contrast, as expected, {\sf A3} (SGD) memorizes the training set.
When only half the training images are replaced by noise (50\% noise), {\sf M3} reaches a training accuracy near 50\% suggesting that only the real images are learned. This is confirmed in \Cref{fig:overview:resnet_imagenet_a3_vs_m3_50pct_breakdown}.}
\vspace{1cm}
\label{fig:overview:resnet_imagenet_a3_vs_m3}
\end{figure*}

\begin{figure*}[t]
\centering
\includegraphics[width=0.75\textwidth]{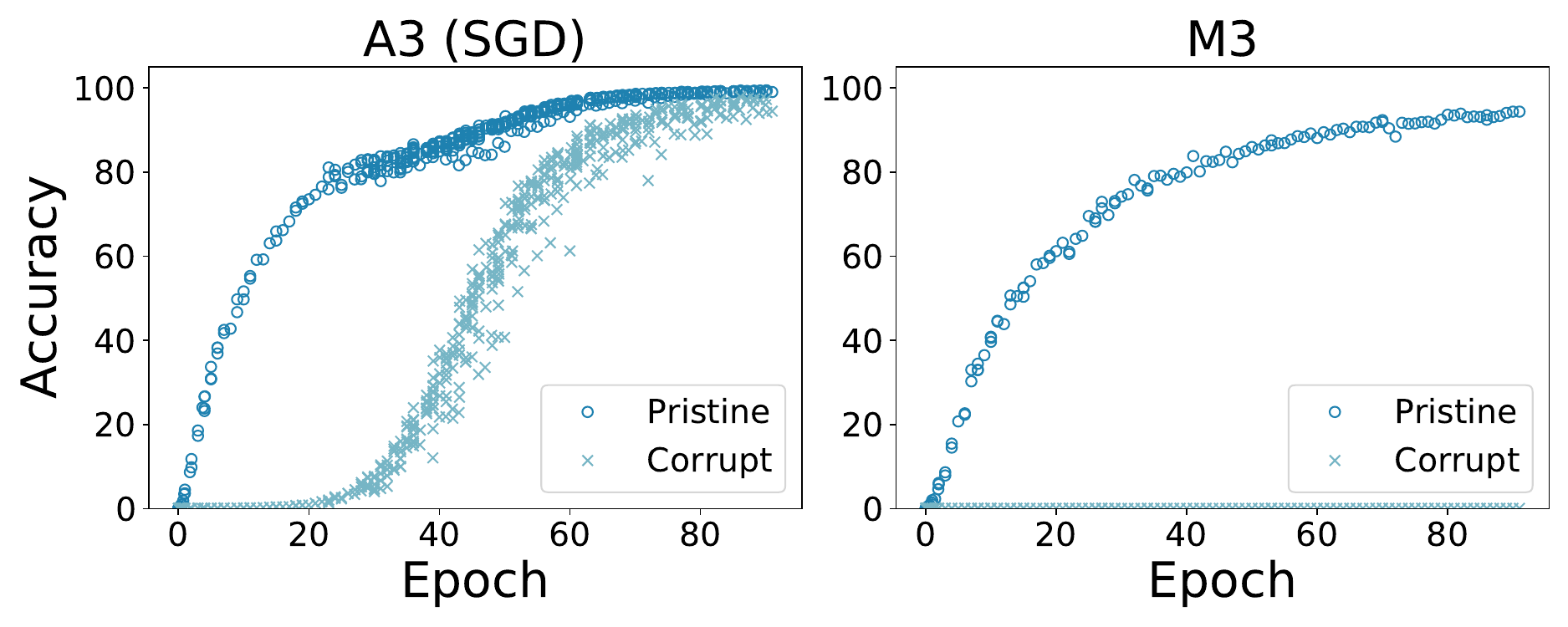}
\vspace{0.1cm}
\caption{In the experiment of \Cref{fig:overview:resnet_imagenet_a3_vs_m3}, on the dataset where half the training images are replaced by noise (50\% noise), {\sf M3} only learns the real (pristine) training images but does not learn the random (corrupt) images. In contrast, {\sf A3} (that is, SGD) learns both, though the corrupt examples are learned later in training. 
This is further evidence that the gradients of real data are well-aligned, and add up to strong directions in the average gradient; whereas those of random data are not well-aligned, and may be suppressed if weak directions are attenuated through robust averaging. 
This difference in the strength of the gradients in the two cases explains the observed difference in generalization between them.}
\vspace{1cm}
\label{fig:overview:resnet_imagenet_a3_vs_m3_50pct_breakdown}
\end{figure*}

\newbold{Scaling up.}
Winsorization requires computing and storing many per-example gradients.%
\footnote{For exact computation, the storage would depend on $c$ and may be tractable for small $c$. This could be reduced further by employing approximations to compute streaming quantiles.}
This makes it impractical for large datasets such as ImageNet, as well as 
inapplicable to popular networks such as ResNet-50 which employ batch normalization as per-example gradients are not defined in that case.

An alternative to winsorization that addresses both these problems is the well known technique of taking median of means (see, for example, \citet{Minsker13}).%
\footnote{
Median of means has the theoretical property that its deviation from the true mean is bounded above by $O(1/\sqrt{m})$ with high probability (where $m$ is the number of samples). The sample mean satisfies this property only if the observations are Gaussian. 
}
The main idea of the \emph{median of means} algorithm is to divide the samples into $k$ groups, computing the sample mean of each group, and then returning the median of these $k$ means. 
The most obvious way to apply this idea to SGD is to divide a mini-batch into $k$ groups of equal size. We compute the mean gradients of each group as usual, and then take their coordinate-wise median.
The median is then used to update the weights of the network.\footnote{Since we are simply replacing the mini-batch gradient with a more robust alternative, this technique may be used in conjunction with other optimizers.}

Even though the algorithm is straightforward, its most efficient implementation (that is, where the $k$ groups are large and processed in parallel) on modern hardware accelerators requires low-level changes to the stack to allow for a median-based aggregation instead of the mean. 
Therefore, in this work, we simply compute the mean gradient of each group as a separate {\em micro-batch} and only update the network weights with the median every $k$ micro-batches, i.e., we process the groups serially. 

In the serial implementation, $k = 3$ is a sweet spot. We have to remember only 2 previous micro-batches, and since ${\rm median}(x_1, x_2, x_3) = \sum_i x_i - \min_i x_i - \max_i x_i$ (where $i \in \{1, 2, 3\}$), we can compute the median with simple operations. 
We call this {\sf M3} (median of 3 micro-batches).

\begin{example}[{\sf M3} applied to \Cref{ex:two-datasets}]
In this case, applying {\sf M3} with a micro-batch size of 1 (batch size of 3) completely suppresses the weak gradient directions (corresponding to the idiosyncratic signals) in both the ``real" and ``random" datasets, and leads to perfect generalization. 

\qed
\end{example}

\Cref{fig:overview:resnet_imagenet_a3_vs_m3} shows the effect of {\sf M3} on 3 datasets: real ImageNet, ImageNet with half the images replaced by Gaussian noise (50\% noise), and ``ImageNet" with all the images replaced by Gaussian noise (100\% noise). The 100\% noise dataset is the same as the ``random" dataset in \Cref{sec:overview:measurement}.
We run {\sf M3} with a micro-batch of 256 (that is, batch size of $3 \times 256$) for 90 epochs and  a learning rate of 0.20.
For comparison we also run SGD with an identical setup. Instead of the median of 3 micro-batches, in SGD we compute their average, and hence, in this context, we refer to SGD as {\sf A3} to highlight that the {\em only} difference in the two cases is replacing the {\sf median} operation with {\sf average}.%
\footnote{In practice, SGD with a batch size of $3 \times 256$ is subtly different from A3 with a micro-batch of $256$ in how the low-level computation is distributed among the accelerators and in the resulting statistics for batch normalization. Although that difference is immaterial, we use {\sf A3} here to eliminate {\em all} differences other than the method of aggregation.}

We observe the following:

\begin{enumerate}
    \item On real data, {\sf M3} only slightly slows down training after about 50\% of the dataset has been learned, and almost reaches the 100\% training accuracy achieved by A3. Also, {\sf M3} has a lower generalization gap (43.34\%) compared to {\sf A3} (47.15\%).
    
    \item With 50\% noise, {\sf M3} reaches a training accuracy less than 50\% which is much lower than the 100\% training accuracy of {\sf A3}. Thus {\sf M3} has a much lower generalization gap (25.90\%) compared to {\sf A3} (77.13\%).
    Further investigation shows that {\sf M3} learns the real 50\% of the training dataset (the ``pristine" examples) well, but does not learn the remaining 50\% of the training set that is Gaussian noise (the ``corrupt" examples). This is, of course, in contrast to {\sf A3} which learns both almost equally well. See \Cref{fig:overview:resnet_imagenet_a3_vs_m3_50pct_breakdown}.
    
    \item Finally, in the case of 100\% noise, {\sf M3} fails to learn any of the training data at all in sharp contrast to {\sf A3}.
    
\end{enumerate}

This provides further evidence that (a) weak directions are responsible for memorization, and (b) suppressing weak directions improves generalization (via improving stability).

The choice of hyper-parameters of {\sf M3} requires care. Larger the size of the micro-batch, less effective is the median operation in suppressing weak directions (consider the extreme case when the micro-batch is the entire training set and taking the median serves no purpose).
However, smaller the micro-batch, less reliable are the batch statistics (for batch normalization), and higher the effective level of winsorization (consider a micro-batch size of 1 where the effective winsorization level now is $c = 50$).
The learning rate has to be set in conjunction with the micro-batch size, and for some learning rates we found the training to be unstable (see \Cref{fig:overview:resnet_imagenet_a3_vs_m3_unstable}).
This is similar to the optimization stability seen with winsorization previously.

\section{Why are Some Examples (Reliably) Learned Earlier?}
\label{sec:overview:easyhard}

\begin{figure*}[t]
\centering
\includegraphics[width=\textwidth]{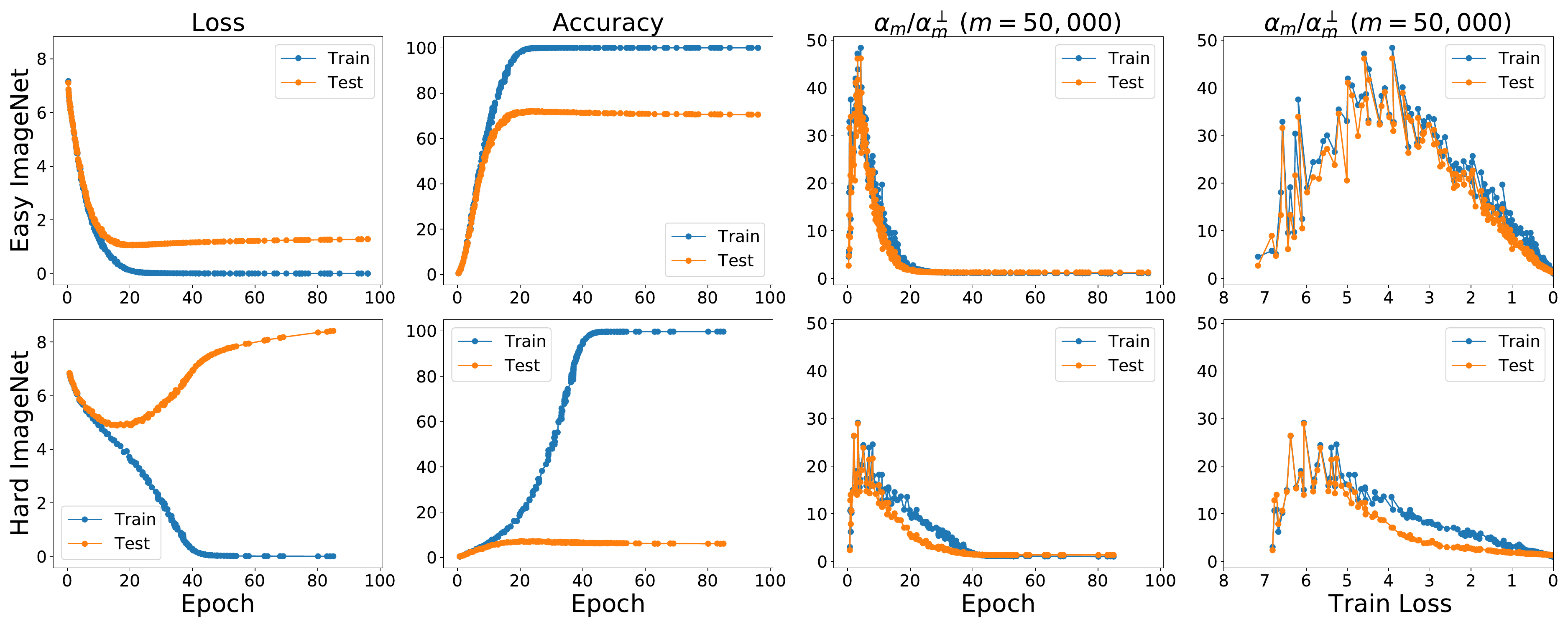}
\vspace{0.1cm}
\caption{Coherence as measured by $\ralpha$ is higher for examples that are learned early in ImageNet training (Easy ImageNet) than for those that are learned later (Hard Imagenet).
As expected, the higher coherence translates into better generalization. 
The network used is a ResNet-50.}
\vspace{1cm}
\label{fig:chap1:easyhardfigure1}
\end{figure*}

In a study on memorization in shallow fully-connected networks and small 
convolutional networks on {\sc mnist} and {\sc cifar}-10, \citet{Arpit17} discovered that for real datasets, starting from different random initializations, many examples are consistently classified correctly or incorrectly after one epoch of training. 
They call these {\em easy} and {\em hard} examples respectively. They hypothesize that this variability of difficulty in real data ``is because the easier examples are explained by some simple patterns, which are reliably learned within the first epoch of training.'' 

But that begs the question what makes a pattern ``simple" and why are such patterns reliably learned early? 
The theory of Coherent Gradients provides a refinement of their hypothesis:
\boldcenter{
Easy examples are those examples that have a lot in common with other examples where commonality is measured by the alignment of the gradients of the examples. 
}
Under this assumption, it is easy to see why an easy example is learned sooner reliably: most gradient steps benefit it.

Note that this hypothesis is more nuanced than the conjecture in \citet{Arpit17}. 
First, the difficulty of an example is not simply a property of that example (whether it has simple patterns or not), but depends on the relationship of that example to others in the training set (what it shares with other examples).

Second, the dynamics of training, including initialization, can determine the difficulty of examples. 
Consequently, it can accommodate the observed phenomenon of adversarial initialization~\citep{Liao18,Liu19} where examples that are easy to learn with random initialization become significantly harder to learn with a different, specially constructed initialization, and the generalization performance of the network suffers. 
Any notion of simplicity of patterns intrinsic to an example alone cannot explain adversarial initialization since the dataset remains the same, and therefore, so do the patterns in the data.

In order to test the above hypothesis, we create two datasets with 500K training and 50K test examples from the standard ImageNet training set. These datasets consists of examples that a ResNet-50 reliably learns early (``Easy ImageNet") or late (``Hard ImageNet"). See \Cref{app:easyhard} for more details on this construction.
Next, as in \Cref{sec:overview:measurement}, we measure loss, accuracy, and coherence during training on both datasets. 

\begin{figure*}[t]
\centering
\includegraphics[width=0.5\textwidth]{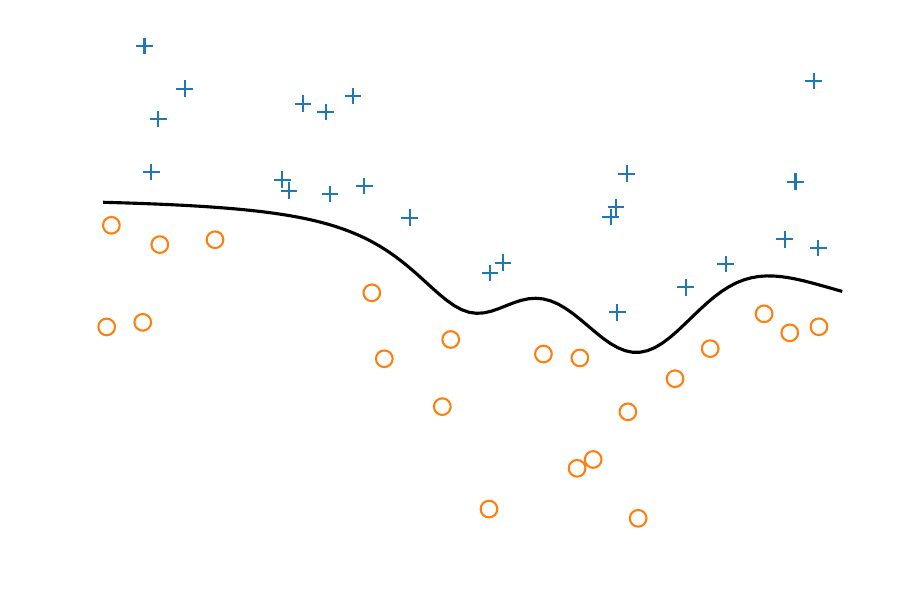}
\caption{If gradient descent enumerates hypotheses of increasing complexity
then the examples learned early, that is, the easy examples, should be the ones 
far away from the decision boundary since they can be separated by simpler hypotheses. 
However, a model learned from the hard examples then, should generalize well to easy examples, but that is not what we find on ImageNet. A model trained on hard ImageNet examples has only a 17\% accuracy on easy ImageNet examples (in contrast to a 71\% accuracy obtained by a model trained on other easy examples).}
\vspace{1cm}
\label{fig:decision}
\end{figure*}

The results are shown in  \Cref{fig:chap1:easyhardfigure1}.
We observe that the coherence for Easy ImageNet is significantly higher than that of Hard ImageNet, and as might be expected from the theory, the generalization gap of Easy ImageNet is smaller since the average gradient in that case is more stable.%
\footnote{This is also seen in an {\em in situ} measurement of coherence of easy and hard examples during regular ImageNet training, but due to a technicality involving batch normalization statistics, that measurement is less reliable. See \Cref{app:easyhard} for details.}

In other words, since easy examples, as a group have more in common with each other (than with the hard examples, or the hard examples have amongst themselves), (1) a model is less impacted by the presence or absence of a single easy example in the training set, leading to good generalization on easy examples; and (2) easy examples have a stronger (focused) effect on the average gradient direction than the hard examples (which are diffuse) leading to easy examples being learned faster.
Thus there is a correlation between how quickly examples are learned and their generalization.%
\footnote{This may be seen as an additional justification for early stopping to the one in \Cref{sec:overview:informal}.}

Although the above experiment shows that easy examples have more in common with each other, it does not rule out that nonetheless there is something intrinsic to each easy example that accelerates learning. For example, it may be that easy examples have larger gradients 
than hard examples.
To investigate this, we train a network on a {\em single} example from Easy ImageNet or Hard ImageNet at a time and measure the number of steps it takes to fit the example. We find that there is no statistically significant difference between the two datasets in this regard.
It is only when sufficiently many easy
examples come together (as in the earlier experiment), that they train faster than hard examples. This may be seen as further evidence that it is the relationship between examples rather than something intrinsic to an example that controls difficulty.
Please see \Cref{app:easyhard} for more details.

\newbold{Does gradient descent explore functions of increasing complexity?}
One intuitive explanation of generalization is that GD somehow explores candidate hypotheses of increasing ``complexity'' during training. Although this is backed by some observational studies (for example, \citet{Arpit17, Nakkiran19, Rahaman19}), we are not aware of any causal mechanism that has been proposed to explain this sequential exploration.%
\footnote{For example, there is no causal intervention similar to suppressing weak gradient directions.}

Our theory suggests one. As described above, as training progresses, more and more examples get fit, starting with examples whose gradients are well-aligned with those of others, and ending with those examples whose gradients are more idiosyncratic.
Thus, one may observe the function implemented by the network to get more and more complicated in the course of training simply because it fits (interpolates through) more and more training examples (points).
An interesting question for future work is if this argument can be formalized to explain some of the specific empirical observations in the literature (such as ReLU networks learning low frequency functions first~\citep{Rahaman19} or  learning low descriptive-complexity functions first~\citep{VallePerez19}).

However, one may wonder if there is some {\em other} causal mechanism that causes SGD to explore functions in increasing order of complexity? 
A simple extension of the previous experiment with Easy and Hard ImageNet suggests not.
To motivate the extension, consider a thought experiment motivated by the 2D classification problem shown in \Cref{fig:decision}. If SGD were to somehow ``try" simpler hypothesis early on, then the easy examples would be the ones far from the decision boundary (which could be easily separated by say, linear classifiers), and the hard ones would be ones close to the boundary (which would need more complex curves to separate).
Therefore, based on the intuition provided by \Cref{fig:decision}, one might expect that the decision boundary learned from the hard examples, would generalize \underline{well} to the easy examples. 

But this is not what we see with real data in high dimensions:
The model trained on Hard ImageNet has an accuracy of only 17\% on the Easy ImageNet test set. This is much lower than the 71\% test accuracy of the model trained on Easy ImageNet (as seen in \Cref{fig:chap1:easyhardfigure1}). 
In other words,
\boldcenter
{
Examples learned late by gradient descent (that is, the hard examples), by themselves, are insufficient to define the decision boundary to the same extent that early (or easy) examples are. 
}
This observation leads us to a view of deep models that is closer in spirit to kernel regressors~\citep{Nadaraya64, Watson64}, but, of course, with a ``learned kernel,'' and is consistent with the intuition 
that in high dimensions, learning (almost) always amounts to extrapolation~\citep{balestriero21}.

\begin{figure*}
\centering
\includegraphics[width=\textwidth]{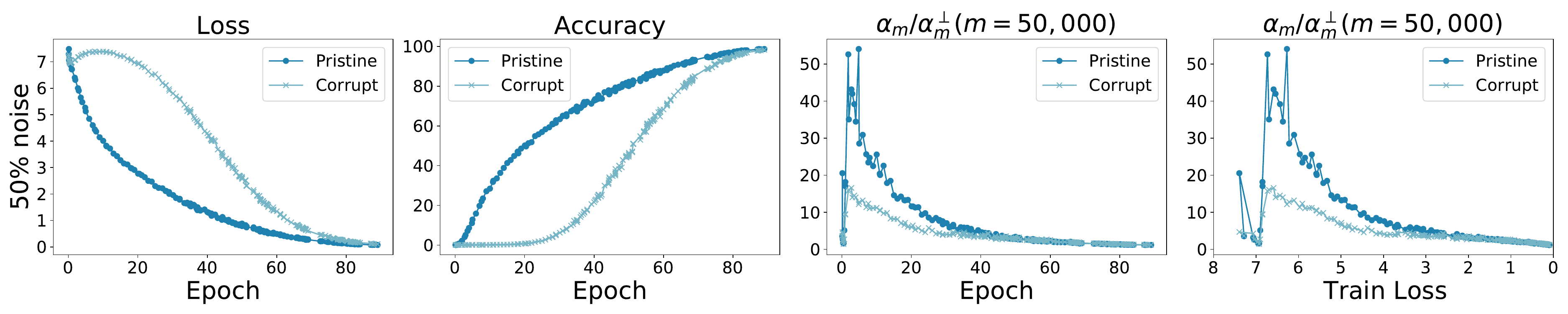}
\caption{Pristine examples, that is examples with correct labels, show higher coherence than the corrupt examples, and consequently are learned much faster. Here, we train a Resnet-50 on ImageNet where half the images in the training set have randomly assigned labels (that is, ImageNet with 50\% label noise). 
}
\label{fig:fig1_pristine_corrupt}
\end{figure*}
 
\section{Learning With Noisy Labels}
The theory of Coherent Gradients provides insight on why it is possible to learn in the presence of label noise~\citep{Rolnick17}. 
As \Cref{fig:fig1_pristine_corrupt} shows, similar to real and random data, coherence as measured by $\ralpha$ is greater for examples with correct labels (the pristine examples) than it is for examples with incorrect labels (corrupt examples).
Consequently, the strong directions in the average gradient correspond to the pristine examples, and they get learned before the corrupt examples, mirroring the 
situation with easy and hard examples discussed in the previous section.

We can thus explain the empirical success of early stopping when learning with noisy labels and the empirical observation that the loss of pristine examples falls faster than that of corrupt examples (see, for example,  ``small-loss" trick in \citet[Section III E]{Song22}).
Furthermore, even if early stopping is not used, and training is continued until convergence, the gradients of corrupt examples are not strong enough to ``interfere'' with those of the pristine examples which reinforce each other.%
\footnote{However, there may still be some coherence amongst incorrectly labeled examples which may cause unintended generalization to the test set. This is manifested as a degradation in test performance as the corrupt examples are learned, and can be seen in experiments (see, for example, \citet[Figure 1 (b)]{Chatterjee20}, \citet[Figure 1 (a)]{Zielinski20}, and in \citet[Figures 7 and 8 (b)]{Arpit17}).}

\section{Depth, Feedback Loops, and Signal Amplification}
\label{sec:overview:feedback}

Gradient descent through the lens of Coherent Gradients is a mechanism for soft feature selection, or equivalently, for signal amplification. 
The over-parameterized linear regression model in Example~\ref{ex:two-datasets} provides a good illustration of this.
When training the model on dataset {\bf L} (``real"), we see that 
the weight of the common signal grows more rapidly since increasing it simultaneously reduces the loss on all training examples.

%

In this section, we shall see how the signal amplification effect becomes stronger 
when the model has multiple layers, 
due to the interaction between the parameters belonging to the different layers.
We illustrate this with two simple examples that build on Example~\ref{ex:two-datasets} to add depth.

\begin{example}[Adding Depth to Linear Regression]
\label{ex:two-deep}
Instead of fitting the a model of the form $y = w \cdot x$ as is the case in linear regression, consider fitting the following ``deep" model:
\[
y = ( w \odot u ) \cdot x = \sum_{j = 1}^{6} w^{(j)} h^{(j)} = \sum_{j = 1}^{6} w^{(j)} u^{(j)} x^{(j)}
\]
\begin{center}
\resizebox{0.75\textwidth}{!}
{
\vspace{1cm}
\pgfmathsetmacro{\xsphp}{3}
\pgfmathsetmacro{\ysphp}{3}
\begin{tikzpicture}

\foreach \i in {1,...,6}
{
    \node[place] at ({\xsphp * (\i - 1)}, 0) (x\i) {$x^{(\i)}$};
}

\foreach \i in {1,...,6}
{
    \node[place] at ({\xsphp * (\i - 1)}, \ysphp) (h\i) {$h^{(\i)}$};
}

\node[place] at (\xsphp * 5 / 2, 2 * \ysphp) (y) {$y$}; 


\node at (\xsphp * 0.3, 1.5 * \ysphp) () {$\displaystyle h^{(j)} =  u^{(j)}\,x^{(j)}$};

\foreach \i in {1,...,4}
{
    \draw (h\i) edge [post, pink] (y);
    \path (h\i) to node [weight, text=pink] {$w^{(\i)}$} (y); 
    
    \draw (x\i) edge [post, pink] (h\i);
    \path (x\i) to node [weight, text=pink] {$u^{(\i)}$} (h\i); 
}

\draw (h5) edge [post, pink] (y);
\path (h5) to node [weight, text=pink] {$w^{(5)}$} (y); 

\draw (x5) edge [post, pink] (h5);
\path (x5) to node [weight, text=pink] {$u^{(5)}$} (h5); 

\draw (h6) edge [post, blue] (y);
\path (h6) to node [weight, text=blue] {$w^{(6)}$} (y); 

\draw (x6) edge [post, blue] (h6);
\path (x6) to node [weight, text=blue] {$u^{(6)}$} (h6); 

\end{tikzpicture}
\vspace{1cm}
}
\end{center}
under the square loss  $\ell(w) \equiv \frac{1}{2} (y - ( w \odot u ) \cdot x)^2$ to 
the dataset {\bf L} (``real") from Example~\ref{ex:two-datasets}:

\begin{center}
\begin{tabular}{r|c|c}
$i$ & $x_i$ & $y_i$ \\
\hline
\hline
$1$ & $\langle\,\textcolor{pink}{1}, \textcolor{pink}{\mz}, \textcolor{pink}{\mz}, \textcolor{pink}{\mz}, \textcolor{pink}{\mz}, \textcolor{blue}{\mo}\,\rangle$ & $\mo$ \\
$2$ & $\langle\,\textcolor{pink}{0}, \textcolor{pink}{\mn}, \textcolor{pink}{\mz}, \textcolor{pink}{\mz}, \textcolor{pink}{\mz}, \textcolor{blue}{\mn}\,\rangle$ & $\mn$ \\
$3$ & $\langle\,\textcolor{pink}{0}, \textcolor{pink}{\mz}, \textcolor{pink}{\mn}, \textcolor{pink}{\mz}, \textcolor{pink}{\mz}, \textcolor{blue}{\mn}\,\rangle$ & $\mn$ \\
$4$ & $\langle\,\textcolor{pink}{0}, \textcolor{pink}{\mz}, \textcolor{pink}{\mz}, \textcolor{pink}{\mo}, \textcolor{pink}{\mz}, \textcolor{blue}{\mo}\,\rangle$ & $\mo$ \\
\hline
$5$ & $\langle\,\textcolor{pink}{0}, \textcolor{pink}{\mz}, \textcolor{pink}{\mz}, \textcolor{pink}{\mz}, \textcolor{pink}{\mn}, \textcolor{blue}{\mn}\,\rangle$ & $\mn$ \\
\end{tabular}
\end{center}

Now start gradient descent (GD) with a constant learning rate of 0.1 at $w^{(j)} = u^{(j)} = 0.01$ for $j \in [6]$.%
\footnote{Unlike linear regression, we use a near-zero initialization for the deep model rather than a zero initialization since the latter is a stationary point.}
As in the case of linear regression, the parameters corresponding to the common feature $x^{(6)}$, that is, $u^{(6)}$ and $w^{(6)}$ grow more quickly than those corresponding to the idiosyncratic features $x^{(1)}, \dots, x^{(4)}$. But what is striking is the extent of the relative difference.
This can be seen by comparing the plots of their values in the two cases:%
\footnote{We only show $w^{(1)}$ and $w^{(6)}$ since, along the trajectory of GD, from symmetry, $u^{(j)} = w^{(j)}$ for all $j \in [6]$ in the deep model and $w^{(2)}, w^{(3)},$ and $w^{(4)}$ are equal to $w^{(1)}$ in both models.}

\begin{center}
\begin{tabular}{cc}
deep model & linear regression \\
\includegraphics[width=0.45\textwidth]{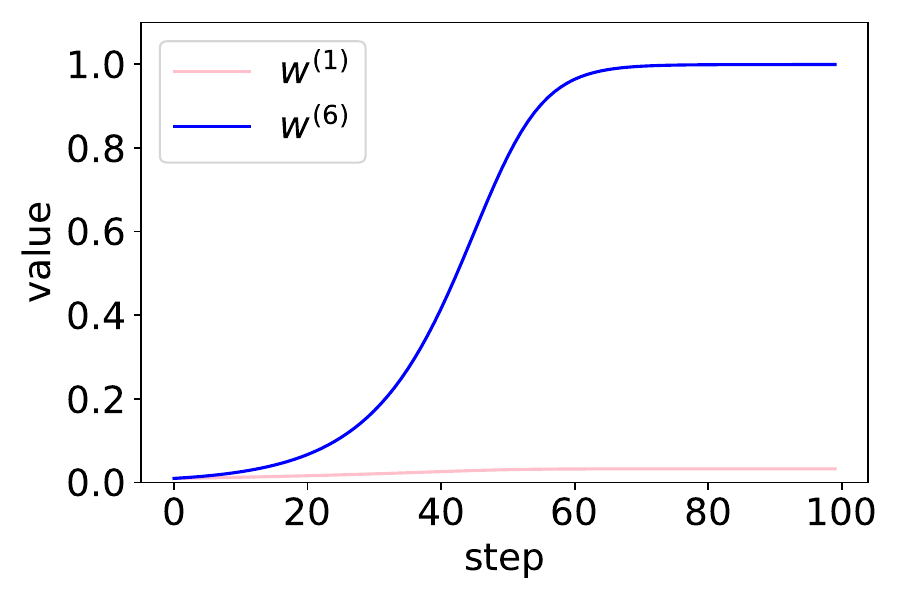} 
&
\includegraphics[width=0.45\textwidth]{plots/fresh_plotting/product_common_signal_w1_w6_value_v12.pdf} 
\\
\end{tabular}
\end{center}
In fact, for the deep model, the difference between $w^{(1)}$ and $w^{(6)}$ understates the importance of the corresponding features since the weight of feature $i$ in the model is given by $u^{(j)} w^{(j)}$ (which, in this case, is $(w^{(j)})^2$ due to symmetry).

The improvement in amplification efficiency is also reflected in the generalization gap (using the 5th example in {\bf L} as the test example) which is negligible for  the deep model in comparison to that of linear regression: 

\begin{center}
\begin{tabular}{cc}
deep model & linear regression \\
\includegraphics[width=0.45\textwidth]{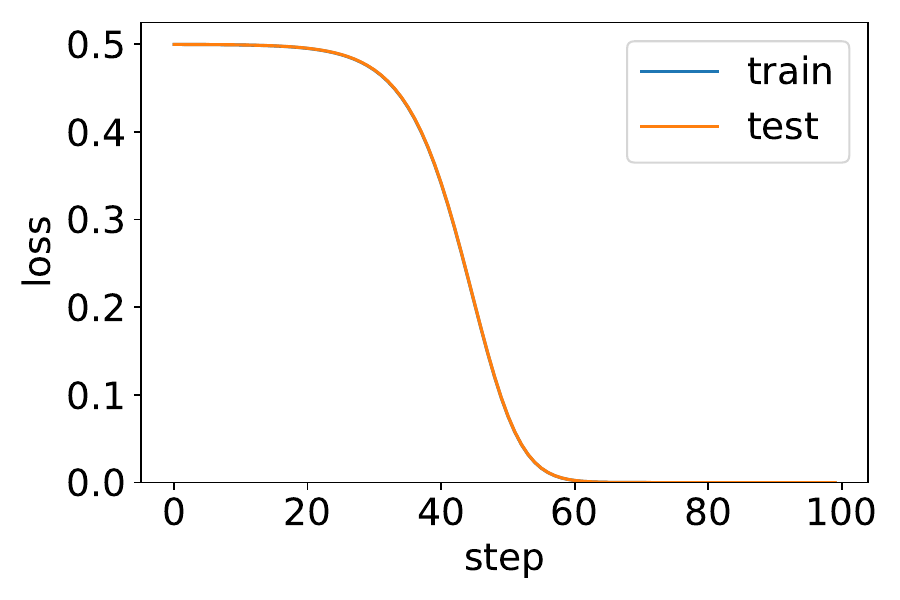}
&
\includegraphics[width=0.45\textwidth]{plots/fresh_plotting/product_common_signal_loss_v12.pdf} 
\\
\end{tabular}
\end{center}

\noindent This naturally leads to the following question:
\boldcenter{
Why is the signal amplification so much greater in the case of the deep model compared to linear regression?
}
The answer lies in the interaction of the weights across the two layers. Consider the
components of the gradient corresponding to $u^{(j)}$ and $w^{(j)}$ for the $i$th example. These are, respectively, 
\[
\frac{\partial \ell_i}{\partial \textcolor{red}{u^{(j)}}} = r_i(u, w)\,x^{(j)}\,\textcolor{red}{w^{(j)}}
\quad
\quad
\quad
\textrm{and}
\quad
\quad
\quad
\frac{\partial \ell_i}{\partial \textcolor{red}{w^{(j)}}} = r_i(u, w)\,x^{(j)}\,\textcolor{red}{u^{(j)}}
\]
where $r_i(u, w) \equiv y_i - ( w \odot u ) \cdot x_i$ is the residual.
Observe that the component of the gradient corresponding to $u^{(j)}$ is directly proportional to $w^{(j)}$
and vice-versa. 
By slight abuse of terminology we call this a {\em positive feedback loop} between $u^{(j)}$ and $w^{(j)}$.\footnote{It is an abuse of terminology since the positive feedback is not absolute. There is a second ``dampening" dependency through the residual which becomes more important closer to convergence.}

Now consider the gradient $g_i(u, w)$ for the $i$th example, and the average gradient $g(u, w)$:%
\footnote{$r(u, w)$ is equal to $|r_i(u, w)|$%
, a quantity that is independent of $i$ for $(u, w)$ in the trajectory of GD due to the symmetries in the problem.}

\newdimen\toprowlength
\toprowlength=26.4pt
\newcommand{\PTL}[1]{\parbox{\toprowlength}{\raggedleft #1}}

\newdimen\examplerowlength
\examplerowlength=30pt
\newcommand{\PT}[1]{\parbox{\examplerowlength}{\raggedleft #1}}

\begin{center}

\resizebox{0.85\textwidth}{!}
{

\begin{tabular}{c|c}
\multicolumn{2}{c}{} \\
$i$ & $g_i(u, w)=\,\langle $ \PTL{$ \frac{\partial \ell_i}{\partial {u^{(1)}}} $,} \PTL{$ \frac{\partial \ell_i}{\partial {u^{(2)}}} $,} \PTL{$ \frac{\partial \ell_i}{\partial {u^{(3)}}} $,} \PTL{$ \frac{\partial \ell_i}{\partial {u^{(4)}}} $,} \PTL{$ \frac{\partial \ell_i}{\partial {u^{(5)}}} $,} \PTL{$ \frac{\partial \ell_i}{\partial {u^{(6)}}} $,} \PTL{$ \frac{\partial \ell_i}{\partial {w^{(1)}}} $,} \PTL{$ \frac{\partial \ell_i}{\partial {w^{(2)}}} $,} \PTL{$ \frac{\partial \ell_i}{\partial {w^{(3)}}} $,} \PTL{$ \frac{\partial \ell_i}{\partial {w^{(4)}}} $,} \PTL{$ \frac{\partial \ell_i}{\partial {w^{(5)}}} $,} \PTL{$ \frac{\partial \ell_i}{\partial {w^{(6)}}}$}$\rangle$ \\
\\[-1em]
\hline
\hline
\rule{0pt}{2.5ex}$1$ & \phantom{$aa,$} $r(u, w)\, \langle $\PT{\textcolor{pink}{$w^{(1)}$},}$ $\PT{\textcolor{pink}{$0$},}$ $\PT{\textcolor{pink}{$0$},}$ $\PT{\textcolor{pink}{$0$},}$ $\PT{\textcolor{pink}{$0$},}$ $\PT{\textcolor{blue}{$w^{(6)}$},}$ $\PT{\textcolor{pink}{$u^{(1)}$},}$ $\PT{\textcolor{pink}{$0$},}$ $\PT{\textcolor{pink}{$0$},}$ $\PT{\textcolor{pink}{$0$},}$ $\PT{\textcolor{pink}{$0$},}$ $\PT{\textcolor{blue}{$u^{(6)}$}}$\rangle$ \\
$2$ & \phantom{$aa,$} $r(u, w)\, \langle $\PT{\textcolor{pink}{$0$},}$ $\PT{\textcolor{pink}{$w^{(2)}$},}$ $\PT{\textcolor{pink}{$0$},}$ $\PT{\textcolor{pink}{$0$},}$ $\PT{\textcolor{pink}{$0$},}$ $\PT{\textcolor{blue}{$w^{(6)}$},}$ $\PT{\textcolor{pink}{$0$},}$ $\PT{\textcolor{pink}{$u^{(2)}$},}$ $\PT{\textcolor{pink}{$0$},}$ $\PT{\textcolor{pink}{$0$},}$ $\PT{\textcolor{pink}{$0$},}$ $\PT{\textcolor{blue}{$u^{(6)}$}}$\rangle$ \\
$3$ & \phantom{$aa,$} $r(u, w)\, \langle $\PT{\textcolor{pink}{$0$},}$ $\PT{\textcolor{pink}{$0$},}$ $\PT{\textcolor{pink}{$w^{(3)}$},}$ $\PT{\textcolor{pink}{$0$},}$ $\PT{\textcolor{pink}{$0$},}$ $\PT{\textcolor{blue}{$w^{(6)}$},}$ $\PT{\textcolor{pink}{$0$},}$ $\PT{\textcolor{pink}{$0$},}$ $\PT{\textcolor{pink}{$u^{(3)}$},}$ $\PT{\textcolor{pink}{$0$},}$ $\PT{\textcolor{pink}{$0$},}$ $\PT{\textcolor{blue}{$u^{(6)}$}}$\rangle$ \\
$4$ & \phantom{$aa,$} $r(u, w)\, \langle $\PT{\textcolor{pink}{$0$},}$ $\PT{\textcolor{pink}{$0$},}$ $\PT{\textcolor{pink}{$0$},}$ $\PT{\textcolor{pink}{$w^{(4)}$},}$ $\PT{\textcolor{pink}{$0$},}$ $\PT{\textcolor{blue}{$w^{(6)}$},}$ $\PT{\textcolor{pink}{$0$},}$ $\PT{\textcolor{pink}{$0$},}$ $\PT{\textcolor{pink}{$0$},}$ $\PT{\textcolor{pink}{$u^{(4)}$},}$ $\PT{\textcolor{pink}{$0$},}$ $\PT{\textcolor{blue}{$u^{(6)}$}}$\rangle$ \\
\hline
\hline
\rule{0pt}{2.5ex} $g(u, w)$ & \phantom{$aa,$} $r(u, w)\, \langle $\PT{\textcolor{pink}{$\frac{1}{4}w^{(1)}$},}$ $\PT{\textcolor{pink}{$\frac{1}{4}w^{(2)}$},}$ $\PT{\textcolor{pink}{$\frac{1}{4}w^{(3)}$},}$ $\PT{\textcolor{pink}{$\frac{1}{4}w^{(4)}$},}$ $\PT{\textcolor{pink}{$0$},}$ $\PT{\textcolor{blue}{$w^{(6)}$},}$ $\PT{\textcolor{pink}{$\frac{1}{4}u^{(1)}$},}$ $\PT{\textcolor{pink}{$\frac{1}{4}u^{(2)}$},}$ $\PT{\textcolor{pink}{$\frac{1}{4}u^{(3)}$},}$ $\PT{\textcolor{pink}{$\frac{1}{4}u^{(4)}$},}$ $\PT{\textcolor{pink}{$0$},}$ $\PT{\textcolor{blue}{$u^{(6)}$}}$\rangle$ \\
\end{tabular}
}
\phantom{M}\\
\phantom{M}\\
\end{center}

\noindent We see that, just as in linear regression with dataset {\bf L} (Example~\ref{ex:two-datasets}), 
the per-example gradients in the common component (shown in \textcolor{blue}{blue}) reinforce each other, and add up in the average gradient $g(u, w)$, but the idiosyncratic components (shown in \textcolor{pink}{pink}) do not.%
\footnote{This is also reflected in the higher coherence of the $u^{(6)}$ and $w^{(6)}$ components compared to the rest as measured by (component-wise) $\alpha$.}

However, in contrast to linear regression where the {\em relative} difference in the gradient directions was constant during training (1 v/s $\frac{1}{4}$), in the deep model, the relative difference depends on $u$ and $w$ (for example, $w^{(6)}$ v/s $\frac{1}{4}w^{(1)}$ and $u^{(6)}$ v/s $\frac{1}{4}u^{(1)}$). 
This combined with the positive feedback loop between the $u^{(i)}$ and the corresponding $w^{(i)}$ leads to the relative difference between the gradient directions {\em exponentially  increasing} as training progress. 

As a result, the relative importance of the common signal increases much faster in the deep model than in the linear regression case leading to a winner-take-all situation.
In summary, 
\boldcenter{
In the deep model, in contrast to linear regression, the relative difference 
between the common and idiosyncratic gradient directions gets exponentially amplified during training  due to positive feedback between layers.
}

This rapid growth in the strong component of the gradient is reflected in the plot of training coherence as measured by $\ralpha$ ($m = 4$) over time 
since the $u^{(6)}$ and $w^{(6)}$ components (the strong, or the high coherence components) come to dominate the gradient for the deep model, and they have perfect coherence:

\begin{center}
\begin{tabular}{cc}
deep model & linear regression \\
\includegraphics[width=0.45\textwidth]{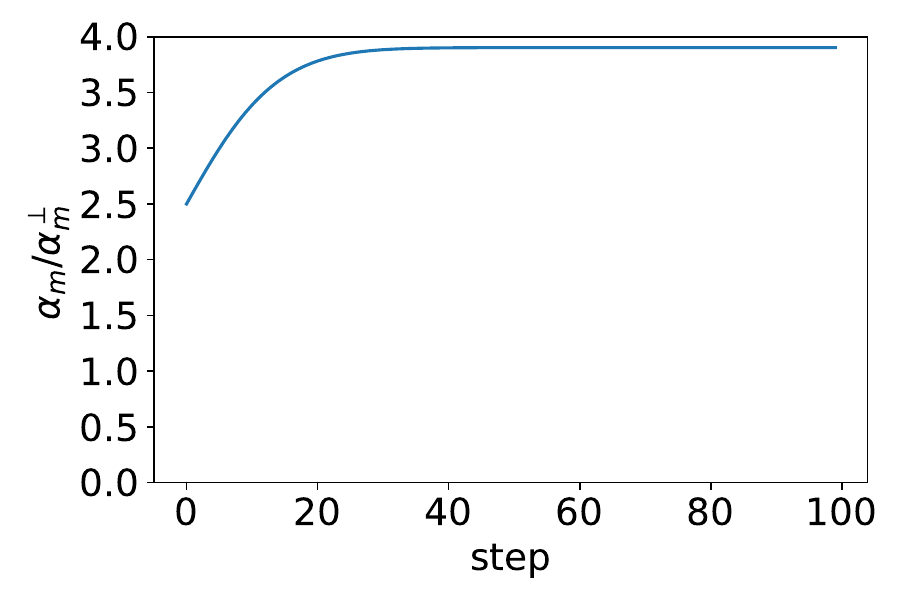}
&
\includegraphics[width=0.45\textwidth]{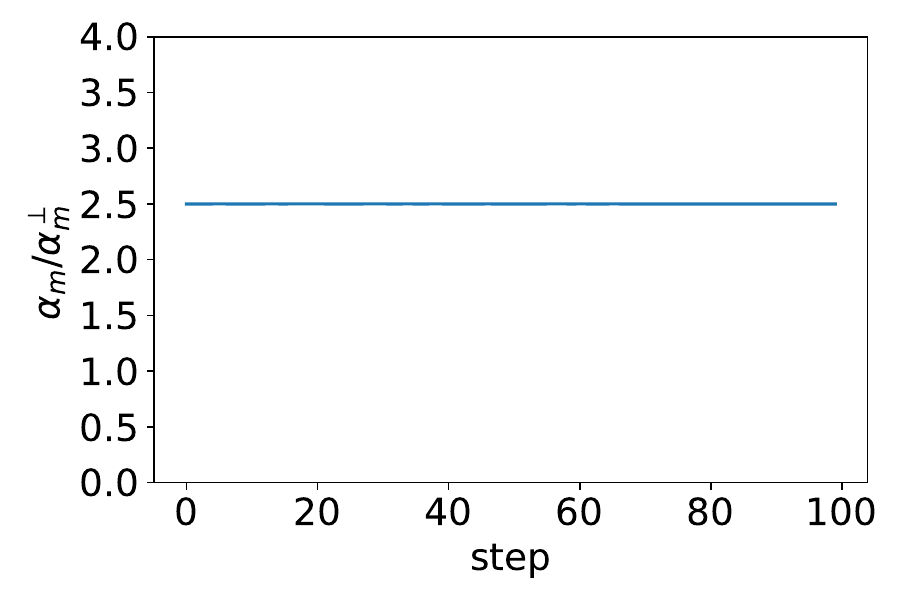} 
\\
\end{tabular}
\end{center}

\noindent Thus although $\ralpha$ for the deep model starts off at 2.5 (exactly the same as for linear regression), it rises to near maximum of 4 (while in the linear regression case it stays constant).

\qed
\end{example}

The previous example showed how feedback loops due to depth can amplify input signals or features that generalize better relative to noise signals. 
Feedback loops can also play a critical role in differentially amplifying {\em internal} signals in the network that generalize better than their peers.
To see this, consider the following thought experiment that builds on Example~\ref{ex:two-datasets} by adding a second layer to select between a ``memorization" neuron and a ``learning" neuron.
    
\begin{example}[A Tale of Two Neurons]
\label{ex:two-neurons}

Consider the following linear neural network with one hidden layer with two neurons: 

\begin{center}
\resizebox{0.75\textwidth}{!}
{
\vspace{1cm}
\pgfmathsetmacro{\xsphp}{3}
\pgfmathsetmacro{\ysphp}{3}
\begin{tikzpicture}

\foreach \i in {1,...,6}
{
    \node[place] at ({\xsphp * (\i - 1)}, 0) (x\i) {$x^{(\i)}$};
}

\node[place] at (\xsphp * 1, \ysphp) (h1) {$h^{(1)}$};
\node[place] at (\xsphp * 4, \ysphp) (h2) {$h^{(2)}$};

\node[place] at (\xsphp * 5 / 2, 2 * \ysphp) (y) {$y$}; 

\node at (\xsphp * 0.5, 1.5 * \ysphp) () {$\displaystyle h^{(1)} = \sum_{j = 1}^{5} u^{(j)}\,x^{(j)}$};

\node at (\xsphp * 4.5, 1.5 * \ysphp) () {$\displaystyle h^{(2)} = \sum_{j = 1}^{6} v^{(j)}\,x^{(j)}$};

\node at (\xsphp * 1.5, 2 * \ysphp) () {$\displaystyle y =  \sum_{j = 1}^{2} w^{(j)}\,h^{(j)}$};


\draw (h1) edge [post,pink] (y);
\path (h1) to node [weight,text=pink] {$w^{(1)}$} (y); 

\draw (h2) edge [post,blue] (y);
\path (h2) to node [weight,text=blue] {$w^{(2)}$} (y); 

\foreach \i in {1,...,4}
{
    \draw (x\i) edge [post, pink] (h1);
    \draw (x\i) edge [post, pink] (h2);
}
\draw (x5) edge [post, pink] (h1);
\draw (x5) edge [post, pink] (h2);
\draw (x6) edge [dotted, blue] (h1);
\draw (x6) edge [post, blue] (h2);

\path (x1) to node [weight, pos=0.60, text=pink] {$u^{(1)}$} (h1); 
\path (x1) to node [weight, pos=0.10, text=pink] {$v^{(1)}$} (h2); 

\path (x2) to node [weight, pos=0.60, text=pink] {$u^{(2)}$} (h1); 
\path (x2) to node [weight, pos=0.10, text=pink] {$v^{(2)}$} (h2); 

\path (x3) to node [weight, pos=0.60, text=pink] {$u^{(3)}$} (h1); 
\path (x3) to node [weight, pos=0.60, text=pink] {$v^{(3)}$} (h2); 

\path (x4) to node [weight, pos=0.65, text=pink] {$u^{(4)}$} (h1); 
\path (x4) to node [weight, pos=0.60, text=pink] {$v^{(4)}$} (h2); 

\path (x5) to node [weight, pos=0.10, text=pink] {$u^{(5)}$} (h1); 
\path (x5) to node [weight, pos=0.50, text=pink] {$v^{(5)}$} (h2); 

\path (x6) to node [weight, pos=0.10, text=red] {\xmark} (h1); 
\path (x6) to node [weight, pos=0.50, text=blue] {$v^{(6)}$} (h2); 

\end{tikzpicture}

\vspace{1cm}
}
\end{center}

\noindent Now consider the task of fitting it to the following dataset (dataset {\bf L} (``real") from Example~\ref{ex:two-datasets} that we have been using as a running example):

\begin{center}
\begin{tabular}{r|c|c}
$i$ & $x_i$ & $y_i$ \\
\hline
\hline
$1$ & $\langle\,\textcolor{pink}{1}, \textcolor{pink}{\mz}, \textcolor{pink}{\mz}, \textcolor{pink}{\mz}, \textcolor{pink}{\mz}, \textcolor{blue}{\mo}\,\rangle$ & $\mo$ \\
$2$ & $\langle\,\textcolor{pink}{0}, \textcolor{pink}{\mn}, \textcolor{pink}{\mz}, \textcolor{pink}{\mz}, \textcolor{pink}{\mz}, \textcolor{blue}{\mn}\,\rangle$ & $\mn$ \\
$3$ & $\langle\,\textcolor{pink}{0}, \textcolor{pink}{\mz}, \textcolor{pink}{\mn}, \textcolor{pink}{\mz}, \textcolor{pink}{\mz}, \textcolor{blue}{\mn}\,\rangle$ & $\mn$ \\
$4$ & $\langle\,\textcolor{pink}{0}, \textcolor{pink}{\mz}, \textcolor{pink}{\mz}, \textcolor{pink}{\mo}, \textcolor{pink}{\mz}, \textcolor{blue}{\mo}\,\rangle$ & $\mo$ \\
\hline
$5$ & $\langle\,\textcolor{pink}{0}, \textcolor{pink}{\mz}, \textcolor{pink}{\mz}, \textcolor{pink}{\mz}, \textcolor{pink}{\mn}, \textcolor{blue}{\mn}\,\rangle$ & $\mn$ \\
\end{tabular}
\end{center}

\noindent Of the two neurons in the hidden layer, one, 
the ``good" neuron, $h^{(2)}$, 
is connected to the common input feature $x^{(6)}$, whereas, the other, the ``bad" neuron, $h^{(1)}$, is not.
Therefore, the bad neuron can only help in memorizing the training data.%
\footnote{In terms of Example~\ref{ex:two-datasets}, the bad neuron ``sees" the random dataset {\bf M} instead of {\bf L}.}

When we perform gradient descent starting, as before, from 0.01 for all parameters, we find that the weight for the good neuron (\textcolor{blue}{$w^{(2)}$}) grows much more rapidly than that for the bad neuron (\textcolor{pink}{$w^{(1)}$}) and we have good generalization (as measured on $(x_5, y_5)$):

\begin{center}
\begin{tabular}{ccc}
\includegraphics[width=0.30\textwidth]{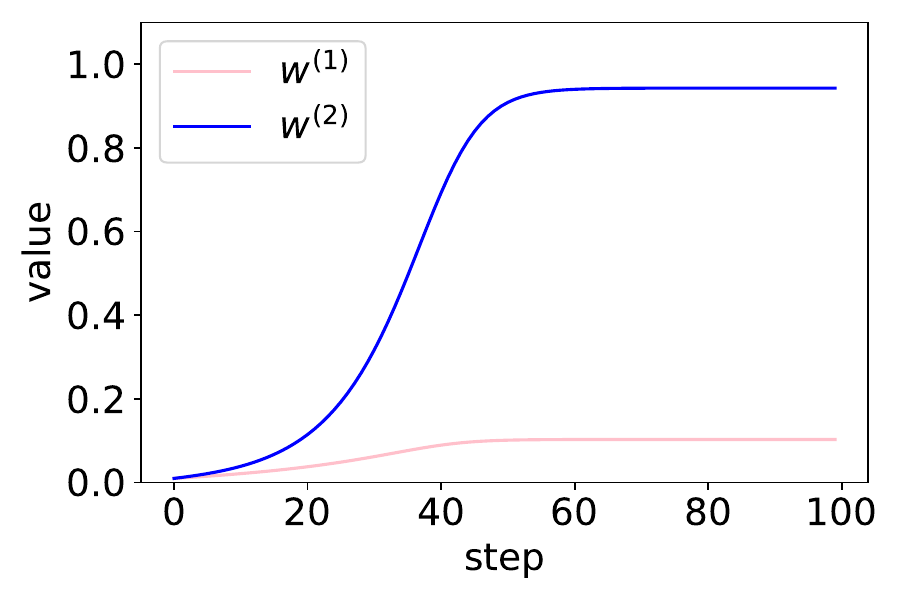}
&
\includegraphics[width=0.30\textwidth]{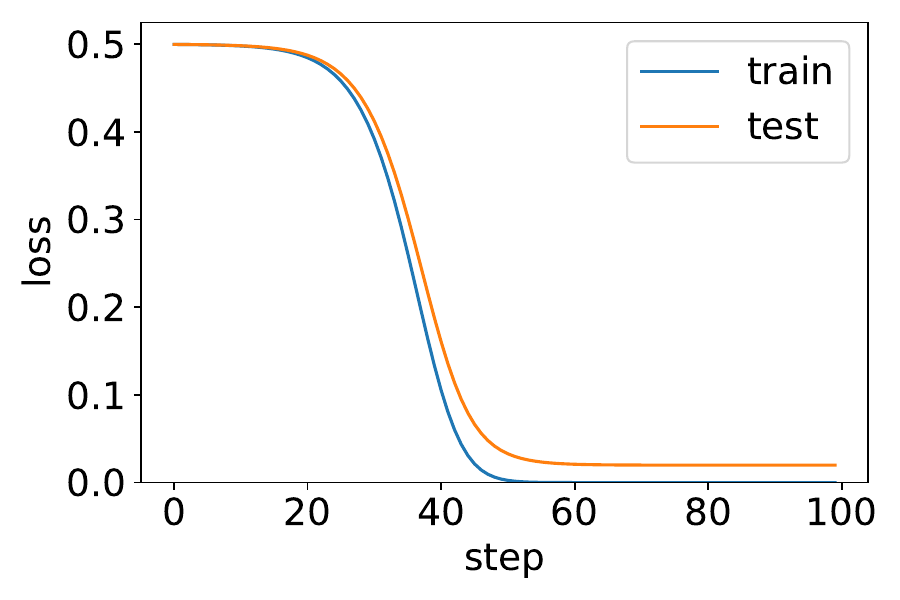}
&
\includegraphics[width=0.30\textwidth]{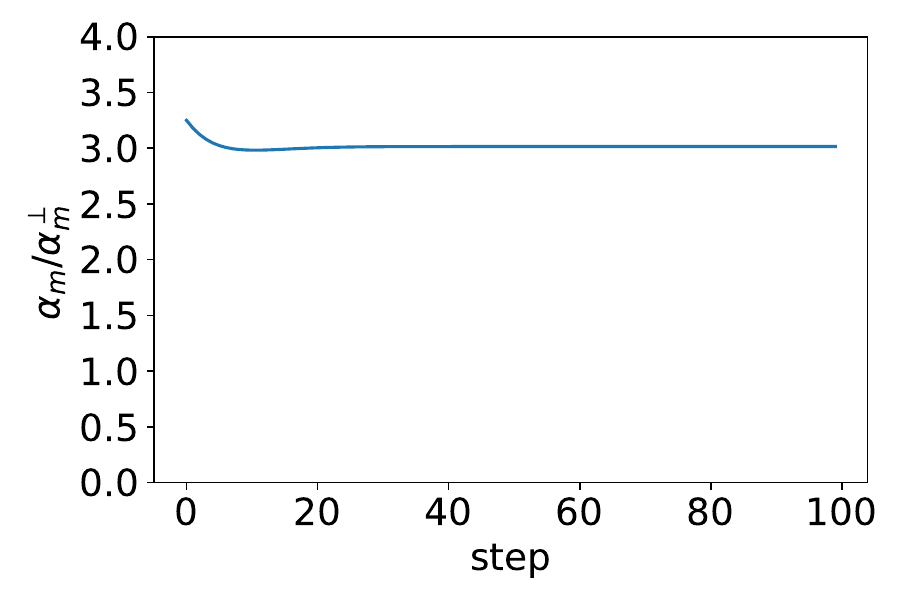}
\\
\end{tabular}
\end{center}

As in the previous example, the reason behind the rapid increase in \textcolor{blue}{$w^{(2)}$}  relative to \textcolor{pink}{$w^{(1)}$} is due to the interaction between the layers. 
In brief, there is a positive feedback loop%
\footnote{As in the previous example, there is a second dampening dependence through the residual that becomes important near convergence.} 
between \textcolor{pink}{$w^{(1)}$} and each of \textcolor{pink}{$u^{(1)}, \dots, u^{(5)}$} and between \textcolor{blue}{$w^{(2)}$} and each of $\textcolor{pink}{v^{(1)}, \dots, v^{(5)}}$. However, in addition, \textcolor{blue}{$w^{(2)}$} also has a feedback loop with $\textcolor{blue}{v^{(6)}}$, a component with high coherence since it is the weight for the common input feature $x^{(6)}$.
Consequently, \textcolor{blue}{$w^{(2)}$} grows exponentially faster than \textcolor{pink}{$w^{(1)}$}.%
\footnote{Indeed, due to the feedback loop, and the faster growth of $\textcolor{blue}{w^{(2)}}$, even \textcolor{pink}{$v^{(1)}, \dots, v^{(4)}$} grow faster relative to \textcolor{pink}{$u^{(1)}, \dots, u^{(4)}$}.}

\newbold{Adversarial initialization.} The relative amplification effect is so strong that even if  \textcolor{pink}{$w^{(1)}$} is very favorably initialized at 0.1 or 0.2 while keeping \textcolor{blue}{$w^{(2)}$} initialized at 0.01, in the course of training, \textcolor{blue}{$w^{(2)}$} eventually surpasses \textcolor{pink}{$w^{(1)}$} (we also show the test and train curves and $\ralpha$ on the training set ($m$ = 4) measured over the whole network for subsequent discussion):

\begin{center}
\begin{tabular}{ccc}
\includegraphics[width=0.30\textwidth]{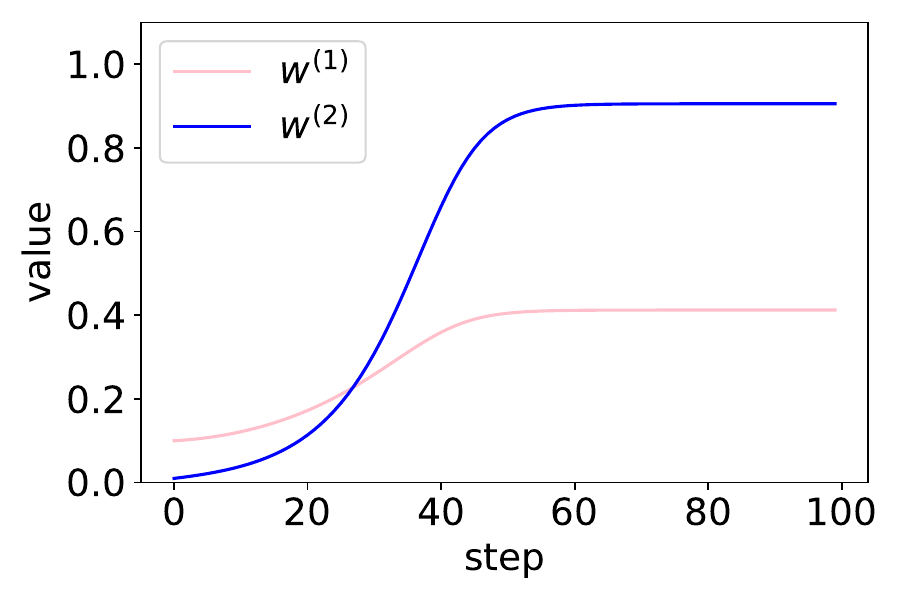}
&
\includegraphics[width=0.30\textwidth]{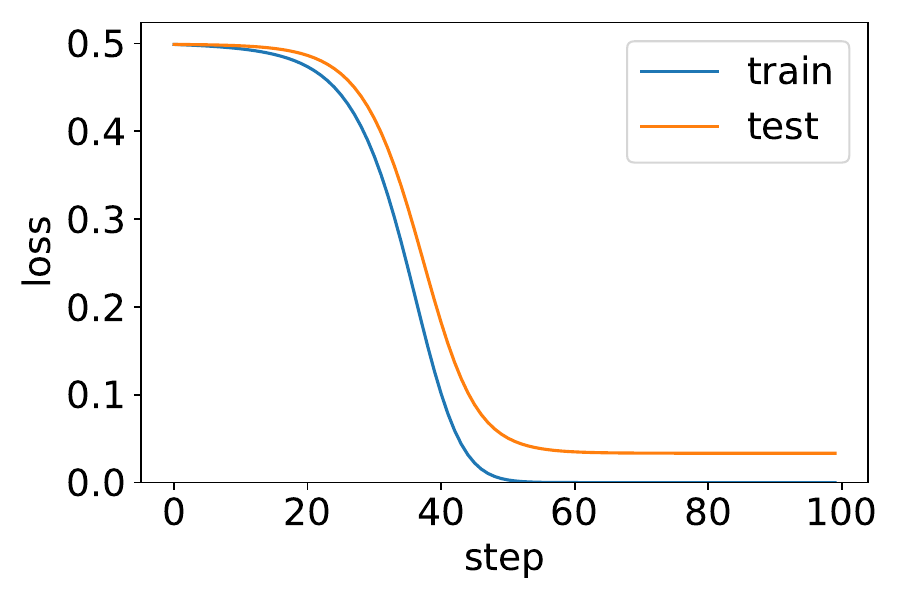}
&
\includegraphics[width=0.30\textwidth]{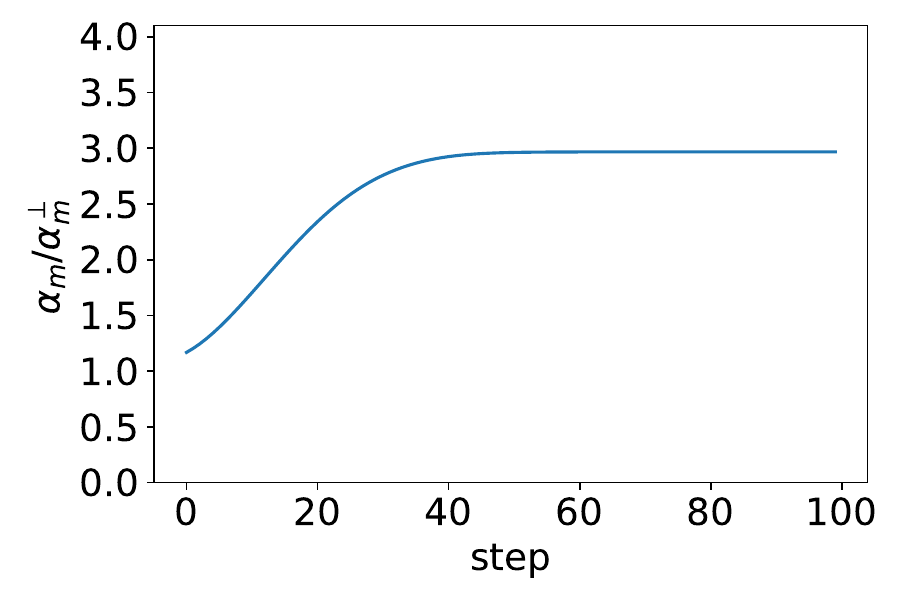}
\\
\includegraphics[width=0.30\textwidth]{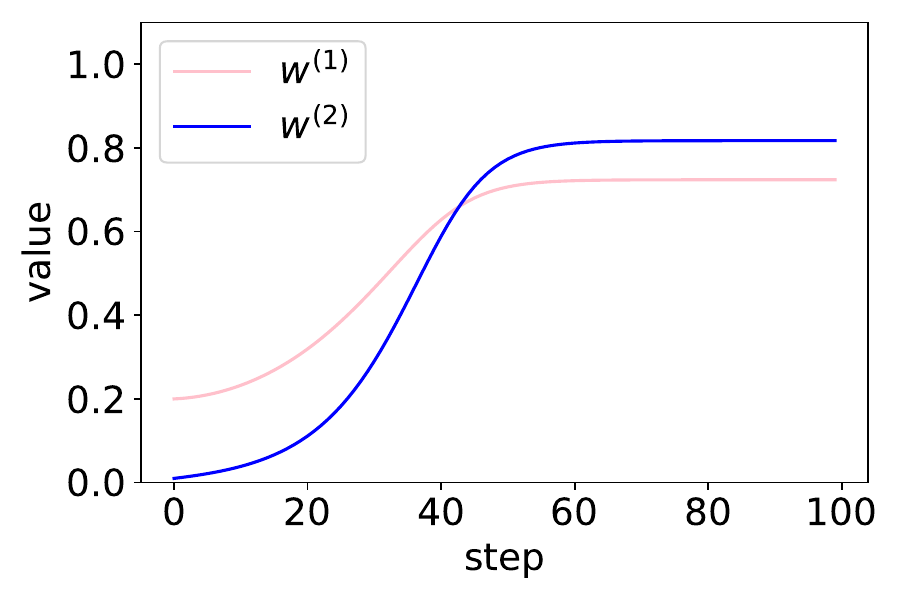}
&
\includegraphics[width=0.30\textwidth]{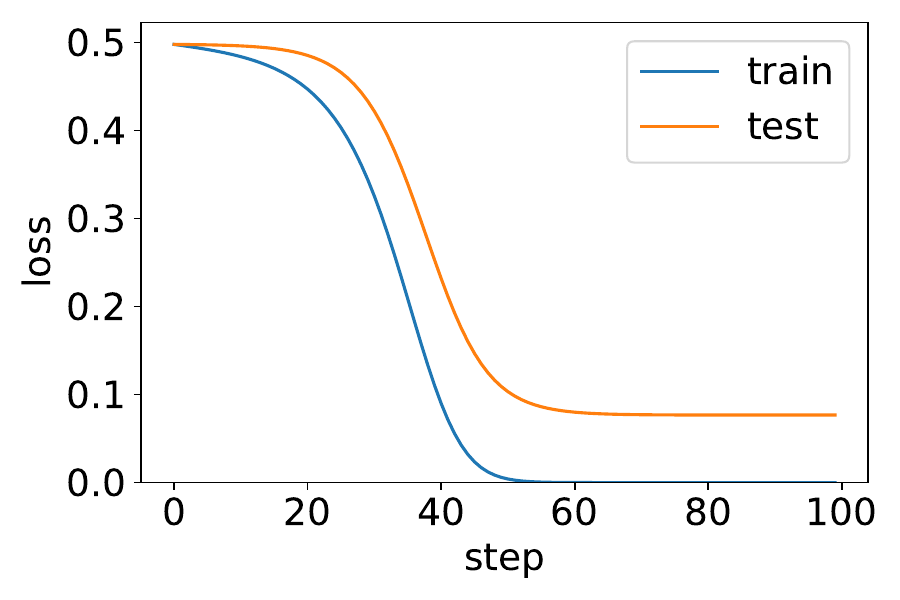}
&
\includegraphics[width=0.30\textwidth]{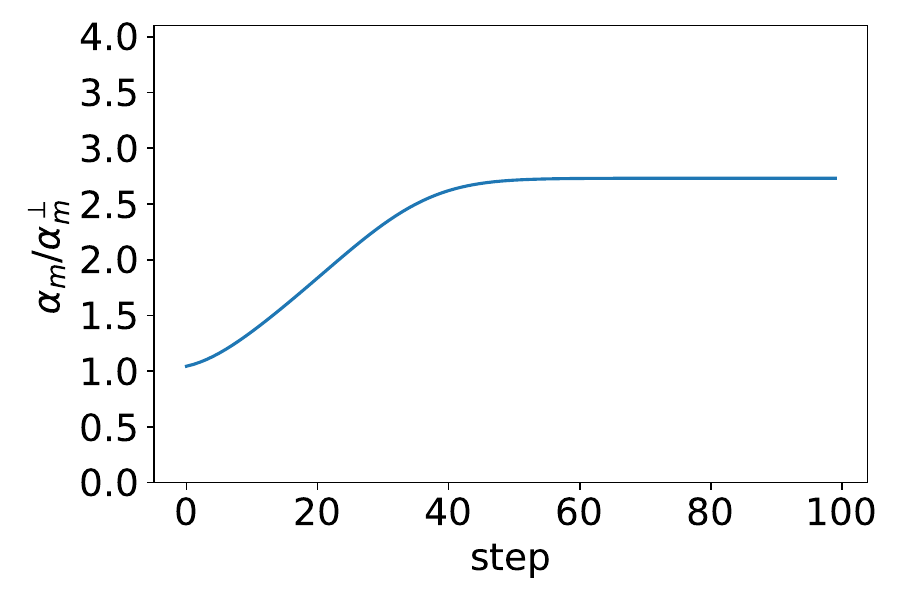}
\\
\end{tabular}
\end{center}

\noindent Of course, for a large enough adversarial initialization,  \textcolor{blue}{$w^{(2)}$} can no longer overcome \textcolor{pink}{$w^{(1)}$} and we have poor generalization:

\begin{center}
\begin{tabular}{ccc}
\includegraphics[width=0.30\textwidth]{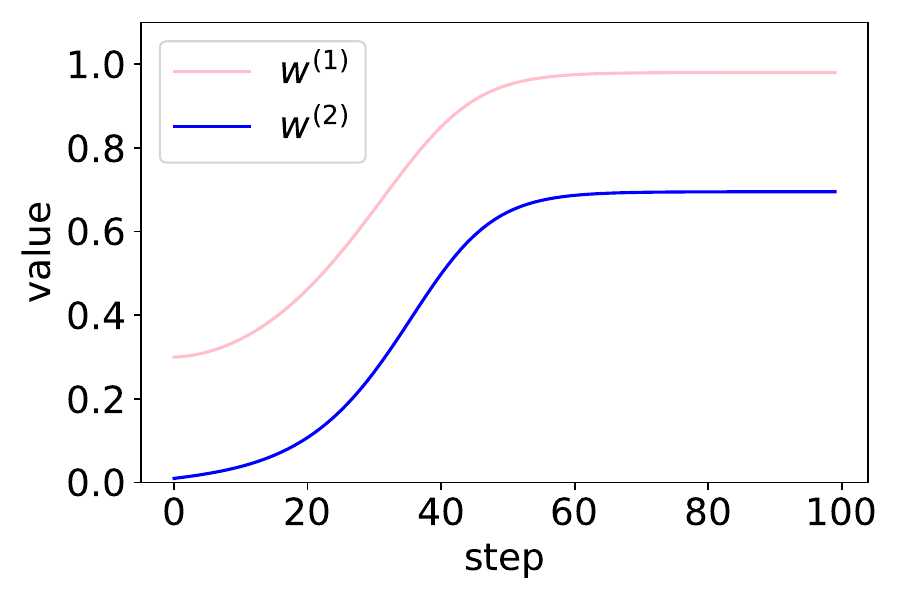}
&
\includegraphics[width=0.30\textwidth]{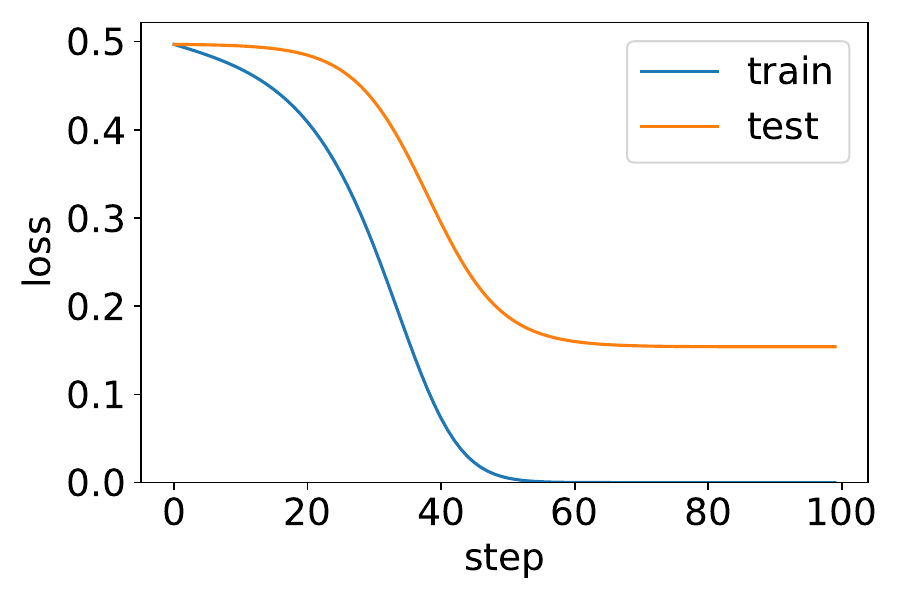}
&
\includegraphics[width=0.30\textwidth]{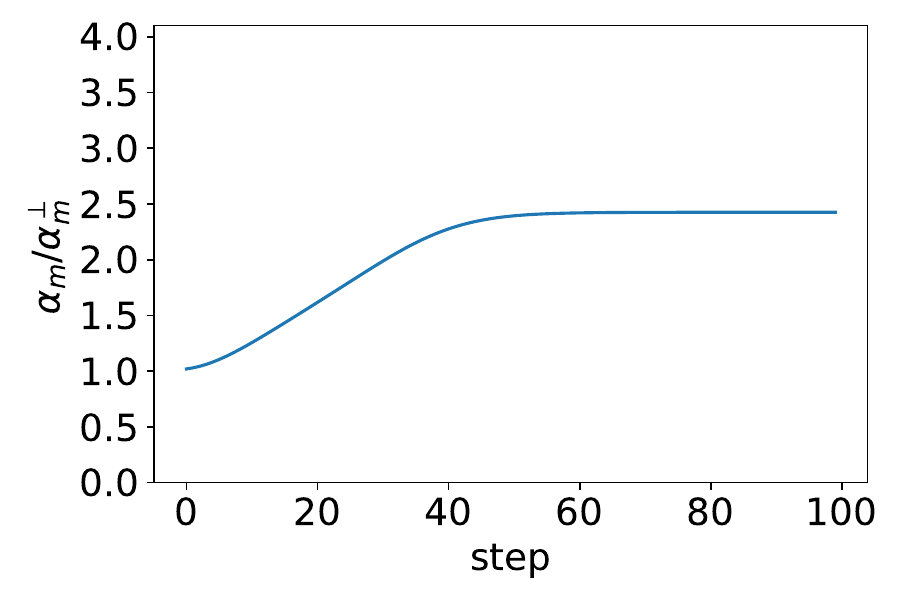}
\\
\end{tabular}
\end{center}

\noindent We can study this systematically by plotting a heatmap of the test loss at the end of training as a function of the initial values of \textcolor{pink}{$w^{(1)}$} and \textcolor{blue}{$w^{(2)}$} (all other parameters are initialized at 0.01):
\begin{center}
\includegraphics[width=0.50\textwidth]{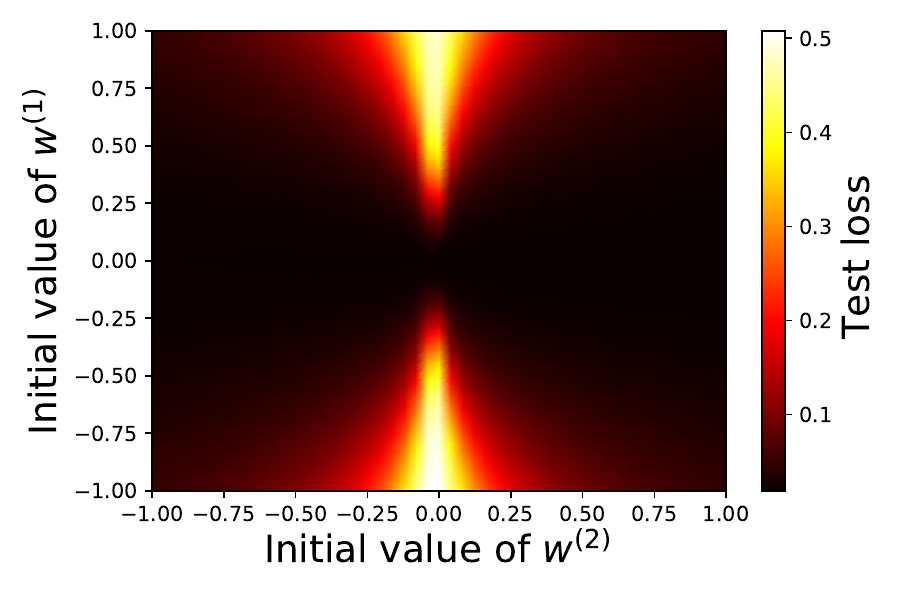}
\end{center}
The additional positive feedback loop between \textcolor{blue}{$w^{(2)}$} and $\textcolor{blue}{v^{(6)}}$ ensures that from almost all initializations, the common feature has greater importance in the final output than any idiosyncratic feature. 
In summary,
\boldcenter{
Due to interaction between layers, the hidden neuron with access to the common feature 
gets a much higher weight compared to its peer that does not.
}
In this manner, gradient descent amplifies intermediate features that generalize better.%
\footnote{Informally, a feature that generalizes better is one that is stable, that is, one that shows a smaller perturbation if the training set we slightly changed. A feature such as $h^{(1)}$ may be thought of as a  ``memorization" feature since it takes on very different values on a particular example depending on whether that example was in the training set or not.
In this context, the method of counterfactual simulation presented in \citet{Chatterjee20b}, a method to analyze the stability, and hence generalization of intermediate signals of a trained model, may be of interest. 
}

\qed

\end{example}

\newbold{Memorization and Layer-wise Coherence.} In the previous example, it may be tempting to think that the difference in generalization capability between the good neuron and the bad neuron is reflected in the training coherence 
of their respective weights.
That is not so.
Both \textcolor{pink}{$w^{(1)}$} and \textcolor{blue}{$w^{(2)}$} have the same coherence over the training set during gradient descent for any of the above trajectories, that is, both have a component-wise $\alpha$ of 1. Thus,
\boldcenter{
Coherence alone at a hidden layer cannot distinguish between a feature from the previous layer that does not generalize well and one that does.
}
\noindent This has two consequences for coherence measured on random data:
\begin{enumerate}
    \item The coherence of a given layer may be high if memorization happens in other layers. We believe that this is a likely reason for why high coherence is observed in AlexNet for the higher layers when it is trained on random data (Figure~\ref{fig:chap1:alexnetimagenet_layer}). 
    
    \item Coherence as measured over the whole network, that is, $\alpha$ computed using the entire gradient vector, can skew high since it is a weighted sum of the per-layer coherences, since many of the layers can have high coherence if memorization happens in only a few layers. This was seen in the case of AlexNet (Figure~\ref{fig:chap1:alexnetimagenet_overall}). 
    
\end{enumerate}

\begin{example}[Tale of Two Neurons with Memorization]
\label{ex:two-neurons-m}
Both these consequences for random data can be illustrated by training the model used in \Cref{ex:two-neurons} on the ``random" dataset {\bf M} of Example~\ref{ex:two-datasets} (instead of dataset {\bf L} as we have been doing so far):
\begin{center}
\begin{tabular}{ccc}
\includegraphics[width=0.30\textwidth]{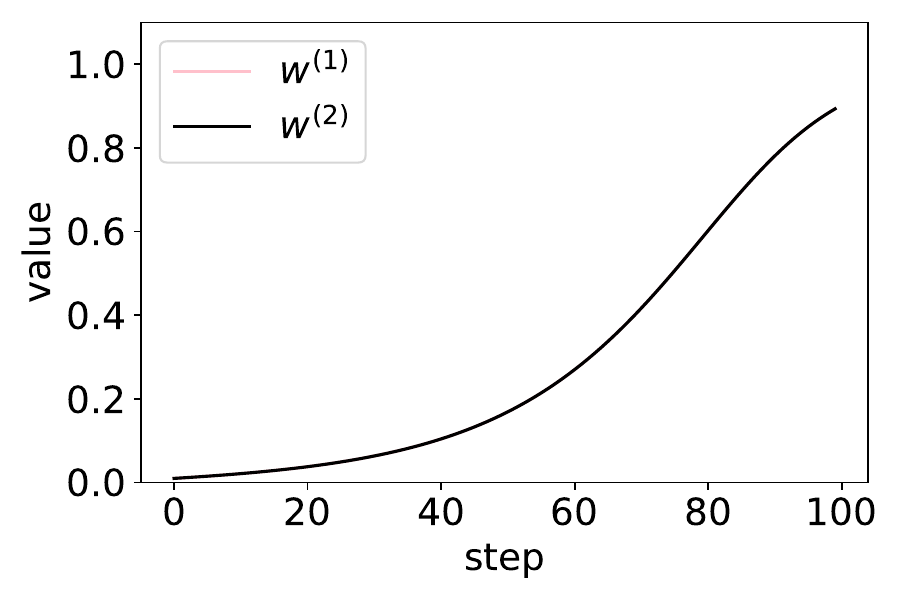}
&
\includegraphics[width=0.30\textwidth]{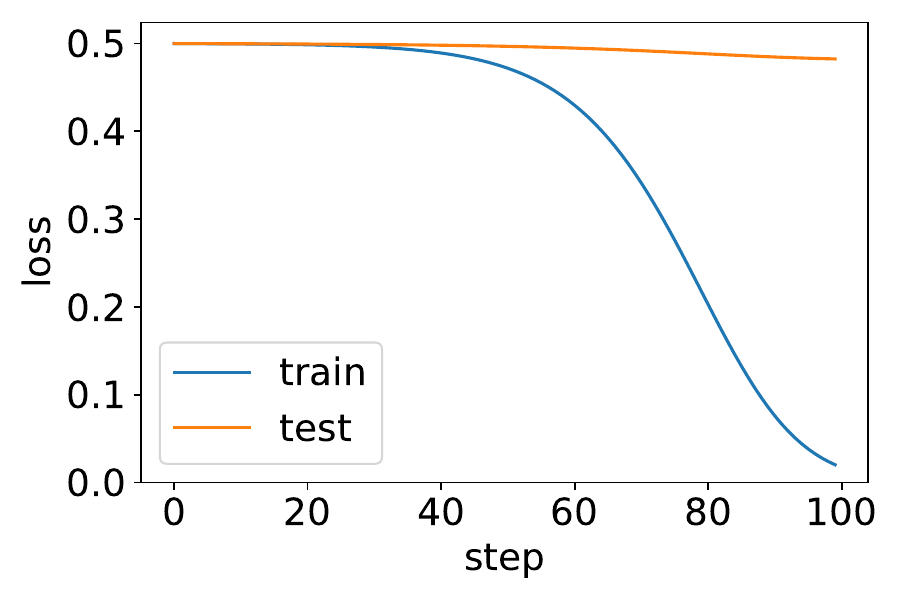}
&
\includegraphics[width=0.30\textwidth]{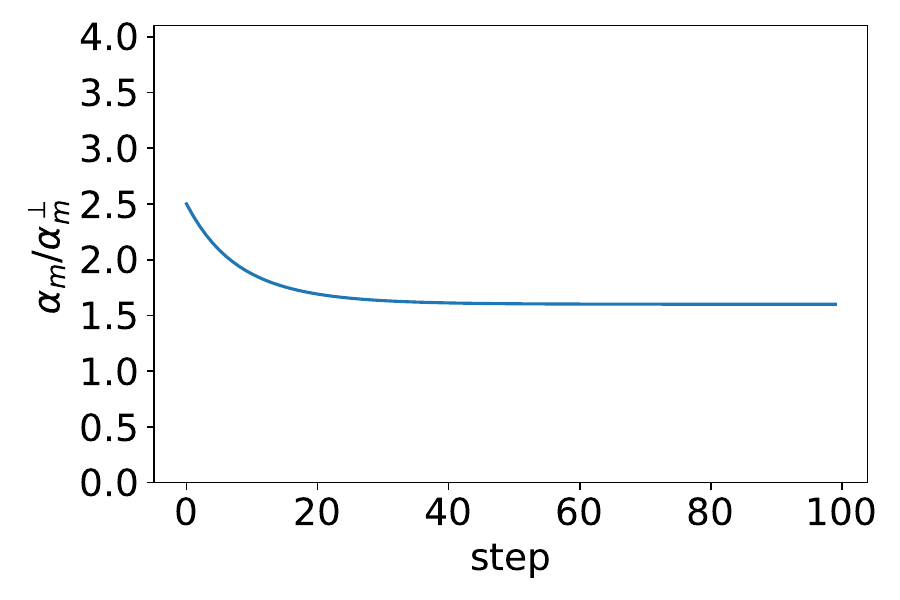}
\\
\end{tabular}
\end{center}
It is easy to check, on the training set of {\bf M} ($m = 4$),
that $\ralpha = 1$ for the first layer (where we may say memorization, or the case-by-case fitting, occurs) and $\ralpha = 4$ for the second layer (where all the examples are treated uniformly). 
Thus, the overall $\ralpha$ is greater than 1 (the orthogonal limit), as seen above, even with memorization.

\qed

\end{example}

The observation above, that is, coherence of the entire network can be high on random data, is not
specific to the $\alpha$ metric. 
It is indeed the case that the per-example gradients of different examples are similar in the components of the gradient corresponding to the higher (non-memorization) layers.
Therefore, even other metrics for coherence such as dot product and stiffness that depend on the entire gradient vector would be skewed upwards.

\thoughtsep

To summarize, the examples in this section show how depth can help with feature selection through positive feedback loops between layers. Due to these feedback loops, features whose weights have high coherence get amplified beyond what coherence alone can accomplish.


\section{What Should a Theory of Generalization Look Like?} 
The ultimate goal of a theory of generalization for deep learning would be to provide a tight guarantee on the generalization gap in problem setups of practical interest, but that appears out of reach at present. Even so, there are a number of different qualitative properties that such a theory should satisfy. For example, in the introduction, we already noted that the work of \citet{Zhang17} puts a very strong constraint on the form of any theory of generalization:

\begin{enumerate}[label={\bf [C\arabic*]}]

\item {\bf Uniform explanation of memorization and generalization.} \label{C1}
It must simultaneously be able to explain why on some datasets, there is good generalization, when the exact same setup (architecture, learning rate, number of steps) is capable of memorizing a random dataset. Thus any non-trivial generalization bound must depend on the dataset in some non-trivial way (that is, beyond just the number of training examples). And ideally, in the case of label noise experiments, the bound should increase with increasing amount of label noise.

\end{enumerate}

\noindent In addition to this constraint, practical experience with deep learning suggests several other constraints on a theory of generalization. \citet{Nagarajan19b} propose the following criteria:

\begin{enumerate}[resume,label={\bf [C\arabic*]}]

\item {\bf Incorporate architecture dependence.} \label{C2} It should provide an explanation of why the generalization error reduces with increasing width (or depth), as has been surprisingly observed in practice.

\item {\bf Apply to real networks (not post-processed networks, noisy variants, or infinite limits).} \label{C3} The theory should apply to the network learned by gradient descent without any modification or explicit regularization, or to idealizations.

\item {\bf Generalization should improve with more training data.} \label{C4} The generalization gap predicted by the theory should decrease with increase in the size of the training set (and not increase as is the case for many bounds, particularly those based on weight-norms!).

\end{enumerate}

\noindent To these four criteria from the literature, we would like to add the following:

\begin{enumerate}[resume,label={\bf [C\arabic*]}]

\item {\bf Uniform explanation of early stopping and training to convergence.} \label{C5} The theory should provide an explanation of why early stopping works in practice, and furthermore, such an explanation should ideally be uniform with respect to the two extremes of early stopping: perfect generalization at initialization (before any data is seen) and good generalization at end of training (that is, at a local minimum or a stationary point of the loss function), particularly since early stopping is not essential to good generalization (for example, see \Cref{fig:chap1:ResNet-50imagenet}).

\item {\bf Uniform handling of stochasticity of batch construction.} \label{C6} Although for many datasets and architectures of practical interest it is computationally prohibitive to do full-batch gradient descent, where such experiments have been performed (for example, \citet{Wu18} and \citet{Chatterjee20}\footnote{See data here: \url{https://openreview.net/forum?id=ryeFY0EFwS&noteId=rJglFKTBoB}}), or equivalent optimization outcome achieved in practice (for example, over-parameterized linear regression using full-batch gradient descent starting from the origin converges to the solution obtained by computing the pseudo-inverse using the singular value decomposition), we find that full-batch also leads to good generalization. Furthermore, careful experiments show that there is no clear connection between batch-size and generalization gap though it was suggested previously that there may be (see \citet{Shallue19} for a comprehensive study). (Now, this is not to say that stochasticity does not help: for example, sampling noise from mini-batch construction could help get out of local minima, or perhaps even serve as a regularizer; but it appears that stochasticity is not essential for a non-trivial generalization gap.)   

\item {\bf Right level of generality for the phenomenon.} \label{C7} Generalization in deep learning is an extremely broad phenomenon and has been observed in many situations that differ in specifics along dimensions such as type of activation function (for example, ReLU, sigmoid, sinusoidal, etc.); type of loss function; whether the task is a regression, classification (few classes or very many classes) or a generative task; discreteness of the problem (both in inputs, such as language models, and in weights and activations, such as extremely quantized models such as binarized neural networks~\citep{Hubara16}); training schedules such as SGD with warm restarts~\citep{Loshchilov17}; etc.

Therefore, a satisfactory explanation of generalization in deep learning cannot be too closely tied to the specific properties of a system such as scale invariance of ReLU activations (the explanation would not work for other activation functions), or the learning problem being a classification problem with a few output classes such as a margin-based explanation (the explanation would not be natural for regression or generative models, or even classification with 100K output classes), or even the training steps monotonically reducing in size (would not work for warm restarts).

\end{enumerate}

\thoughtsep

\noindent To this list, one might be tempted to add: {\em uniformity of the explanation with respect to the degree of over-parameterization,} that is, an uniform explanation for both under-parameterized models and over-parameterized models. However, given the observations of \citet{Belkin19} around double descent, such a criterion may be difficult to satisfy. We believe that this dimension is an interesting one to explore in the future, but for now we focus only on the over-parameterized case.  

\thoughtsep

We believe our approach satisfies all the criteria outlined in this  section (even though some such as \ref{C2} require further work). In the next section, we contrast our approach with the leading alternatives using these criteria as the framework for comparison. 

\section{Comparison with Other Theories and Explanations}
\label{sec:overview:comp}

There has been a lot of work to understand the nature of deep learning. At the highest level, we can categorize the efforts into two buckets: understanding optimization, and understanding generalization. This follows from decomposing the test loss as follows: 
\[
{\rm test\ loss} = {\rm training\ loss} + {\rm generalization\ gap}
\]
where the generalization gap is defined as the difference between the test and training losses.

The work on optimization focuses on trying to understand why given the non-convexity of the fitting problem in general, gradient descent can reach good global minima (and often zero training loss). On the other hand, the work on generalization tries to understand why the generalization gap is small. Although the work on optimization is fascinating in its own right, and there may be reasons to believe that there is an interesting interplay between optimization and generalization in deep learning (for example, our work suggests that when one can optimize quickly, that is also when one is most likely to generalize well), our focus here is on generalization. We refer the reader to
\citet{LiuC20a} as a recent entry point into the optimization literature.

\newbold{Approaches based on Uniform Convergence.} 
Uniform convergence~\citep{Vapnik71} is the earliest and most popular mathematical tool for proving generalization. This theory was initially developed to understand under what circumstances is it possible to identify a model or a hypothesis (from a possibly infinite family) based on only a finite training sample, and it provided a framework to bound the generalization gap based on the complexity of the family of hypotheses (quantified, for example, by notions such as Vapnik-Chervonenkis (VC) dimension) from which the model was selected relative to the size of the training set.

Early generalization bounds for neural networks were obtained by bounding their VC dimension by the parameter count~\citep{Baum89}, but even then it was realized that far fewer training samples are needed for good generalization in practice than is implied by these bounds.
\citet{Bartlett96,Bartlett98} showed in the case of sigmoid networks, that as long as the learning algorithm finds a network with small weights, what matters for a small generalization gap is not the number of weights (i.e., the parameter count), but the size of weights (i.e., their norm).

The growing popularity of deep learning along with the surprising experimental results of \citet{Zhang17} (showing that deep networks used in practice can easily fit random data) brought renewed attention to the generalization problem in deep learning. The results of \citet{Zhang17} imply that the complexity of hypothesis class in the case of practical neural networks is already too large in comparison with the size of the training set to make these bounds non-vacuous (that is, they fail \ref{C1}). 
This motivated a large number of new investigations that try to capture the implicit bias or regularization of gradient descent (for example, \citet{Bartlett17, Golowich18, Arora18, Neyshabur18a, Dziugaite17, Zhou19, Li18, Zhu19, Nagarajan19a, Neyshabur19}).
While these bounds aim to satisfy \ref{C1}, and often provide valuable insights into the nature of deep learning, \citet{Nagarajan19b} observe that they generally fail to either have the correct architecture dependence (\ref{C2}) or are not applicable to
the network obtained from gradient descent, but to some modification thereof (\ref{C3}). 
Furthermore, they point out that many of these approaches (particularly, the non-PAC-Bayesian ones), fail in a more fundamental way by failing to satisfy \ref{C4}: the bounds on the generalization gap obtained from these approaches increases with more training examples. This happens because the bound depends on some norm of the weights (that is, distance from initialization) that typically {\em increases} with larger training sets.\footnote{The PAC-Bayesian ones typically fail \ref{C3} since they apply only to stochastic ensembles. We discuss them in greater detail in a little later.}

Finally, \citet{Nagarajan19b} construct examples of overparameterized problems involving linear models and neural networks where uniform convergence {\em provably} fails to explain generalization. They show that the set of models with low test error and actually explored by GD (that is, every model in the set corresponds to a GD solution for some training set) is too large to get anything other than an almost vacuous uniform convergence bound for its members.

For the expert in uniform convergence, we note that the fact that uniform convergence is necessary for learnability (in the setting of classification or regression) does not preclude the fact that uniform convergence may be unable to explain generalization in a particular algorithm for a particular distribution (which is the situation we are faced with in explaining the generalization mystery in deep learning). For a detailed discussion of this, please see Appendix I of \citet{Nagarajan19b}. In fact, to their argument, we may add that what we are interested in is not just a particular algorithm for a particular distribution, but even a specific sample size. In other words, a result that is asymptotic in sample size may not be illuminating at all. For example, consider that there are only finitely many ResNet-50 networks when the parameters are quantized to 32 bits. Thus, for a large enough sample size, even the humble combination of Hoeffding's inequality and the union bound would provide a non-trivial bound on the generalization gap, but that argument does not provide much insight on why ResNet-50 trained with gradient descent generalizes well on ImageNet given the tiny size of the actual ImageNet training set in comparison.

\thoughtsep

{\em In comparison, our approach is not based on uniform convergence, but on a different tool from learning theory, namely algorithmic stability, and so does not suffer from the problems pointed out by \cite{Nagarajan19b}. In particular, note that the bound in Theorem~\ref{thm:cogen} has an explicit inverse dependence on the size of the training set, as expected.}

\newbold{Previous approaches based on Algorithmic Stability.} 
In short order after the development of uniform convergence, an alternative framework was developed to explain generalization in models based on ``local rules'' (for example, nearest neighbors) where the complexity of the hypothesis class is not fixed but scales with the size of the training dataset~\citep{Rogers78, Devroye79}.
Unlike uniform convergence where the focus is on the {\em size} of the space to be searched, the focus here is on {\em how} the algorithm searches the space.
This approach is now known as Algorithmic Stability~\citep{Bousquet02, ShalevShwartz10, Kearns99}, and the key idea, as we saw, is as follows: If the model learnt from the training set does not depend too much on any one training example (that is the presence or absence of any one example does not change the model very much) then we expect the model to generalize well.

\citet{Hardt16} analyzed the stability of {\em stochastic} gradient descent in the setting of deep learning under the assumption that each batch consists of one example and the learning rate decays over time. Under these assumptions, they are able to bound the generalization gap by the number of training steps, a result they summarize by the slogan ``train faster, generalize better.'' \citet{Kuzborskij18} tighten the \citet{Hardt16} bound by taking into account the curvature of the loss function at initialization but at the cost of a more onerous assumption: no training example is looked at more than once.

This line of argument critically relies on the assumption that each training sample is seen rarely (or not at all) and so the presence or absence of a training sample cannot matter too much. But one should feel uneasy about such an explanation since it may throw the baby out with the bath water. For example, if we don’t see any example at all (that is, take 0 steps of training and stay at the randomly chosen initial point), then the generalization gap should be zero in expectation, that is, we should expect the training loss to be an unbiased estimate of the test loss. The previous slogan taken to its logical conclusion is: Train in no time, generalize perfectly!

So what has gone wrong? The problem is that the generalization mystery requires us to keep in mind simultaneously the real data case and the random data case, and what we need is a ``differential" analysis of the two cases. In other words, in the random case, if you see each example rarely (or not at all), you will not be able to memorize it, but the fact that in practice neural networks memorize random labels means that they are run for long enough to see each example several times, and it is in {\em that} setting that we need to explain generalization on real data. 
In other words, the existing stability approaches fail to satisfy \ref{C1} above. Finally, as noted in Section 7 of  the extended version\footnote{\url{https://arxiv.org/abs/1509.01240}} of \cite{Hardt16}, this approach does not extend to full-batch case, that is, it fails \ref{C6}.

\thoughtsep

{\em Our result may be seen as a natural extension of this approach that takes into account the dataset (more precisely, the interaction of the dataset and the architecture) which is necessary to explain the generalization mystery.} 

\thoughtsep

We note here a line of work that combines stability and PAC-Bayesian analysis to study noisy gradient descent (stochastic gradient Langevin dynamics) where isotropic Gaussian noise is added at every SGD update step (thus failing to meet \ref{C3} and \ref{C6}). Although a more detailed discussion is out of scope, please see \citet{Li20} for a starting point into this literature.

\newbold{Sharpness and Flatness of Minima.}
%
A popular explanation of why SGD generalizes well in large networks in practice is based on the empirical observation that the minima found by SGD are ``flat,'' that is, they are in regions where the overall loss (on the training set) does not change very much around the solution found by SGD~\citep{Hochreiter97, Keskar17}.
While this is intuitively appealing (for example, see Figure 1 in \citep{He19}), and may be combined with a PAC-Bayesian analysis~\citep{McAllester99} to provide non-vacuous bounds for stochastic (noisy) variants of real networks~\citep{Langford02, Dziugaite17, Zhou19}, as a {\em fundamental} explanation of generalization it is unsatisfactory for several reasons:

\begin{enumerate}[label=(\alph*)]

\item {\bf Does not work for linear regression.} Flatness is not useful in understanding why gradient descent generalizes well in the simple case of over-parameterized linear regression (such as \Cref{ex:two-datasets}) since the Hessian of the loss function is constant over the parameter space, and thus the same at both well and poorly generalizing minima~\citep[Section 5]{Zhang21}. 
  
\item {\bf Not invariant to reparameterization in deep networks.} In deep neural networks, it is difficult to formalize the notion of flatness in a manner that is invariant to reparameterization. For example, \citet{Dinh17} show that one can construct arbitrarily sharp minima through weight reparameterization without affecting the generalization performance. 

\item {\bf Not uniform with respect to early stopping.} That is, it fails \ref{C5}. It provides different explanations for good generalization at start of training (``initialization being random is independent of the training data"), and for end of training (``gradient descent is biased towards producing solutions that lie in flat regions'').

\item {\bf It is incomplete.} It does not provide an explanation for why gradient descent prefers flat minima in deep nets, only that it does so (but see \citet{Jastrzebski20, Jastrzebski21} for exciting recent progress in this direction based on choice of early learning rate). It is important to note that the PAC-Bayesian bounds on the generalization gap mentioned above only exploit this empirical observation but do not explain it (for example, see the last paragraph of Section 1.1 in \citet{Dziugaite17}).

\end{enumerate}

\noindent From an algorithmic stability perspective, the relationship between the minima of different examples is more fundamental than the nature of those minima. For example, if the loss of each example has a sharp minimum (for whatever definition of sharpness one likes), but the minima of all the examples coincide, then one could expect better generalization than when each example has a shallow minimum, but the minima of different examples do not coincide.
Of course, it is possible that in practice, the shallowness of the overall loss function is a consequence of better alignment of the minima of different examples, but the point is that the flatness of the overall loss is a secondary to the relationship between the minima of different examples.

Finally, we note that just as the flatness perspective provides inspiration to modify gradient descent to find flatter minima to improve generalization, as seen in \Cref{sec:overview:suppress}, the Coherent Gradients perspective provides similar inspiration to design new algorithms to avoid weak directions.

\thoughtsep

\newbold {On Minimum Norm Solutions.} Another popular explanation of why GD generalizes well is that it finds solutions with small $\ell^2$ norms.
There has been renewed interest in understanding the over-parameterized linear regression case---where GD provably finds the minimum $\ell^2$ norm solution\footnote{though only when initialized at {\bf 0}}---with the hope that these results may be extended to deep networks using observations from the infinite width scaling limit~\citep{Hastie20, Bartlett20}.
However, an explanation based on the size of the norm obtained by gradient descent is also unsatisfactory as a {\em fundamental} explanation for much of the same reasons as before:

\begin{enumerate}[label=(\alph*)]

\item {\bf It is incomplete or has wrong asymptotics.}
Without an explanation for {\em why} gradient descent prefers small norm solutions in deep networks, any explanation for the generalization mystery based on smallness of $\ell^2$ norm is incomplete. For example,  \citet{Bartlett96} showed that certain deep networks generalize well {\em if} gradient descent finds small norm solutions. But this explanation is incomplete because the question for the generalization mystery then is: Why does gradient descent find a small norm solution in one case (real data) and not in the other (random)? That is, why does SGD not always find the large norm (memorization) solution? 
More recent work bounds the norm in terms of the size of the training set $m$, but as \citet[Section 1.1]{Nagarajan19b} observe, these bounds typically have the wrong asymptotics in $m$ (they predict generalization gap increases with more training examples).

\item {\bf Not invariant to reparameterization in deep networks.} For example, in a ReLU network, one can change the norm without changing the function, and hence, the generalization of the model.

\item {\bf Not uniform with respect to initialization.} That is, it fails \ref{C5}. In particular, there is good generalization at initialization regardless of the norm of the parameters (since the parameters are chosen independently of the data).

\end{enumerate}
Furthermore, solutions with non-optimal $\ell^2$ norm solutions can generalize better than those with minimum $\ell^2$ norm (see \citet{Kawaguchi17, Shah20}, or even the median solution in \Cref{ex:two-datasets}).

\thoughtsep

{\em We believe that an explanation based on stability and alignment of gradients during training is a deeper and more general explanation than flatness or small $\ell^2$ norm, especially since it provides an uniform explanation for generalization in over-parameterized linear regression and for non-minimal $\ell^2$ norm solutions and early stopping.}

\thoughtsep

Coherent Gradients may provide a deeper explanation of why SGD gravitates towards small norm or flat-minima solutions, and that would be an interesting direction for future work.

\newbold{Kernel Methods and the Neural Tangent Kernel.}
\citet{Belkin18} point out that traditional kernel methods show the same generalization mystery that deep neural networks do, and furthermore can be viewed as a special case of neural networks (a kernel machine is a two-layer neural network with the first layer frozen). Therefore, they argue, in order to understand deep learning, we need to understand ``shallow'' kernel methods first, and foreshadowing~\citet{Nagarajan19b}, they  point out that the traditional norm-based bounds for kernel methods have the wrong asymptotics. 
We believe that the notion of algorithmic stability (rather than uniform convergence) is the correct lens through which to understand generalization even in kernel machines, and that our theory for understanding generalization in deep learning can be extended to understand kernel learning as well.

Recently, there has been an increased appreciation that the behavior of certain classes of deep models become simpler in the limit that their width goes to infinity, and certain convergence results to zero training loss can be established~\citep{Du18, Li18, Du19, Allen-Zhu19, Zou20}.
\cite{Jacot18} showed that training behaves like kernel regression in this limit with a (constant) kernel that is derived from the linearization of the network around its initialization, called the {\em Neural Tangent Kernel} (NTK). 
%

While this is a promising line of research, particularly with respect to understanding optimization, it does not directly shed light on the generalization mystery. In essence, it makes the observation that under some structural assumptions (for example, infinite width limit with linear output layer but no bottleneck layers~\citep{LiuC20}, or a suitably large initialization~\citep{Chizat19}), a complex non-linear model can behave like a simple linear model (kernel regression).
Thus the generalization mystery in those cases becomes: why does that simple kernel regression model generalize well in one case (real data) and not in another (random data)?

{\em Our theory provides an answer for the latter question---as it does even for ordinary over-para\-meter\-ized linear regression---but our explanation is general enough that it can be applied directly to the complex models, thus obviating the need for the structural assumptions that guarantee a constant kernel.} 
From our list of criteria for a theory of generalization, this line of work falls short of \ref{C3} and \ref{C7} (and arguably \ref{C1} to the extent it even aims to address that).
We stress that our concern here is generalization alone: we do not address why an over-parameterized network can fit the training data. The NTK line of research sheds valuable light on that fitting question, and points the way to more general theories of optimization for deep learning such as \citep{LiuC20a}.

It is unclear if real networks actually operate in the constant kernel regime. \citet{Chizat19} argue that they do not, and show experimentally that the lazy training degrades performance.%
\footnote{\citet{LiuC20} however argue that constant kernel is not the same as lazy training, and the kernel may be constant even in cases where the parameters change significantly (depending on the Hessian).}
The experiments of \citet{Arora19} on convolutional nets also show a significant gap between the constant kernel, and finite convolutional networks (even with regularization turned off), and there is work to close this gap by identifying more sophisticated tangent constructions that, for example, depend not just on the input images but the output labels as well~\citep{Chen20}.
At any rate, \citeauthor{Chizat19} conclude that ``the intriguing phenomenon that still defies theoretical understanding is [the non-lazy regime]: neural networks trained with gradient-based methods (and neurons
that move) have the ability to perform high-dimensional feature selection through highly non-linear
dynamics."
Our work tackles this head on (see \Cref{sec:overview:feedback}), and we see that the non-linear dynamics help even in very simple (and very finite!) cases (such as Examples~\ref{ex:two-deep} and \ref{ex:two-neurons}).

Finally, a ``kernel" that varies during training (the ``path kernel") has been considered by \citet{domingos20}, but he does not investigate how the path taken by SGD is chosen, and thus does not answer the generalization mystery. His theory (implicitly) assumes that the paths for real and random data are different, and therefore does not explain why models along one path generalize better than those along the other. Our coherence and stability-based analysis answers that question.

\newbold{Alignment of Per-Example Gradients.}
In addition to our own work outlining some preliminary results~\citep{Chatterjee20, Zielinski20, Chatterjee20c}, there has been a small, but growing, number of other efforts to understand optimization and generalization in deep learning from the perspective of gradient alignment. 
Here we provide a brief overview of these efforts, and defer detailed discussions of theoretical connections and experimental results to \Cref{app:coherence-comparison} and \Cref{app:measuring-coherence} respectively.
 
\citet{Yin18} identify a quantity called {\em gradient diversity} on the training set that they use to study the question of convergence speed saturation of SGD, that is, what is the point at which increasing batch size no longer helps reduce the number of passes needed through the data (leading to more aggregate computation). They also study the corresponding question for generalization, that is, when is the generalization with a larger batch size still similar to the batch-size of 1 case considered in \citet{Hardt16}. They restrict their attention to the convex case and provide a bound on the batch size based on a quantity called {\em differential gradient diversity} (which like diversity is also computed on the training set).

\citet{Fort19} defined two measures of per-example gradient alignment called {\em sign} and {\em cosine stiffness} and conducted experiments to study the connection between stiffness and generalization. 
\citet{Sankararaman20} propose a measure for gradient alignment called confusion and study its effect on the convergence of SGD (as the depth and widths of the networks change). They do not study generalization.
Recently, \citet{Mehta21} present an adversarial initialization scheme for the sine activation~\citep{Liu19}, and study it using the lens of gradient alignment using a metric very similar to coherence. They show experimentally on small datasets such as {\sc cifar}-10, SHVN, etc. that within a class, their metric correlates well with generalization, as the scale of initialization is varied.

We note here that these efforts, particularly those looking at generalization, study gradient alignment within (and sometimes across) classes which is unsatisfactory since 
class-based analyses do not extend naturally to large output spaces as in regression or generative models (see \ref{C7}). 
Furthermore, in the context of generalization, their focus is primarily on the question of when GD generalizes well (that is, Question 1 in Section~\ref{sec:intro}). In our work, we are also interested in understanding the causal mechanism that enables GD to find a well-generalizing solution when such a solution exists (that is, Question 2 in Section~\ref{sec:intro}).

Although not directly based on per-example gradients, the notion of {\em local elasticity} discovered experimentally by \citet{He20} is worth highlighting here. They find that
an SGD update in a deep network
for an input $x$ does not change the predictions for a ``dissimilar" input $x'$. Although they do not provide a general definition of dissimilarity, intuitively, images of a Persian cat and an Egyptian cat (to use an example from their paper) would be less dissimilar than images of a German shepherd and a trailer truck.
They note that local elasticity is observed (a) only in deep models (and not, say, in linear models) and (b) is ``pronounced only after several epochs of training"~\citep[footnote 1]{Deng21}. Thus, as an explanation of generalization of gradient descent it fails \ref{C5} and \ref{C7}.
In later work, they have used the notion of local elasticity to suggest an improvement to the NTK~\citep{Chen20} (this was discussed previously), as well as a new notion of algorithmic stability called {\em locally elastic stability}~\citep{Deng21} that is an elegant refinement of the notion of uniform stability to incorporate data dependence. 
They show that SGD is locally elastic stable, but their proof is based on the same insight as \citet{Hardt16} that each training example is seen rarely (due to stochasticity), and therefore, it does not naturally extend to the non-stochastic case (fails \ref{C6}).


\thoughtsep

Finally, \citet{Liu20} is the closest work in the literature to ours, but their primary focus is also Question 1.%
\footnote{It is interesting to note that their paper appeared in parallel with our preliminary work~\citep{Chatterjee20} which addressed the other half of the generalization mystery, that is, Question 2 in Section~\ref{sec:intro}.}
They define two quantities {\em Gradient Signal to Noise Ratio} (GSNR) (which is on a per-coordinate basis on the training set) and {\em One Step Generalization Ratio} (OSGR) (which links the one-step change in the training and test losses) and show how in {\em early} training, under the assumption that the average gradients in test and training are the same (their Assumption 2.3.1), in a full-batch setting, the two quantities are related, and high GSNR implies an OSGR close to 1.
We stress that their result is a local result that holds only for one step of training, unlike \Cref{thm:cogen} which holds globally.
Furthermore, a reformulation in terms of $\{$test, training$\} \times \{$overall, per-coordinate$\}$
quantities provides a simpler view of their result. 
See discussion of GSNR and OSGR in \Cref{app:coherence-comparison} for more details.

Furthermore, like us, they also observe that in deep networks, the alignment of per-example gradients can increase in the course of training (that is, there is coherence amplification). 
However, in their experiments, they only see it on real data and not with random labels in contrast to what we see in \Cref{sec:overview:measurement} with our metric on ResNet (where it is a small effect) and on AlexNet (where it is a large effect).
They provide a heuristic argument for why the numerator of GSNR can increase due to feature learning, but stop short of their ultimate goal which is showing why GSNR itself can increase. Their heuristic explanation based on the dynamics of the training is similar to our discussion in \Cref{sec:overview:feedback} but see our discussion of the difference between coherence amplification and signal amplification.
Of course, here it must be noted GSNR does not always increase (even in their data), and in our work, we show why that must be so as examples get fitted albeit in the context of our coherence metric $\alpha$ rather than GSNR (see \Cref{app:evolution}). 

It should be noted that since they do not address Question 2 in \Cref{sec:intro}, their theory is descriptive rather than causal, and does not directly suggest modifications such as those considered in \Cref{sec:overview:suppress} to reduce memorization and improve generalization.


\section{Discussion and Directions for Future Work}

In this paper, we have presented a simple explanation for the generalization mystery in deep learning: 
When the gradients of different examples are well-aligned during training, that is, when they are {\em coherent}, gradient descent is stable, and the resulting model is expected to generalize well. Otherwise, gradient descent may fail to generalize if there are too few examples or if it is run for too long.

This explanation has several attractive features. First, due to its simplicity, it is very general. It provides an uniform explanation for memorization and generalization, and scales from toy over-parameterized linear regression problems to deep models at ImageNet scale.
Second, it separates optimization-related issues (which can be hard to reason about in the non-convex settings common to deep learning) from generalization-related ones, allowing us to make progress independently on these aspects.
Third, it is a {\em causal} explanation for the implicit bias of gradient descent that provides perspective on existing regularization techniques, the relative hardness of examples, the role of depth, and motivates novel but simple modifications such as winsorization and {\sf M3} for improving the generalization of gradient descent.
Last, but not least, it avoids critical (and possibly fatal) problems facing competing explanations that have been put forward to explain the generalization mystery. 

There are many directions to explore from here, particularly, in the formal development of the theory, in establishing how widely it holds, in using it to improve our understanding of deep learning, and in developing new, more efficient, representation learning algorithms. We outline some specific ideas in greater detail below:

\begin{enumerate}

    \item {\bf Better metrics for coherence and a tighter bounds.} Our generalization bound depends on $\alpha$ as measured over entire per-example gradients. 
    However, as we have seen, both on toy examples and some real architectures, some parts of the network may play a greater role in memorization than others, and a measurement over the entire per-example gradient averages over these differences leading to a necessarily weak bound.
    Can we obtain tighter bounds by accounting for the graphical structure of the network, say by developing some notion of ``minimum coherence" cuts that control generalization?
    
    \item {\bf Incorporating the under-parameterized case.} Our goal here was to understand what happens in highly over-parameterized settings. While the inverse dependence on the number of examples $m$ in our bound provides a limiting guarantee in the case of extreme under-parameterization (as $m \to \infty$), a more direct dependence on the dimension, or better yet, on the graphical structure of the model (as noted above), is desirable.
    The ideal theory would handle the over- and under-parameterized cases uniformly, thus providing a continuous explanation for the phenomenon of deep double descent~\citep{Belkin19} and also apply to pathological learning problems that are simultaneously over- and under-parameterized (in different parts of the data distribution). See also \Cref{app:linear-regression}.
    
    \item{\bf Generalization bound based on training set only.} Is there an efficient way to measure stability degradation during training using {\em only} the training set in order to provide a non-vacuous bound on the generalization gap? Perhaps this could be done with a ``stability accountant" similar to a privacy accountant in differential privacy.
        
    \item {\bf Confirmation on other datasets and architectures.} In this work we have only studied vision datasets. Does the explanation hold for other types of data (such as language and speech) and architectures (such as transformers)? 
    
    \item {\bf Generalization and width.} \citet{Neyshabur18} found that wider networks generalize better. Can we now explain this? Intuitively, wider networks have more sub-networks at any given level, and so the sub-network with maximum coherence in a wider network may be more coherent than its counterpart in a thinner network, and hence generalize better.
    In other words, since---as discussed in \Cref{sec:overview:feedback}---gradient descent is a feature selector that prioritizes well-generalizing (coherent) features, wider networks are likely to have better features simply because they have more features.
    In this connection, see also the Lottery Ticket Hypothesis~\citep{Frankle18}.
    
    \item {\bf Understanding texture and shape bias.} Can use coherence at random initialization and signal amplification to understand why vision networks have a bias towards features based on ``bulk" properties such as texture as opposed to those based on more subtle properties such as shape (see, for example, \citet{Geirhos18})?
   
   \item {\bf Practical use for new algorithms.} Can algorithms based on robust averaging (such as winsorized gradient descent, or {\sf M3}) help improve the state-of-the-art in large-scale learning with noisy labels, learning in low data settings, and in private learning? Can they be extended to provide privacy guarantees?
   
   \item {\bf Hardware acceleration for private learning.} Modern hardware accelerators for deep learning typically have multiple nodes (for example, GPU or TPU cores) over which the examples in a mini-batch are distributed, and the gradient for a mini-batch is computed via a global {\sf all-reduce} operation over the average gradient computed at each node~\citep{Xu20, Kumar20}. Platform support for robust averaging (by clipping outliers), perhaps by supporting multiple simultaneous operations as part of an {\sf all-reduce}, would greatly improve the efficiency of algorithms such as {\sf M3} for
   the use cases outlined above.
   
   \item {\bf Using coherence to evaluate architectures.} Can we use coherence to get insights into the relative strengths and weaknesses of different network architectures, and to improve architectures to generalize better?
   For example, as we saw in \Cref{sec:overview:measurement}, for random data, $\alpha$ of (entire) per-example gradients is much lower for ResNet-50 than for AlexNet. Does that mean that overfitting in ResNet-50 is harder than in AlexNet since it cannot be as ``local" (confined to a few layers)?
   
   \item {\bf Using coherence to study optimizers.} In this work we have restricted our focus to vanilla gradient descent. However, the notion of coherence can be used to analyze the generalization behavior of other optimization algorithms such as {\sc adam}, a topic of much interest (see, for example, \citet{zhou20} and references therein). 
   For other optimizers, in addition to the coherence of the per-example gradients, one would have to consider the transformation applied by the optimizer to obtain the actual updates to the network parameters. 

   \item {\bf Explore Implications of the Theory.} In order to gain further confidence in the theory, it would be good to check some of the following predictions which we believe are consequences of the theory:
   
   \begin{enumerate}
       \item  {\em Full-batch gradient descent and gradient flow should have non-trivial generalization on real data.} Based on our theory, finite learning rate, and mini-batch stochasticity are not necessary for generalization (though they may help).
       Note that we are only talking about generalization gap, not test loss: Finite learning rate and mini-batch stochasticity can help escape saddle points during training, and so in their absence we may get stuck with poor training loss, but any source of noise could be used to escape these saddle points.
       Thus, we would expect, for instance, that full-batch gradient descent on ImageNet with a tiny learning rate should show non-trivial generalization.
       
       \item {\em Assuming a small enough learning rate, as training progresses, the generalization gap cannot decrease.} 
       This follows from the iterative stability analysis of training: with more steps, stability can only degrade.
       If this is violated in a practical setting, it would point to an interesting limitation of the theory.
       
       \item {\em Continuity is not necessary for good generalization.} Although gradient descent involves real numbers (or close approximations), if the theory is right, it should be possible to engineer fully discrete representation learning systems using voting among training examples to make optimization decisions during fitting that generalize well.
       By avoiding floating point operations (or even high precision fixed point), such systems may be much more computationally efficient than gradient descent.
       
   \end{enumerate}
   
\end{enumerate}


\thoughtsep

\noindent {\sc Acknowledgments.}
We thank Alan Mishchenko and Shankar Krishnan for close collaborations on related projects, and the following for interesting discussions and feedback: 
Anand Babu,
Sungmin Bae,
Michele Covelle,
Josh Dillon, 
Sergey Ioffe,
Firdaus Janoos,
Stanislaw Jastrzebski, 
Chandramouli Kashyap,
Matt Streeter,
Rahul Sukthankar, 
and Jay Yagnik.

\thoughtsep

\addtocontents{toc}{\phantom{M}}
\begin{appendices}
\crefalias{section}{appendix}
\section{Mathematical Properties of \texorpdfstring{$\alpha$}{Alpha}}
\label{app:alphafacts}

Our metric for coherence $\alpha$ defined in \Cref{sec:overview:metrics} is not specific to gradients, but extends naturally to vectors in Euclidean spaces. It is convenient, therefore, to study its properties in that general setting. 
Let $\mathcal{V}$ be a probability distribution on a collection of $m$ vectors in an Euclidean space.
In accordance with (\ref{eq:overview:metric}), we define $\alpha(\mathcal{V})$ to be
\begin{equation}
\alpha(\mathcal{V}) \equiv \frac{\displaystyle \E_{v \sim \mathcal{V}, v' \sim \mathcal{V}} \  [ v \cdot v' ]}{\displaystyle \E_{v \sim \mathcal{V}} \  [ v \cdot v]}
\label{eq:definealpha}
\end{equation}
Note that $\E[v \cdot v] = 0$ implies $\E[v \cdot v'] = 0$.
In what follows, we ignore the technicality of the denominator being 0 by always assuming that there is at least one non-zero vector in the support of $\mathcal{V}$.\footnote{This is also generally
true in our experiments, but if necessary, a small constant can be added to the denominator to ensure this quantity is always well-defined.}

The following example generalizes the coherence calculation of \Cref{ex:two-datasets}, and shows that greater the commonality between vectors, greater is $\alpha$.

\begin{example} [Commonality]
\label{ex:app:commonality}
For $1 \le i \le m$, suppose each $v_i$ has a common component $c$ and an idiosyncratic component $u_i$, that is, $v_i = c + u_i$ with $u_i \cdot u_j = 0$ for $1 \le j \le m$ and $j \ne i$; $u_i \cdot c = 0$; and say, $u_i \cdot u_i = \|u\|^2$ for some $u$. It is easy to see that $\alpha$ over the uniform distribution on $v_i$ in this case is 
$ \frac{1}{m} \left[ 1 + (m - 1) \cdot f  \right] $
where $f = \|c\|^2 / (\|c\|^2 + \|u\|^2)$.

\qed
\end{example}

\begin{theorem}[Boundedness and Scale Invariance]
We have,
\[
0 \leq \alpha(\mathcal{V}) \leq 1
\]
where $\alpha(\mathcal{V}) = 0$ iff $\E_{v \sim \mathcal{V}} [v] = 0$ and $\alpha(\mathcal{V}) = 1$ iff all the vectors are equal.
 
Furthermore, for non-zero $k \in \mathbb{R}$, we have, 
\[
\alpha(k \mathcal{V}) = \alpha(\mathcal{V})
\]
where $k \mathcal{V}$ denotes the distribution of the random variable $k v$ where $v$ is drawn from $\mathcal{V}$.
\label{thm:basicprop}
\end{theorem}

\begin{proof}
{\em Boundedness.} Since $v \cdot v \geq 0$ for any $v$, we have $\E_{v \sim \mathcal{V}}[v \cdot v] \geq 0$. Furthermore, it is easy to verify by expanding the expectations (in terms of the vectors and their corresponding probabilities) that
\begin{equation}
\displaystyle \E_{v \sim \mathcal{V}, v' \sim \mathcal{V}} \  [ v \cdot v' ] = \E_{v \sim \mathcal{V}} [v]\ \cdot \E_{v \sim \mathcal{V}} [v] 
\geq 0.
\label{eq:norm}
\end{equation}
Therefore, $\alpha(\mathcal{V}) \geq 0$.
Likewise, another direct computation shows that
\begin{equation}
%
0 \leq 
\E_{v' \sim \mathcal{V}}\ \left[
  (\E_{v \sim \mathcal{V}}[v] - v')
        \cdot 
  (\E_{v \sim \mathcal{V}}[v] - v') 
\right] 
= 
\E_{v \sim \mathcal{V}}[v \cdot v] -
\E_{v \sim \mathcal{V}}[v]\ \cdot \E_{v \sim \mathcal{V}}[v] 
\end{equation}
Since from Equation~\ref{eq:norm} we have $\E[v] \cdot \E[v] = \E[v \cdot v']$, it follows that $\alpha(\mathcal{V}) \leq 1$.
Furthermore, since each term of the expectation on the left is
non-negative, equality is attained only when all the vectors are equal.

{\em Scale invariance.} We have,

\begin{equation}
\alpha(k\mathcal{V}) 
= \frac{\displaystyle \E_{v \sim k \mathcal{V}, v' \sim k \mathcal{V}} \  [ v \cdot v' ]}{\displaystyle \E_{v \sim k \mathcal{V}} \  [ v \cdot v]}
= \frac{\displaystyle \E_{v \sim \mathcal{V}, v' \sim \mathcal{V}} \  [ k v \cdot k v' ]}{\displaystyle \E_{v \sim \mathcal{V}} \  [ k v \cdot k v]}
= \frac{\displaystyle \E_{v \sim \mathcal{V}, v' \sim \mathcal{V}} \  [ v \cdot v' ]}{\displaystyle \E_{v \sim \mathcal{V}} \  [ v \cdot v]}
= \alpha(\mathcal{V})
\end{equation}
\end{proof}

\newbold{Coherence of mini-batch gradients.}
We now prove a result that connects the $\alpha$ of (the uniform distribution on) a set of individual examples to that of the distribution of mini-batches constructed from those examples, by selecting $k$ elements at random for each mini-batch with replacement. This is critical for measuring $\alpha$ in practical networks, and furthermore, allows us to estimate per-example $\alpha$ without computing per-example gradients which leads to a huge speedup in practice (see \Cref{app:measuring-coherence}).

\begin{theorem}[Stylized mini-batching]
\label{thm:stylized_mini_batching}
Let $v_1, v_2, .., v_k$ be $k$ i.i.d. variables drawn from $\mathcal{V}$. Let $\mathcal{W}$ denote the distribution of the random variable $w = \frac{1}{k} \sum_{i=1}^{k} v_i$. We have,

\begin{equation}
\label{eq:batchformula}
    \alpha(\mathcal{W}) = 
    \alpha(k \mathcal{W}) = 
    \frac{k \cdot \alpha(\mathcal{V})}{1 + (k - 1) \cdot \alpha(\mathcal{V})}
\end{equation}
\noindent Furthermore, 
$\alpha(\mathcal{W}) \geq \alpha(\mathcal{V})$
with equality iff $\alpha(\mathcal{V}) = 0$ or $\alpha(\mathcal{V}) = 1$.
\end{theorem}

\begin{proof}
The first equality in (\ref{eq:batchformula}) follows from Theorem~\ref{thm:basicprop}. For the second equality, we have,
\begin{equation*}
\alpha(k \mathcal{W}) = 
\frac{\displaystyle \E_{\substack{w \sim k \mathcal{W}, \\ w' \sim k \mathcal{W}}} \  [ w \cdot w' ]}{\displaystyle \E_{w \sim k \mathcal{W}} \  [ w \cdot w]} =
\frac{\displaystyle \E_{\substack{v_1, .., v_k, \\ v_1', .., v_k'}} \  [ (\sum_i v_i) \cdot (\sum_i v_i') ]}{\displaystyle \E_{v_1, .., v_k} \  [(\sum_i v_i) \cdot (\sum_i v_i)]} =
\frac{\displaystyle k^2 \E_{\substack{v \sim \mathcal{V},\\v' \sim \mathcal{V}}} \  [ v \cdot v' ]}
{\displaystyle k \E_{v \sim \mathcal{V}} \  [ v \cdot v] + k\cdot(k - 1)\E_{\substack{v \sim \mathcal{V},\\v' \sim \mathcal{V}}} \  [ v \cdot v' ]} 
\end{equation*}
By dividing the numerator and denominator of the last expression by $\displaystyle k \E_{v \sim \mathcal{V}} \ [v \cdot v]$ the required result follows.

Now, since $\alpha \leq 1$, we have $k \geq 1 + (k - 1) \cdot \alpha$. Since $\alpha \geq 0$, multiplying both sides by $\frac{\alpha}{1 + (k - 1) \cdot \alpha}$ we have $\frac{k \cdot \alpha}{1 + (k - 1) \cdot \alpha} \geq \alpha$. 
Finally, it is easy to check that the two solutions of $\frac{k \cdot \alpha}{1 + (k - 1) \cdot \alpha} = \alpha$ are $\alpha = 0$ and $\alpha = 1$.
\end{proof}

\noindent {\bf Remark.} This formulation provides a nice perspective on the type of results proved in \citet{Yin18} and \citet{Jain18}. When
$\alpha \ll 1/k$ but non-zero (i.e., we have high gradient diversity), creating mini-batches of size $k$ increases coherence almost $k$ times. But, when  $\alpha \approx 1$ (i.e., low diversity) there is not much point in creating mini-batches since there is little room for improvement. See also discussion of gradient diversity in \Cref{app:coherence-comparison}.

\thoughtsep

\newbold{Coherence with some examples fitted.}
During training, as examples get fitted, their gradients become zero. The following lemma characterizes $\alpha$ when some fraction of the vectors are zero. 
It shows that $\alpha$ decreases as larger fractions of the distribution become zero (all else remaining being equal), and therefore, as examples get fitted during training there is a natural tendency for $\alpha$ to decrease. 

\begin{lemma}[Effect of zero vectors]
\label{lem:zero}
If $\mathcal{W}$ denotes the distribution where with probability $p > 0$ we pick a vector from $\mathcal{V}$ and with probability $1 - p$ we pick the zero vector then $\alpha(\mathcal{W}) =  p \cdot \alpha(\mathcal{V})$.
\end{lemma}

\begin{proof}
\begin{equation}
\alpha(\mathcal{W}) 
= \frac{\displaystyle \E_{w \sim \mathcal{W}, w' \sim \mathcal{W}} \  [ w \cdot w' ]}{\displaystyle \E_{w \sim \mathcal{W}} \  [ w \cdot w]}
= \frac{\displaystyle p^2 \cdot \E_{v \sim \mathcal{V}, v' \sim \mathcal{V}} \  [ v \cdot v' ]}{\displaystyle p \cdot \E_{v \sim \mathcal{V}} \  [ v \cdot v]}
= p \cdot \alpha(\mathcal{V})
\end{equation}
\end{proof}

\begin{example}[Coherence reduction with zero vectors]
If we add $k$ zero gradients to the collection of gradients constructed in 
Example~\ref{ex:app:commonality}, using Lemma~\ref{lem:zero}, we get,
\[
\alpha = \frac{m}{m + k} \cdot \frac{1}{m} \left[ 1 + (m - 1) \cdot f  \right] 
= \frac{1}{n} \left[ 1 + (n - k - 1) \cdot f  \right]
\]
where $n = m + k$ is the size of this new sample. For a fixed $n$, as $k$ increases, $\alpha$ decreases going down to $1/n$ (the orthogonal limit) when all but one vector in the sample is zero,
i.e., $k = n - 1$.
\qed
\end{example}

\begin{lemma}[Combining orthogonal distributions]
\label{lem:two_orthogonal_sets}
If $\mathcal{W}$ denotes the distribution where with probability $p > 0$ we pick a vector from $\mathcal{U}$ and with probability $1 - p$ we pick a vector from $\mathcal{V}$ and all elements in the support of $\mathcal{U}$ are orthogonal to those in the support of $\mathcal{V}$ then we have
\[
\alpha(\mathcal{W}) \leq  p \cdot \alpha(\mathcal{U}) + (1-p) \cdot \alpha(\mathcal{V}).
\]
\end{lemma}

\begin{proof}
\begin{align*}
\alpha(\mathcal{W}) 
& = \frac{\displaystyle \E_{w \sim \mathcal{W}, w' \sim \mathcal{W}} \  [ w \cdot w' ]}{\displaystyle \E_{w \sim \mathcal{W}} \  [ w \cdot w]} \\
& = \frac{\displaystyle p^2 \E_{u \sim \mathcal{U}, u' \sim \mathcal{U}} \  [ u \cdot u' ] + (1-p)^2 \E_{v \sim \mathcal{V}, v' \sim \mathcal{V}} \  [ v \cdot v' ] + 2 \cdot p \cdot (1-p) \E_{u \sim \mathcal{U}, v \sim \mathcal{V}} \  [ u \cdot v ]}{\displaystyle p \cdot \E_{u \sim \mathcal{U}} \  [ u \cdot u] + (1-p) \cdot \E_{v \sim \mathcal{V}} \  [ v \cdot v]} \\
& = \frac{\displaystyle p^2 \E_{u \sim \mathcal{U}, u' \sim \mathcal{U}} \  [ u \cdot u' ] + (1-p)^2 \E_{v \sim \mathcal{V}, v' \sim \mathcal{V}} \  [ v \cdot v' ]}{\displaystyle p \cdot \E_{u \sim \mathcal{U}} \  [ u \cdot u] + (1-p) \cdot \E_{v \sim \mathcal{V}} \  [ v \cdot v]} \\
& = \frac{\displaystyle p^2 \E_{u \sim \mathcal{U}, u' \sim \mathcal{U}} \  [ u \cdot u' ]}{\displaystyle p \cdot \E_{u \sim \mathcal{U}} \  [ u \cdot u] + (1-p) \cdot \E_{v \sim \mathcal{V}} \  [ v \cdot v]} + \frac{\displaystyle (1-p)^2 \E_{v \sim \mathcal{V}, v' \sim \mathcal{V}} \  [ v \cdot v' ]}{\displaystyle p \cdot \E_{v \sim \mathcal{U}} \  [ v \cdot u] + (1-p) \cdot \E_{v \sim \mathcal{V}} \  [ v \cdot v]} \\
& \leq \frac{\displaystyle p^2 \E_{u \sim \mathcal{U}, u' \sim \mathcal{U}} \  [ u \cdot u' ]}{\displaystyle p \cdot \E_{u \sim \mathcal{U}} \  [ u \cdot u]} + \frac{\displaystyle (1-p)^2 \E_{v \sim \mathcal{V}, v' \sim \mathcal{V}} \  [ v \cdot v' ]}{\displaystyle (1-p) \cdot \E_{v \sim \mathcal{V}} \  [ v \cdot v]} \\
& = p \cdot \alpha(\mathcal{U}) + (1-p) \cdot \alpha(\mathcal{V})
\end{align*}
\end{proof}

\begin{example}[Coherence reduction with ``fitted" examples]
\label{example:two_orthogonal_sets_2}

Let $U$ be a set with $k$ vectors and $V$ a set with $m - k$ vectors disjoint from $U$. The elements of $V$ are pairwise orthogonal and also orthogonal to the elements of $U$. We may think of the elements of $V$ as gradients of training examples that have already been ``fitted" but have residual (``noisy'') gradients, and the elements of $U$ as the gradients of training examples yet to be fitted. 

Let $\mathcal{U}$, $\mathcal{V}$, and $\mathcal{W}$ denote the uniform distributions on the sets $U$, $V$, and $U \cup V$, respectively. Applying Lemma~\ref{lem:two_orthogonal_sets}, and observing that $\alpha(\mathcal{V}) = 1 / (m - k)$, we get, 
\begin{equation}
    \alpha(\mathcal{W}) 
    \leq \frac{k}{m}\ \alpha(\mathcal{U}) + \frac{m - k}{m}\ \frac{1}{m - k}  
    = r\ \alpha(\mathcal{U}) + \frac{1}{m} 
\end{equation}
where $r \equiv k/m$ is the fraction of examples yet to be ``fitted".
Thus as more examples get fitted, that is, as $r$ goes to 0, the upper bound on $\mathcal{W}$ goes towards the orthogonal limit $1/m$.

\qed
\end{example}

\newbold{Lemma for stability.}
In order to analyze stability of gradient descent,
it is useful to  have a bound on the difference of two vectors in terms of $\alpha$. This is used in the proof of \Cref{thm:costab} (which is used to establish \Cref{thm:cogen}).

\begin{lemma}[Bound on Expected Difference]
\label{lem:diffbound}
We have,
\begin{equation}
\E_{v \sim \mathcal{V},\ v' \sim \mathcal{V}} \  [\|v - v'\|] 
\leq    
\sqrt{2 \  (1 - \alpha(\mathcal{V})) \E_{v \sim \mathcal{V}} \  [v \cdot v]}
\end{equation}
\end{lemma}

\begin{proof}
Starting with Jensen's inequality, we have,
\begin{align*}
\left( \E_{{v \sim \mathcal{V},\ v' \sim \mathcal{V}}} \  [\|v - v'\|] \right)^2 
& \leq \E_{{v \sim \mathcal{V},\ v' \sim \mathcal{V}}} \  [\|v - v'\|^2] \\
& =  \E_{{v \sim \mathcal{V},\ v' \sim \mathcal{V}}} \  [(v - v')\cdot(v - v')] \\
& = 2 \left(\E_{v \sim \mathcal{V}} \  [v \cdot v] - \E_{{v \sim \mathcal{V},\ v' \sim \mathcal{V}}} \  [v \cdot v']\right) \\
& = 2 \  (1 - \alpha(\mathcal{V})) \E_{v \sim \mathcal{V}} \  [v \cdot v]
\end{align*}
from which the required result follows.
\end{proof}

\noindent Note that in the case of perfect coherence, that is $\alpha(\mathcal{V}) = 1$, the expected norm of the difference between any two vectors is zero.

\newbold{Decomposition.}
Finally, the following theorem shows that the overall $\alpha$ of a network (that is, $\alpha$ as measured over entire per-example gradients) is a convex combination of that of its sub-components ($\alpha$ as measured, say, over different layers or even individual parameters). 
This is useful in understanding the coherence of a network in terms of the coherences of its parts (for example, see the discussion of AlexNet in \Cref{sec:overview:measurement}).

\begin{theorem}[Decomposition]
\label{thm:weightedmean}
Let $V$ be the support of the distribution $\mathcal{V}$. Furthermore, let 
\[
V = V_1 \oplus V_2 \oplus \cdots \oplus V_k
\]
where the subspaces $V_i$ $(1 \le i \le k)$ are orthogonal to each other, that is, $V$ is the orthogonal direct sum of the $V_i$. $\mathcal{V}$ induces a distribution $\mathcal{V}_i$ on each $V_i$. Suppose each $\alpha(\mathcal{V}_i)$ exists. Then, 
\[
\alpha(\mathcal{V}) = f_1 \cdot \alpha(\mathcal{V}_1) + f_2 \cdot \alpha(\mathcal{V}_2) + \cdots + f_k \cdot \alpha(\mathcal{V}_k)
\]
where $f_i \equiv \E_{v_i \sim \mathcal{V}_i} \  \left[ v_i \cdot v_i \right] / (\sum_{i=1}^{k} \displaystyle \E_{v_i \sim \mathcal{V}_i} \  [ v_i \cdot v_i])$
and $0 \le f_i \le 1$ and $\sum_{i=0}^k f_i = 1$.
\end{theorem}

\begin{proof}
We have,
\[
\alpha(\mathcal{V}) 
= 
\frac
    {\displaystyle \E_{v \sim \mathcal{V}, v' \sim \mathcal{V}} \  [ v \cdot v' ]}
    {\displaystyle \E_{v \sim \mathcal{V}} \  [ v \cdot v]}
=
\frac
    {\displaystyle \sum_{i=1}^{k} \displaystyle \E_{v_i \sim \mathcal{V}_i, v_i' \sim \mathcal{V}_i} \ \left[ v_i \cdot v_i' \right]}
    {\displaystyle \sum_{i=1}^{k} \displaystyle \E_{v_i \sim \mathcal{V}_i} \  [ v_i \cdot v_i]}
=
\frac
    {\displaystyle \sum_{i=1}^{k} \left( \alpha(\mathcal{V}_i) \cdot \displaystyle \E_{v_i \sim \mathcal{V}_i} \  \left[ v_i \cdot v_i \right] \right)}
    {\displaystyle \sum_{i=1}^{k} \displaystyle \E_{v_i \sim \mathcal{V}_i} \  [ v_i \cdot v_i]}
=
\sum_{i=1}^{k} f_i \cdot \alpha(\mathcal{V}_i). 
\]
\end{proof}


\section{Comparison of \texorpdfstring{$\alpha$}{Alpha} with Other Metrics}
\label{app:coherence-comparison}

The coherence metric presented in \Cref{sec:overview:metrics} was first presented in \cite{Chatterjee20c}. Here, we review some of the other metrics in the literature for studying gradient alignment, most of which were proposed concurrently. See also \Cref{sec:overview:comp} for broader discussions of the papers in which these metrics were proposed.

\newbold{Stiffness.} \citet{Fort19} study two variants of the average pairwise dot product that they call {\em sign stiffness} and {\em cosine stiffness}. In our notation these are
\[
S_{\rm sign} \equiv 
\E_{\substack{z \sim \mathcal{D}, z' \sim \mathcal{D}\\ z \neq z'}}[\ {\rm sign}(g_z \cdot g_{z'})\ ]
\ \ 
{\rm and}
\ \ 
S_{\rm cos} \equiv 
\E_{\substack{z \sim \mathcal{D}, z' \sim \mathcal{D}\\ z \neq z'}}\left[ \ \frac{g_z}{\|g_z\|} \cdot \frac{g_{z'}}{\|g_{z'}\|}\ \right].
\]
These are meant to capture how a small gradient step based on one input example affects the loss on a {\em different} input example. 
Although \citeauthor{Fort19} do not describe why they choose to transform the gradients in these specific ways, we expect it is to normalize the dot product so that it can be tracked in the course of training. 
In their experience, they found sign stiffness to be more useful to analyze stiffness between classes whereas cosine stiffness was more useful within a class.

\newbold{Gradient Confusion.} \citet{Sankararaman20} introduce the notion of a gradient confusion bound. The {\em gradient confusion bound} is $\zeta \geq 0$ if for all
$z, z' \in Z$ and $z \neq z'$, we have,
$g_z \cdot g_{z'} \geq -\zeta$.
They use this concept to study theoretically the convergence rate of gradient descent, but in their experimental results they measure the minimum cosine similarity between gradients, i.e., 
\[
\min_{\substack{z \in Z, z' \in Z \\ z \neq z'}} 
\left[ \ \frac{g_z}{\|g_z\|} \cdot \frac{g_{z'}}{\|g_{z'}\|}\ \right].
\]

\noindent There are several advantages to using $\alpha$ and, by extension, $\alpha_{m} / \alpha_{m}^{\perp}$ over these metrics:

\begin{itemize}[left=1em]

\item {\bf Computational Efficiency.} For a sample of size $m$, due to (\ref{eq:overview:metric}), $\alpha$ can be computed exactly in $O(m)$ time in contrast to $O(m^2)$ time required for stiffness and cosine dot products. Furthermore, it can be computed in a streaming fashion by keeping a running sum for the numerator and a separate one for the denominator, thereby removing the need to store per-example gradients. Thus, in our experiments we are able to use sample sizes a couple of orders of magnitude higher than those in~\citet{Fort19} and \citet{Sankararaman20}. 
Furthermore, Theorem~\ref{thm:stylized_mini_batching} leads to a huge speedup in estimating per-example coherence by avoiding the need to computer per-example gradients.

\item {\bf Mathematical Simplicity.} We believe our definition is cleaner mathematically. This allows us to reason about the metric more easily. For example,

\begin{enumerate}

    \item We can show that $\alpha$ of minibatch gradients is greater than that of individual examples (Theorem~\ref{thm:stylized_mini_batching}). Therefore, care must be taken
    if minibatch gradients are used in lieu of example gradients in computing coherence (for example, as in \citet{Sankararaman20}).

    \item Explicitly ruling out $z \neq z'$ as in done in stiffness and cosine similarity to eliminate  self-correlation is unnatural and can get tricky in practice due to near-duplicates or multiple examples leading to same or very similar gradients. 
    We obtain meaningful values without imposing those conditions, but if one insists on removing self-correlations, then subtracting $1/m$ from $\alpha_m$ or 1 from $\alpha_{m} / \alpha_{m}^{\perp}$ is a more principled way to do it. 
    
    \item The non-linearities in stiffness and cosine similarity amplify small per-example gradients potentially overstating their importance, and lead to a discontinuity (or undefined behavior) with zero gradients. However, we can cleanly account for the effect of 
    negligible gradients in our observations (e.g., see Lemma~\ref{lem:zero}).
    
\end{enumerate}

\item {\bf Interpretability.} Finally, as discussed in detail above, they are normalized and yet easily interpretable due to the natural connection with loss.
In contrast, the non-linearities (and to a lesser extent the $z \neq z'$ restriction) make it hard to tie stiffness or minimum cosine similarity to what happens during training; specifically, to the change in the loss function as a result of a gradient step which is the expectation over {\em all} per-example gradients.

\end{itemize}

\newbold{GSNR and OSGR.} \citet{Liu20} propose two metrics to study gradient alignment. The first one is {\em gradient signal-to-noise ratio} (GSNR) and is defined for the $i$th gradient coordinate $g^i_z$ as follows
\[
r^i \equiv 
\frac
    {\left[\E_{z \sim \mathcal{D}}\left[\ g^i_z \ \right]\right]^2}
    {\V_{z \sim \mathcal{D}}\left[\ g^i_z \ \right]}.
\]
They measure GSNR on the training sample. 

To study generalization, they define another quantity called {\em one-step-generalization ratio} (OSGR) denoted by ${\bf R}(n)$ which measures the ratio of the expected change in the test loss (on a test sample of size $n$) due to one step of gradient descent relative to the expected change in the training loss (on a training sample of size $n$). 
They show that under the assumption that the expected gradient over the training sample and test sample have the same distribution (the {\em early training assumption}), GSNR is related to OSGR (their equation 22), and confirm this with experiments on small datasets such as {\sc mnist} and {\sc cifar10}.

In comparison to two different quantities (one over the training sample and on a per-coordinate basis, and another over both training and test and on a aggregate basis across all coordinates), we define only a single quantity that can be measured on either the test sample or the training sample, and on a per-coordinate basis,or on an aggregated basis. 
This leads to a simpler theory, and in our formulation the equivalent result to their equation 22 factors into using their early training assumption to argue that training and test coherence (either aggregate or per-coordinate) is equal (we also confirm this with experiments in \Cref{sec:overview:measurement}), and then relating the aggregate quantities to the per-coordinate quantities using \Cref{thm:weightedmean}.

Furthermore, in order to understand the generalization mystery, we are interested in a bound not just for one step (as they have), or even for early training, but for {\em all} of training. We note here that \Cref{thm:cogen} does not depend on an early training assumption such as theirs, and indeed as their experiments and ours show, that assumption does not hold later on in training.

\newbold{Gradient Diversity.} The reciprocal of $\alpha$ appears in the theory literature as {\em gradient diversity} This was used by~\cite{Yin18} in theoretical bounds to understand the effect of mini-batching on convergence of SGD. (A similar result appears for least squares regression in \citet{Jain18}.) 
They show that the greater is the gradient diversity, the more effective are large mini-batches in speeding up SGD (see also \Cref{thm:stylized_mini_batching} and the remark following it). 
Although they support their theoretical analysis with experiments on {\sc cifar}-10 (where they replicate $1/r$ of the dataset $r$ times and show that greater the value of $r$ less the effectiveness of mini-batching to speed up) they never actually measure the gradient diversity in their experiments, or further study its properties. Note that they compute gradient diversity only on the training set since they use it for reasoning about optimization only. 
 
For our purposes, $\alpha$ is a better choice than $1/\alpha$---not just because coherence rather than incoherence is 
what leads to generalization---but also since the latter can diverge: $g$ can be 0 without all $g_z$ being zero (e.g., at the end of training in an under-parameterized setting). 
Furthermore, a better measure of the lack of coherence is $1 - \alpha$ which is the quantity that appears in our bound in Theorem~\ref{thm:cogen}. 

\citet{Yin18} also define a related quantity called {\em differential gradient diversity} $\Delta(w, w')$ which they use to reason about the relative stability of SGD with mini-batch size $b$ compared to SGD with mini-batch size of 1 (the case analyzed in \citet{Hardt16}). 
If $g_i(w)$ denotes the gradient of the $i$th training example at (parameter) $w$, the differential gradient diversity between $w$ and $w'$ (for $w \neq w')$\footnote{Though that does not guarantee that the denominator is non-zero.} is defined as follows:
\[
\Delta(w, w') 
\equiv 
\frac
    {\sum_{i = 1}^{m} \| g_i(w) - g_i(w') \|^2}
    {\| \sum_{i = 1}^{m} \left( g_i(w) - g_i(w') \right) \|^2}
\]
where $m$ is the size of the training set.

Their main result is that in the convex and strongly convex cases, if with high probability $b$ is smaller than $m \Delta(w, w')$ for all $w \neq w'$, then SGD with mini-batch size $b$ and SGD with mini-batch size 1 can be both stable in roughly the same range of step-sizes, and their generalization errors are roughly equal. (They do not consider the non-convex case.)
 
In contrast, our focus is on understanding generalization in the non-convex case, and for all batch sizes, including the full batch case.

\thoughtsep

Finally, when studying a form of adversarial initialization, 
\citet{Mehta21} (independently) introduce a notion very similar to $\alpha$ except (like stiffness and cosine similarity) they explicitly require $z \neq z'$ (but see discussion on this point above under mathematical simplicity). They show experimentally on small datasets such as {\sc cifar}-10, {\sc shvn}, etc. that within a class, their metric correlates well with generalization ability of the net when scale of initialization is varied.


\section{Proof of The Generalization Theorem}
\label{app:generalization_theorem}

We present a proof of \Cref{thm:cogen} which bounds the expected generalization gap of gradient descent (GD) in terms of the coherence as measured by $\alpha$ during training.
The proof goes through the notion of algorithmic stability~\citep{Bousquet02, Kearns99, Devroye79}, using the iterative framework first established by \citet{Hardt16} (using uniform stability) and later extended by \citet{Kuzborskij18} (using a more general notion of stability that we use as well). 

Both \citet{Hardt16} and \citet{Kuzborskij18} prove stability of stochastic gradient descent in the non-convex case (that is, for neural network training) by relying solely on {\em early stopping}. Their arguments require that no example is seen ``too often." 
The guarantee provided by \citet{Hardt16} is independent of the training dataset, and therefore, as noted before, cannot explain the generalization mystery. Indeed, their result says that stochastic gradient descent (SGD) is stable if run for a few steps with a step size that decays as $1/t$ (they summarize it under the slogan ``train faster, generalize better"), but such a short run is insufficient to memorize random labels. And, if SGD is run long enough to memorize random labels, then clearly it is not stable (since if a training example with a random label is missing in the dataset, the loss on it is expected to be much higher than otherwise). The bound in \citet{Kuzborskij18}, while improving on that of \citet{Hardt16}, by taking into account the curvature of the loss landscape at initialization (and thus incorporating some data-dependence), comes at the price of requiring that no example be seen more than once. 

Our approach towards arguing the data-dependent stability of gradient descent is fundamentally different since it is based on the coherence seen during training, and does not rely on solely on early stopping. Through coherence it is able to differentiate between the real data case and the random data case, providing different bounds in the two cases. An interesting consequence of this line of argument is that our proof applies equally well to non-stochastic gradient descent (that is, full-batch gradient descent) thus resolving an open question in \cite{Hardt16}.

\newbold{Preliminaries.}
Let $Z$ be a domain of examples, and let $\mathcal{D}(z)$ be a distribution on $Z$ from which the training and test examples are sampled. Let $\ell(w, z)$ denote the loss of the model with parameters $w$ on an example $z$. 
The goal is to find model with small {\em population risk}, defined as:
\[
R(w) \equiv \E_{z \sim \mathcal{D}} \ell(w,\,z).
\]
Let $S \equiv (z_1, .., z_m)$ be a sample of $m$ examples drawn i.i.d. from $\mathcal{D}$.
The {\em empirical risk} of a model $w$ on the sample $S$, denoted by $\hat{R}(w,\,S)$ is defined as follows:
\[
\hat{R}(w, \, S) \equiv \frac{1}{m} \sum_{i \in [m]} \ell(w,\,z_i).
\]
where $[n]$ denotes the set $\{1, 2, .., n\}$. 

Let $A$ be a (possibly) randomized training algorithm.
We are interested in the {\em expected generalization gap} of $A$ when training with $m$ examples from a distribution $\mathcal{D}$ which is given by, 
\[
\gap(\mathcal{D}, m) \equiv \E_{S \sim \mathcal{D}^m} \E_{\theta} \left[R({\bf A}_\theta(S)) - \hat{R}({\bf A}_\theta(S), \, S) \right]
\]
where ${\bf A}_\theta(S)$ denotes the model obtained by running $A$ on $S$ and $\theta$ denotes the sequence of coin tosses used by $A$ in a given run. 
(We omit the distribution of $\theta$ in the expression for expectation to reduce clutter.)
The expected generalization gap is similar to the notion on on-average generalization in \citet{ShalevShwartz10}.\footnote{Also see Lemma 14 in \citet{ShalevShwartz10} for how on-average generalization relates to a stronger notion of generalization under certain asymptotic assumptions on empirical risk minimization.}

As before, let $S \equiv (z_1, .., z_m)$ be a sample of $m$ examples drawn i.i.d. from $\mathcal{D}$. Furthermore, let $S' \equiv (z_1', .., z_m')$ be a second sample of $m$ examples drawn i.i.d. from $\mathcal{D}$. The {\em expected stability} of $A$ is given by, 
\[
\stab(\mathcal{D}, m) 
\equiv 
    \E_{S \sim \mathcal{D}^m} 
    \E_{S' \sim \mathcal{D}^m}  
    \E_\theta
        \left[
        \frac{1}{m} \sum_{i \in [m]} 
            \left[ \ell({\bf A}_\theta(S^{(i)}), z_i) - \ell({\bf A}_\theta(S), z_i) \right]
        \right]
\]
where $S^{(i)} \equiv (z_1, .., z_{i-1}, z_i', z_{i+1}, .., z_m)$, that is, $S^{(i)}$ is $S$ with $z_i$ replaced by $z_i'$.

Our notion of stability, expected stability,
most closely resembles the notion of {\em average-RO (replace-one) stability} introduced by \citet{ShalevShwartz10} but extended to randomized algorithms.%
\footnote{\citet{Kuzborskij18} also extend on-average stability to randomized algorithms, but they compute the supremum over $i$ instead of the mean.}
Like their Lemma 11 (and also the first equality of Lemma 7 in \cite{Bousquet02}), we have the following connection between expected stability and the expected generalization gap:

\begin{theorem}[Stability equals generalization] When training with $m$ samples from a distribution $\mathcal{D}$, we have, 
\[
\gap(\mathcal{D}, m) = \stab(\mathcal{D}, m).
\]
\label{thm:stabgen}
\end{theorem}

\begin{proof}
This proof is similar to that of Lemma 11 in \cite{ShalevShwartz10} and the first equality of Lemma 7 in \cite{Bousquet02}, and could be also derived from those arguments by considering ${\bf A}_\theta$ as a deterministic algorithm, and then taking expectations over $\theta$. 

From linearity of expectation, $\stab(\mathcal{D}, m)$ is equal to
\begin{align*}
    \E_{S \sim \mathcal{D}^m} 
    \E_{S' \sim \mathcal{D}^m}  
    \E_\theta
        \left[
            \frac{1}{m} \sum_{i \in [m]} \ell({\bf A}_\theta(S^{(i)}), z_i)
        \right]
-
    \E_{S \sim \mathcal{D}^m} 
    \E_{S' \sim \mathcal{D}^m}  
    \E_\theta
        \left[
            \frac{1}{m} \sum_{i \in [m]} \ell({\bf A}_\theta(S), z_i) 
        \right]
\end{align*}
Now, from the definition of $\hat{R}$, we can rewrite the second term of the difference as follows:
\[
\E_{S \sim \mathcal{D}^m} 
\E_{S' \sim \mathcal{D}^m}  
\E_\theta
    \left[
        \frac{1}{m} \sum_{i \in [m]} \ell({\bf A}_\theta(S), z_i) 
    \right]
=
\E_{S \sim \mathcal{D}^m} 
\E_\theta
    \left[ \hat{R}({\bf A}_\theta(S), \, S) \right],
\]
From linearity and exchangeability, we can rewrite the first term of the difference as follows:
\begin{align*}
\E_{S \sim \mathcal{D}^m} 
\E_{S' \sim \mathcal{D}^m}  
\E_\theta
    \left[
        \frac{1}{m} \sum_{i \in [m]} \ell({\bf A}_\theta(S^{(i)}),\,z_i) 
    \right]
& = 
\frac{1}{m} \sum_{i \in [m]} 
    \E_{S \sim \mathcal{D}^m} 
    \E_{S' \sim \mathcal{D}^m}  
    \E_\theta
        \left[
        \ell({\bf A}_\theta(S^{(i)}),\,z_i)
        \right] \\
& = 
\frac{1}{m} \sum_{i \in [m]} 
    \E_{S \sim \mathcal{D}^m} 
    \E_{S' \sim \mathcal{D}^m}  
    \E_\theta
        \left[
        \ell({\bf A}_\theta(S),\,z_i')
        \right] {\rm (from\ exchangeability)}\\
& =
\frac{1}{m} \sum_{i \in [m]} 
    \E_{S \sim \mathcal{D}^m} 
    \E_{z_i' \sim \mathcal{D}}  
    \E_\theta
        \left[
        \ell({\bf A}_\theta(S),\,z_i')
        \right] \\
& =
\frac{1}{m} \sum_{i \in [m]} 
    \E_{S \sim \mathcal{D}^m} 
    \E_\theta
    \E_{z_i' \sim \mathcal{D}}  
        \left[
        \ell({\bf A}_\theta(S),\,z_i')
        \right] \\
& = 
\frac{1}{m} \sum_{i \in [m]} 
    \E_{S \sim \mathcal{D}^m} 
    \E_\theta
        \left[
        R({\bf A}_\theta(S))
        \right] \\
& =
\E_{S \sim \mathcal{D}^m} 
\E_\theta
\left[
        R({\bf A}_\theta(S))
        \right].
\end{align*}
Therefore, 
\begin{align*}
\stab(\mathcal{D}, m)
&= 
    \E_{S \sim \mathcal{D}^m} 
    \E_\theta
        \left[ R({\bf A}_\theta(S)) \right]
-
    \E_{S \sim \mathcal{D}^m} 
    \E_\theta
        \left[ \hat{R}({\bf A}_\theta(S), \, S) \right] \\
&= 
    \E_{S \sim \mathcal{D}^m} 
    \E_\theta
        \left[ 
            R({\bf A}_\theta(S)) - \hat{R}({\bf A}_\theta(S), \, S)
        \right] \\
&=
    \gap(\mathcal{D}, m)
\end{align*}
\end{proof}

\newbold{Smoothness Assumptions.}
We make two standard assumptions on the smoothness of the loss function. These assumptions are also made by \cite{Hardt16} and \cite{Kuzborskij18}, and they are (recalling that $g(w, z)$ denotes $\nabla_w \ell(w, z)$):

\begin{enumerate}

\item We assume that $\ell(\cdot\,, z)$ is $L$-Lipschitz and differentiable for every $z \in Z$, that is, 
\begin{align}
    | \ell(w, z) - \ell(w', z) | \le L\,\| w - w' \|
\label{eq:ell_lip}
\end{align}
and 
\begin{align}
\| g(w, z) \| \le L.
\label{eq:ell_grad_bound}
\end{align}

\item We also assume that $\ell(\cdot\,,z)$ is $\beta$-smooth for every 
$z \in Z$. That is, 
we assume,
\begin{align}
    \| g(w, z) - g(w', z) \| \le \beta\,\| w - w' \|.
\label{eq:grad_lip}
\end{align}
\end{enumerate}

\newbold{Stability of (Stochastic) Gradient Descent.}
In order to analyze the stability of gradient descent, we compare what happens when we train with $S$ and when we train with $S^{(i)}$ for $T$ steps.
Let $w_0, w_1, .., w_T$ and $w_0', w_1', .., w_T'$ denote the sequence of
weights when training with $S$ and $S^{(i)}$ respectively.
Since both runs begin with the same random initialization, we have $w_0 = w_0'$.
Note that $w_t$ for $t \in [T]$ corresponds to the weight {\em after} the descent step at time $t$ has taken place, and therefore, the gradients used for step $t$ are evaluated at $w_{t - 1}$. (As usual $[T]$ denotes the set $\{1, 2, \ldots, T\}$.) 

We treat both stochastic and full-batch gradient descent in the same manner, with the latter being a special case of the former, where the batch size $b$ is equal to the number of training samples $m$.
Each mini-batch is constructed (independently of other mini-batches) by selecting $b$ training examples at random {\em without replacement} from the $m$ training examples.%
\footnote{We choose this specific scheme to analyze since it highlights the essential similarity between the stochastic and the non-stochastic case and is simple to analyze, but this choice is not essential. Other schemes can also be analyzed similarly.}
Let $\theta$ denote the sequence of coin tosses made by gradient descent to select the examples in each mini-batch (these tosses are, of course, is independent of the training sample). Let $I_t(\theta)$ be an indicator variable that is 1 if the $i$th training example is selected in
the mini-batch used at time step $t \in [T]$, and 0 otherwise. 
Therefore,
\begin{align}
\E_{\theta}\ [ I_t(\theta) ] = b / m
\label{eq:exp_I_t}
\end{align}

Let $z_{tj}$ and $z_{tj}'$ be the training examples selected in the two runs as the $j$th mini-batch entry at time $t$ where $j \in [b]$ and $t \in [T]$.
Let the step size at step $t$ be $\eta_t$. From the update rule for stochastic gradient descent, we have,
\[
w_{t} = w_{t - 1} - \eta_t \frac{1}{b} \sum_{j \in [b]} g(w_{t - 1}, z_{tj})
\]
and
\[
w_{t}' = w_{t - 1}' - \eta_t \frac{1}{b} \sum_{j \in [b]} g(w_{t - 1}', z_{tj}')
\]
Therefore,
\[
(w_{t} - w_{t}') = (w_{t - 1} - w_{t - 1}') 
  - \eta_t \frac{1}{b} \sum_{j \in [b]} \left[ g(w_{t - 1}, z_{tj}) - g(w_{t - 1}', z_{tj}') \right] 
\]

\noindent Let $\delta_t = \| w_t - w_t' \|$, and let $\Delta_t''(b) = \frac{1}{b} \sum_{j \in [b]} \left[ g(w_{t - 1}, z_{tj}) - g(w_{t - 1}', z_{tj}') \right]$. From the
above equation, taking norms on both sides, and the triangle inequality, we have,
\begin{align}
\delta_{t} \le \delta_{t - 1} + \eta_t \| \Delta_t''(b) \|
\label{eq:delta_double_prime}
\end{align}

\thoughtsep

We now express this expansion of the difference at step $t$ between the two runs {\em purely} in terms of the $i$th training example, that is, the example that is different in $S$ and $S^{(i)}$.
Intuitively, the following lemma shows that the expansion has two components: (1) one that is always present, and due to different starting points that gets magnified with the step, and (2) an occasional extra expansion if the mini-batch happens to include the $i$th example.

\begin{lemma}[1-Step Expansion] For $t \in [T]$, we have,
\label{lem:onestep}
\begin{align}
\delta_{t} 
& \le 
    (1 + \eta_t \, \beta) \cdot \delta_{t - 1} + 
        \frac{\eta_t I_t(\theta)}{b} \cdot \| g(w_{t - 1}, z_i) - g(w_{t - 1}, z_i') \|
\label{eq:delta_recursive}
\end{align}
\end{lemma}

\begin{proof}
Define, 
\[
\Delta_t'(s) = \frac{1}{s} \sum_{j \in [s]} \left[ g(w_{t - 1}, z_{tj}) - g(w_{t - 1}', z_{tj}) \right]
\]
Now, if $I_t(\theta) = 0$, then $z_{tj} = z_{tj}'$ for all $j \in [b]$, and we have, 
\[
\Delta_{t}''(b) = \frac{1}{b} \sum_{j \in [b]} \left[ g(w_{t - 1}, z_{tj}) - g(w_{t - 1}', z_{tj}) \right] = \Delta_t'(b)
\]
On the other hand, if $I_t(\theta) = 1$, then there is a $k \in [b]$ s.t. $z_{tk} = z_i$ and $z_{tk}' = z_i'$, and $z_{tj} = z_{tj}'$ for all $j \in [b] \setminus \{k\}$. 
Without loss of generality assume that $k = b$. 
Therefore, in this case, 
\[
\Delta_{t}''(b) = \frac{1}{b} (g(w_{t - 1}, z_i) - g(w_{t - 1}', z_i'))
    + \frac{b - 1}{b} \Delta_t'(b - 1)
\]
Putting the $I_t(\theta) = 0$ and the $I_t(\theta) = 1$ cases together, we get,
\begin{align*}
\Delta_{t}''(b) 
= 
    I_t(\theta) 
        \left( \frac{1}{b} (g(w_{t - 1}, z_i) - g(w_{t - 1}', z_i')) + \frac{b - 1}{b} \Delta_t'(b - 1) \right) 
  + (1 - I_t(\theta)) \Delta_t'(b)
\end{align*}
Taking norms, and applying the triangle inequality, we get,
\begin{align*}
\| \Delta_{t}''(b) \| 
\le
    I_t(\theta) 
        \left( \frac{1}{b} \| g(w_{t - 1}, z_i) - g(w_{t - 1}', z_i') \| + \frac{b - 1}{b}  \| \Delta_t'(b - 1) \| \right) 
  + (1 - I_t(\theta)) \| \Delta_t'(b) \|
\end{align*}
Now, from the smoothness condition on the loss functions (\ref{eq:grad_lip}), we have $\| \Delta_t'(s) \| \le \beta \, \delta_t$
for any $s \in [b]$, and this simplifies to
\begin{align*}
\| \Delta_{t}''(b) \| 
\le
    I_t(\theta) 
        \left( \frac{1}{b} \| g(w_{t - 1}, z_i) - g(w_{t - 1}', z_i') \| + \frac{b - 1}{b}  
        \beta \, \delta_t \right) 
  + (1 - I_t(\theta)) \beta \, \delta_t. 
\end{align*}
Now, since
\begin{align*}
\| g(w_{t - 1}, z_i) - g(w_{t - 1}', z_i') \| 
& =
\| g(w_{t - 1}, z_i) - g(w_{t - 1}, z_i') + g(w_{t - 1}, z_i') - g(w_{t - 1}', z_i') \| \\
& \le
\| g(w_{t - 1}, z_i) - g(w_{t - 1}, z_i') \| + \| g(w_{t - 1}, z_i') - g(w_{t - 1}', z_i') \| \\
& \le 
\| g(w_{t - 1}, z_i) - g(w_{t - 1}, z_i') \| + \beta \delta_t
\end{align*}
we get,
\begin{align*}
\| \Delta_{t}''(b) \| 
& \le
    I_t(\theta) 
        \left( \frac{1}{b} ( \| g(w_{t - 1}, z_i) - g(w_{t - 1}, z_i') \| + \beta \delta_t )  + \frac{b - 1}{b}  
        \beta \, \delta_t \right) 
  + (1 - I_t(\theta)) \beta \, \delta_t \\
& =
    I_t(\theta) 
        \left( \frac{1}{b} ( \| g(w_{t - 1}, z_i) - g(w_{t - 1}, z_i') \| )  +  \beta \delta_t  \right) 
  + (1 - I_t(\theta)) \beta \, \delta_t \\
& = 
    I_t(\theta) 
        \left( \frac{1}{b} ( \| g(w_{t - 1}, z_i) - g(w_{t - 1}, z_i') \| ) \right)  +  \beta \delta_t \\
\end{align*}
Recall from (\ref{eq:delta_double_prime}) that
\[
\delta_{t + 1} \le \delta_t + \eta_t \| \Delta_t''(b) \|
\]
Therefore, we have,
\begin{align*}
\delta_{t} 
& \le 
    \delta_{t - 1} + \eta_t \left( I_t(\theta) 
        \left( \frac{1}{b} ( \| g(w_{t - 1}, z_i) - g(w_{t - 1}, z_i') \| ) \right)  +  \beta \, \delta_{t - 1} \right)
\end{align*}
simplifying which the required result follows.
\end{proof}

\thoughtsep

Next, we unroll the recursion to bound the difference in weights at the end of training $\delta_T$
in terms of the differences in gradients in $S$ and $S^{(i)}$ at each step. (In what follows we use the standard convention for $\prod$ that an empty product is equal to 1.)

\begin{lemma}[Unrolling recursion] We have,
\begin{align}
\delta_{T} 
& \le
    \sum_{t \in [T]}
        \left( 
            \frac{\eta_t I_t(\theta)}{b} 
            \cdot
            \| g(w_{t - 1}, z_i) - g(w_{t - 1}, z_i') \| 
            \cdot
        \prod_{k = t + 1}^{T} (1 + \eta_{k} \, \beta) 
        \right)
\end{align}
\label{lem:unrolled}
\end{lemma}
\begin{proof}
By unrolling (\ref{eq:delta_recursive}) and
using the fact that $\delta_0 = 0$.
\end{proof}

\thoughtsep

To reduce clutter, it is convenient to define a couple of abbreviations:
\[
[\eta_{k}\,\beta]_{k = t_0}^{t_1} \equiv \prod_{k = t_0}^{t_1} (1 + \eta_{k} \, \beta)
\quad {\rm and} \quad
\bar{g}(w_t) \equiv \sqrt{\E_{z \sim \mathcal{D}} \  \left[ \| g(w_t,\,z) \|^2 \right]}.
\]
Intuitively, $[\eta_k\,\beta]_{k = t_0}^{t_1}$ denotes the expansion incurred from time step $t_0$ to $t_1$, and $\bar{g}(w_t)$ is the average norm of per example gradients at $w_t$ (roughly speaking).
Note that by the $L$-Lipschitz condition on $\ell(\cdot, z)$, that is, by (\ref{eq:ell_grad_bound}), we have, 
\begin{align}
    \bar{g}(w_t) \le L.
\label{eq:g_bar_bound}
\end{align}

Next, we bound the difference between examples by the coherence as measured by $\alpha$.

\begin{lemma}
\label{lem:alpha_intro}
We have, 
\begin{align*}
    \E_{S \sim \mathcal{D}^m} 
    \E_{z_i' \sim \mathcal{D}}  
    \E_\theta
        \left[ \left| \ell({\bf A}_\theta(S^{(i)}), z_i) - \ell({\bf A}_\theta(S), z_i) \right| \right]
 \le 
    \frac{L^2}{m} \ \sum_{t \in [T]}
    [\eta_{k}\,\beta]_{k = t + 1}^{T}
    \cdot
    \eta_t
    \cdot
    \sqrt{2 \  (1 - \alpha(w_{t - 1}))}
\end{align*}
\label{lem:tech}
\end{lemma}

\begin{proof}
Since, by definition, $w_T = {\bf A}_\theta(S)$, $w_T' = {\bf A}_\theta(S^{(i)})$, and $\delta_t = \|w_T - w_T'\|$, applying the Lipschitz condition on $\ell$ (\ref{eq:ell_lip}), and then applying Lemma~\ref{lem:unrolled} we get, 
\begin{align*}
\left| \ell({\bf A}_\theta(S^{(i)}), z_i) - \ell({\bf A}_\theta(S), z_i) \right| 
& \le L \ \delta_T \\
& =  
    L \ \sum_{t \in [T]}
        \left( 
            \frac{\eta_t I_t(\theta)}{b} 
            \cdot
            [\eta_{k}\,\beta]_{k = t + 1}^{T}
            \cdot
            \| g(w_{t - 1}, z_i) - g(w_{t - 1}, z_i') \| 
        \right)
\end{align*}
Now, taking expectations on both sides, and expanding into the sum, 
we get the following for each term of the sum:
\begin{align*}
&   \E_{S \sim \mathcal{D}^m} 
    \E_{z_i' \sim \mathcal{D}}  
    \E_\theta
        \left[ 
            [\eta_{k}\,\beta]_{k = t + 1}^{T}
            \cdot
            \frac{\eta_t I_t(\theta)}{b} 
            \cdot
            \| g(w_{t - 1}, z_i) - g(w_{t - 1}, z_i') \| 
        \right] \\
= & \ 
    [\eta_{k}\,\beta]_{k = t + 1}^{T}
    \cdot
    \left(
        \frac{\eta_t}{b} \E_\theta \left[ I_t(\theta) \right]
    \right)
    \cdot
    \left(
        \E_{S \sim \mathcal{D}^m}  
        \E_{z_i' \sim \mathcal{D}} 
            \| g(w_{t - 1}, z_i) - g(w_{t - 1}, z_i') \| 
    \right) \\
\le & \ 
    [\eta_{k}\,\beta]_{k = t + 1}^{T}
    \cdot
    \left( \frac{\eta_t}{m} \right)
    \cdot
    \sqrt{2 \  (1 - \alpha(w_{t - 1}))}
    \cdot
    \bar{g}(w_{t - 1}) \\
\le & \ 
    [\eta_{k}\,\beta]_{k = t + 1}^{T}
    \cdot
    \left( \frac{\eta_t}{m} \right)
    \cdot
    \sqrt{2 \  (1 - \alpha(w_{t - 1}))}
    \cdot
    L
\end{align*}
The second step follows from the first since $\theta$ is independent of $S$ and $z_i'$. The 
third step follows from $\E_\theta \left[ I_t(\theta) \right] = \frac{b}{m}$ (see (\ref{eq:exp_I_t})), and  from Lemma~\ref{lem:diffbound} applied to bound the difference in the gradients for $z_i$ and $z_i'$ by the coherence of the per-example gradients at $w_t$. The last step follows from (\ref{eq:g_bar_bound}).

Finally, by summing up the bounds for each term and rearranging, we get the required result.

\end{proof}

\thoughtsep

A straightforward consequence of the above lemma is a bound on the stability of gradient descent.

\begin{theorem}[Stability Theorem]
\begin{align}
|\stab(\mathcal{D}, m)| 
\le 
    \frac{L^2}{m} \ \sum_{t \in [T]}
    [\eta_{k}\,\beta]_{k = t + 1}^{T}
    \cdot
    \eta_t
    \cdot
    \sqrt{2 \  (1 - \alpha(w_{t - 1}))}
\end{align}
\label{thm:costab}
\end{theorem}

\begin{proof}
This is a straightforward consequence of Jensen's inequality and Lemma~\ref{lem:tech}.
\begin{align*}
|\stab(\mathcal{D}, m)|
&= 
    \left|
    \E_{S \sim \mathcal{D}^m} 
    \E_{S' \sim \mathcal{D}^m}  
    \E_\theta
        \left[
        \frac{1}{m} \sum_{i \in [m]} 
            \left[ \ell({\bf A}_\theta(S^{(i)}), z_i) - \ell({\bf A}_\theta(S), z_i) \right]
        \right]
    \right| \\
&\le
    \frac{1}{m} \sum_{i \in [m]} 
    \E_{S \sim \mathcal{D}^m} 
    \E_{S' \sim \mathcal{D}^m}  
    \E_\theta
        \left[ \left| \ell({\bf A}_\theta(S^{(i)}), z_i) - \ell({\bf A}_\theta(S), z_i) \right| \right] \\
&=
    \frac{1}{m} \sum_{i \in [m]} 
    \E_{S \sim \mathcal{D}^m} 
    \E_{z_i' \sim \mathcal{D}}  
    \E_\theta
        \left[ \left| \ell({\bf A}_\theta(S^{(i)}), z_i) - \ell({\bf A}_\theta(S), z_i) \right| \right] \\
&\le 
    \frac{L^2}{m} \ \sum_{t \in [T]}
    [\eta_{k}\,\beta]_{k = t + 1}^{T}
    \cdot
    \eta_t
    \cdot
    \sqrt{2 \  (1 - \alpha(w_{t - 1}))}
\end{align*}
\end{proof}

\newbold{Remark.} Note that in the penultimate step we bound the following quantity (using \Cref{lem:alpha_intro}):
\[
    \frac{1}{m} \sum_{i \in [m]} 
    \E_{S \sim \mathcal{D}^m} 
    \E_{z_i' \sim \mathcal{D}}  
    \E_\theta
        \left[ \left| \ell({\bf A}_\theta(S^{(i)}), z_i) - \ell({\bf A}_\theta(S), z_i) \right| \right]
\]
which corresponds to a stronger notion of stability than the notion of expected stability that we use to link to generalization. Indeed, it is closer to the notion of {\em point-wise hypothesis stability} \citep{Bousquet02, Kearns99} with the main differences being that it is formulated in terms of replace-one (instead of leave-one-out), for a randomized setting (instead of a deterministic one), and does not assume symmetry in the training examples. 
Therefore, it may be possible to extend our results to provide a concentration guarantee along the lines of Theorem 11 in \citet{Bousquet02}.

We are now ready to prove our main theorem.

\begin{i:thm:cogen}[Generalization Theorem]
\begin{align}
|\gap(\mathcal{D}, m)| 
\le 
    \frac{L^2}{m} \ \sum_{t \in [T]}
    [\eta_{k}\,\beta]_{k = t + 1}^{T}
    \cdot
    \eta_t
    \cdot
    \sqrt{2 \  (1 - \alpha(w_{t - 1}))}
\end{align}
\end{i:thm:cogen}

\begin{proof}
This is a straightforward consequence of Theorem~\ref{thm:stabgen} and Theorem~\ref{thm:costab}.
\end{proof}

\thoughtsep

We now consider special cases for two common learning rate schedules. These allow us to replace the expansion term in the bound with an exponential.

\begin{corollary} If the step sizes are fixed, that is, $\eta_t = \eta$ then 
\begin{align}
|\gap(\mathcal{D}, m)| 
\le 
    \frac{L^2 \eta}{m} \ 
    \sum_{t \in [T]}
        {\rm exp}((T - t)\,\eta\,\beta)
        \cdot
        \sqrt{2 \  (1 - \alpha(w_{t - 1}))}
\end{align}
\label{cor:fixedstepsize}
\end{corollary}

\begin{proof}

For any $t \ge 0$, 
\begin{align*}
[\eta_{k}\,\beta]_{k = t + 1}^{T} 
\le 
{\rm exp} \left( \sum_{k = t + 1}^{T} \eta_k\,\beta \right) 
\le
{\rm exp} \left( \eta\beta \sum_{k = t + 1}^{T} 1 \right)
\le
{\rm exp} \left( \eta\beta (T - t) \right)
\end{align*}
where the first inequality follows from $(1 + x) \le {\rm exp}(x)$ for any $x$.
The required result follows substituting this bound on the expansion term in Theorem~\ref{thm:cogen}.
\end{proof}

\newbold{Remark.} Note that the bound in \Cref{cor:fixedstepsize} may be simplified further to
\begin{align}
|\gap(\mathcal{D}, m)| 
\le 
    \frac{L^2\eta\,{\rm exp}(T\,\eta\,\beta)}{m} \ 
    \sum_{t \in [T]}
        \sqrt{2 \  (1 - \alpha(w_{t - 1}))}.
\end{align}

\thoughtsep

\begin{corollary} If we assume as in ~\citet{Hardt16} that step sizes decay linearly, that is, for some $\eta > 0$ we have $\eta_t \le \eta/t$ then 
\begin{align}
|\gap(\mathcal{D}, m)| 
\le 
    \frac{L^2 \eta T^{\eta \beta}}{m} \ 
    \sum_{t \in [T]}
        \sqrt{2 \  (1 - \alpha(w_{t - 1}))}
\end{align}
\end{corollary}

\begin{proof}

For any $t \ge 0$, 
\begin{align*}
[\eta_{k}\,\beta]_{k = t + 1}^{T} 
\le 
{\rm exp} \left( \sum_{k = t + 1}^{T} \eta_k\,\beta \right) 
\le
{\rm exp} \left( \eta\beta \sum_{k = t + 1}^{T} 1/t \right)
\le
{\rm exp} \left( \eta\beta {\rm log} T \right)
= T^{\eta\beta}
\end{align*}
where, as in the previous proof, the first inequality follows from $(1 + x) \le {\rm exp}(x)$ for any $x$.
The required result follows substituting this bound in Theorem~\ref{thm:cogen} and from noticing that $\eta_t \le \eta$ for all $t$.
\end{proof}

\clearpage
\section{Methods to Measure \texorpdfstring{$\alpha$}{Alpha}}
\label{app:measuring-coherence}

\newbold{The Direct and Imputed Methods.}
In theory, $\alpha$ at any point in training can be directly measured on a test or training sample by computing per-example gradients, and computing it from the definition (\ref{eq:overview:metric}): 
\begin{equation*}
\alpha \equiv 
\frac
    {\displaystyle \E_{z \sim \mathcal{D}}\ [\ g_z\ ] \cdot \E_{z \sim \mathcal{D}}\ [\ g_z\ ]}
    {\displaystyle \E_{z \sim \mathcal{D}}\ [ g_z \cdot g_z ]}
\end{equation*}
where $\mathcal{D}$ is the empirical distribution of the sample.
$\alpha$ can be computed without storing per-example gradients by keeping separate running sums for $\E_{z \sim \mathcal{D}}\ [\ g_z\ ]$ and for $\E_{z \sim \mathcal{D}}\ [ g_z \cdot g_z ]$. We call this method the {\em direct} method.

In practice, however, for most modern networks, the presence of batch normalization layers~\citep{Ioffe15} makes it challenging to compute (or even define!) per-example gradients, since the gradient of an example depends on all the other examples in the mini-batch, and one cannot calculate a gradient of a single example in isolation from the others.
One way to get around this is to estimate per-example gradients in evaluation (or inference) mode by using the moving averages of the population for normalization which disentangles the gradient of an example from the others in the batch. Although we used that technique in our preliminary work~\citep{Chatterjee20c}, we found two problems with that approach:

\begin{enumerate}
    \item {\bf Lag in the moving averages.} During training, particularly, when the loss changes rapidly, the moving averages fall behind resulting in highly inflated coherence measurements.
    
    \item {\bf Incorrect gradients.} With batch normalization, an example contributes not just to its own loss, but also to the losses of the other examples through its contribution to the batch statistics. By using the population statistics instead of batch statistics, we ignore this component of the gradient of the example which leads to a disconnect between the gradients used to compute coherence and the ones actually used to update weights which is unsatisfactory. 
    
\end{enumerate}
Even though the first problem can be addressed (at an added compute cost) by updating the moving averages with the population statistics before computing the per-example gradients, there does not appear to be an easy solution for the second.\footnote{Note: this is indeed the main research challenge in replacing batch normalization with a technique that would work even for batches of size 1.}

Therefore, in our experiments, we handle batch normalization in a different manner by appealing to Theorem~\ref{thm:stylized_mini_batching} which relates the coherence of mini-batch gradients $\alpha(\mathcal{W})$ to the coherence of per-example gradients $\alpha(\mathcal{V})$. We invert equation (\ref{eq:batchformula}) to get:
\begin{equation}
\label{eq:inversebatchformula}
    \alpha(\mathcal{V}) = 
    \frac{\alpha(\mathcal{W})}{k - (k - 1) \cdot \alpha(\mathcal{W})}
\end{equation}
where $k$ is the batch size.

In this method, which we call the {\em imputed} method, to compute per-example coherence for $m$ examples, we compute the gradient of each batch in the usual manner for training, and then compute the coherence between the $m / k$ {\em batch gradients} from the definition. Finally, we {\em impute} a per-example coherence using (\ref{eq:inversebatchformula}). 
When computing the corresponding $\ralpha$ for the imputed per-example gradients we use $m$ for the sample size (rather than $m / k$).

Note that Theorem~\ref{thm:stylized_mini_batching} is proven under the assumption that each mini-batch is constructed by picking examples at random {\em with} replacement, although in practice, for efficiency, mini-batches are constructed at random {\em without} replacement.  
However, in the case of networks that do not use batch normalization, when we compare the per-example coherences computed using the direct and the imputed methods, we find them to be in close agreement (for example, see \Cref{fig:alexnet_imagenet_direct_vs_imputed_000}).
In our experiments we use the direct method wherever possible (that is, for networks without batch normalization) and resort to the imputed method where needed (for batch normalization). Note though that the imputed method is much more efficient than the direct method since it does not require the computation of per-example gradients.

\begin{figure*}
\centering
\includegraphics[width=0.65\textwidth]{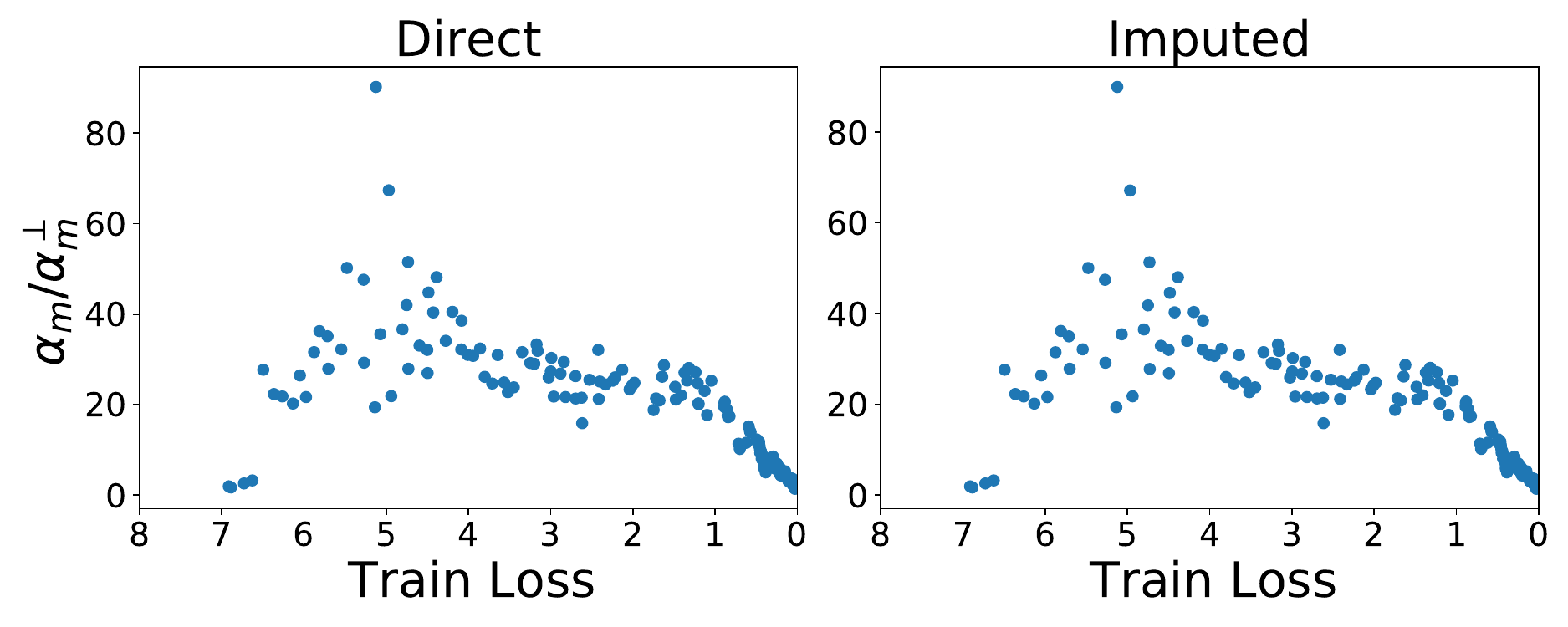}
\caption{$\alpha$ computed using direct and imputed methods on AlexNet (without batch normalization) for ImageNet. The methods produce virtually identical results.}
\label{fig:alexnet_imagenet_direct_vs_imputed_000}
\end{figure*}

\newbold{Effect of sample size used for estimation.} 
Please see \Cref{fig:m_variation}.

\begin{figure*}
\centering
\includegraphics[width=0.65\textwidth]{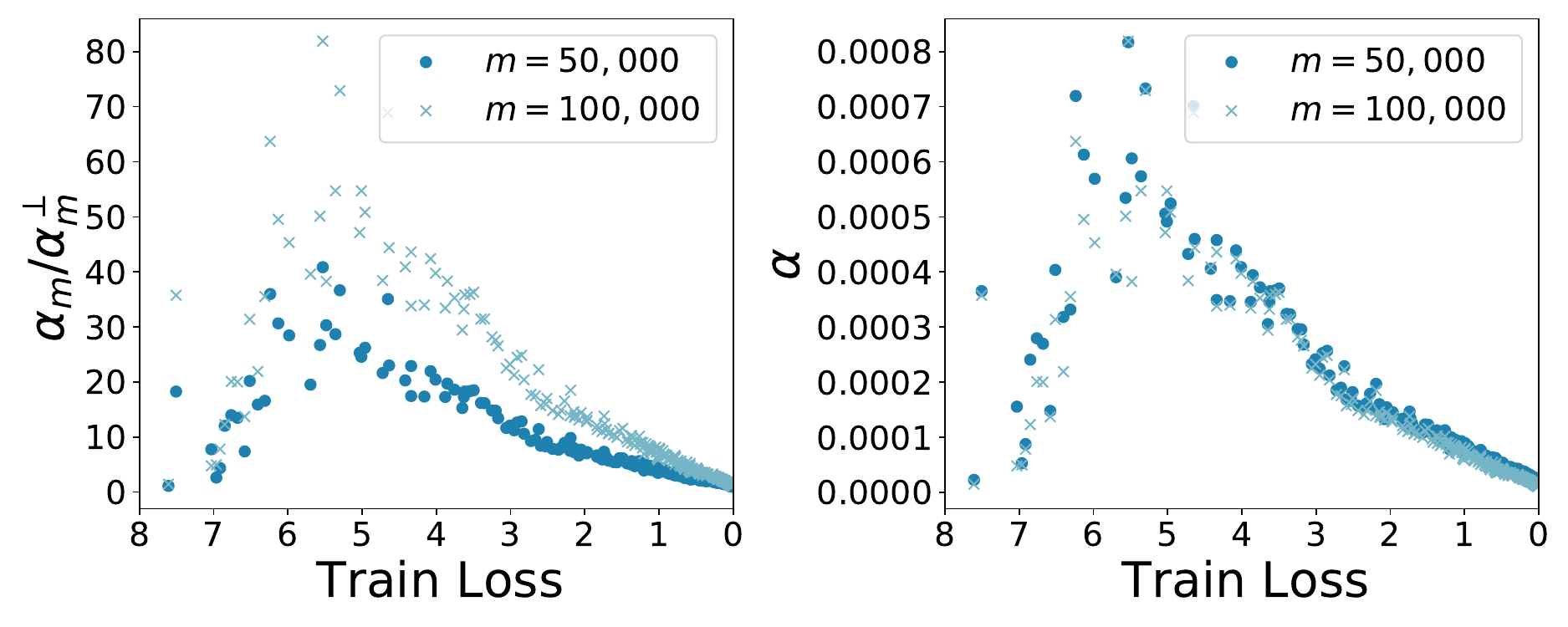}
\caption{$\alpha_m$ and $\ralpha$ on training samples of sizes $m =$ 50,000 and $m =$ 100,000 for a ResNet-50 being trained on ImageNet. The exact size of the sample used to measure $\alpha$ does not matter much if the sample is large enough. However, as discussed in \Cref{sec:overview:metrics}, the smaller sample produces a (slightly) greater estimate of $\alpha$ than the larger sample. To a first order, $\ralpha$ simply scales with the sample size (since it is $m\,\alpha$), but note that at the end of training each sample reaches its (respective) orthogonal limit ($\ralpha = 1$).}
\label{fig:m_variation}
\end{figure*}

\section{Measuring \texorpdfstring{$\alpha$}{Alpha} on Additional Datasets and Architectures}
\label{app:coherence-additional}

We measured coherence during training on the test set and a training sample (of the same size as test) for the following networks and datasets:
\begin{enumerate}
    \item Simple feed forward network with a single hidden layer of 2048 nodes on {\sc mnist}~\citep{lecun10} with random labels as a source of noise (\Cref{fig:figure1_mnist_fc_direct}).
    \item A simple convolutional network (AlexNet \citep{alexnet}) on {\sc cifar}-10 \citep{cifar10} with random labels as a source of noise (\Cref{fig:figure1_cifar_alexnet_direct}).
    \item AlexNet on ImageNet \citep{deng2009imagenet} with random pixels (\Cref{fig:chap1:alexnetimagenet_overall}) and with random labels (\Cref{fig:figure1_alexnet_imagenet_random_labels_a}).
    \item A deep residual convolutional network (ResNet-50 \citep{He16}) on ImageNet with random pixels (\Cref{fig:chap1:ResNet-50imagenet}) and with random labels (\Cref{fig:figure1_resnet50_imagenet_random_labels_a}).
\end{enumerate}
Each network was trained on the original dataset (0\% noise), as well as the corresponding dataset with either the labels or pixels randomized (100\% noise).
In all cases, we used vanilla stochastic gradient descent with a fixed learning rate. The batch size, learning rate, and the total number of steps was chosen so as to reach almost perfect training accuracy for random data in a reasonable amount of training time.
Since our goal is to study the implicit regularization of gradient descent, no explicit regularizers (such as dropout, weight-decay, or input augmentation) were used.

\begin{figure}
\centering
\includegraphics[width=\textwidth]{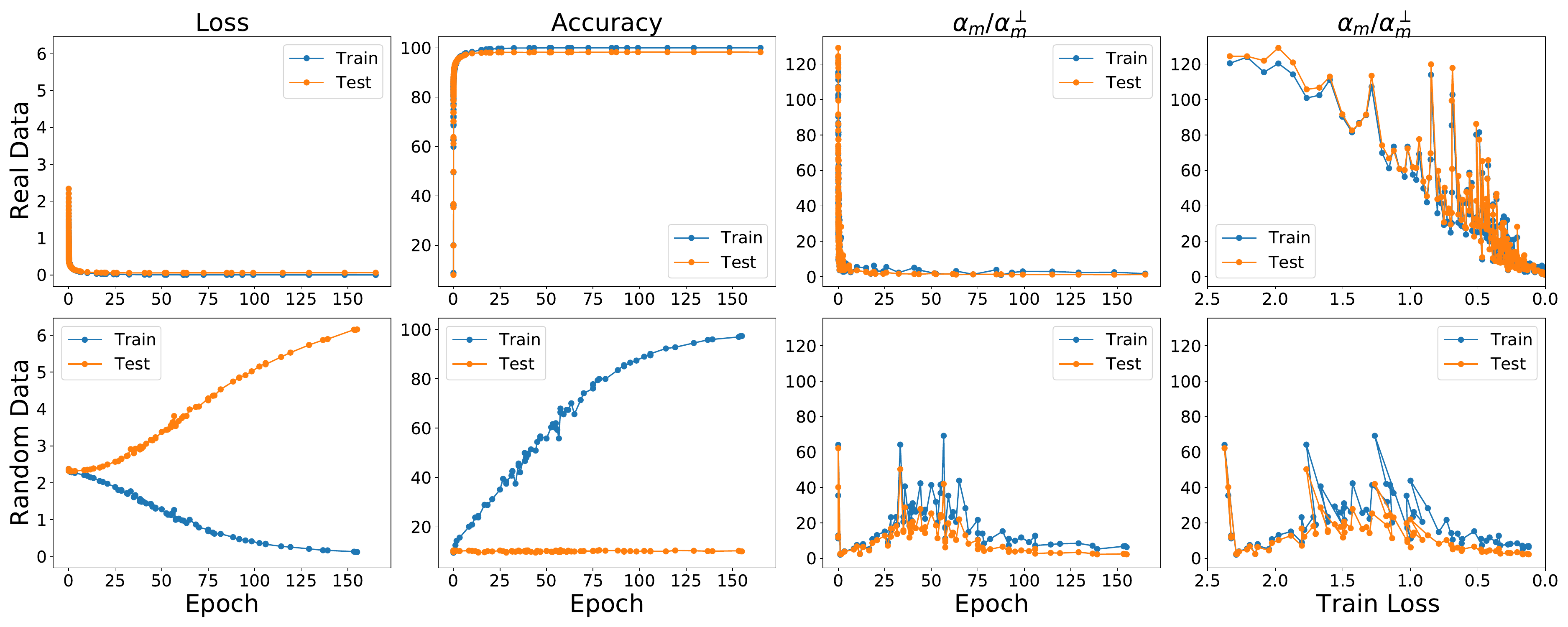}
\caption{The evolution of alignment of per-example gradients during training as measured with $\ralpha$ on samples of size $m=10{,}000$ on {\sc mnist} dataset. The model is a simple feed-forward network
with a single hidden layer containing 2048 neurons. We use SGD with a constant learning rate of 0.1. No regularizers were used.}
\label{fig:figure1_mnist_fc_direct}
\end{figure}

\begin{figure}
\centering
\includegraphics[width=\textwidth]{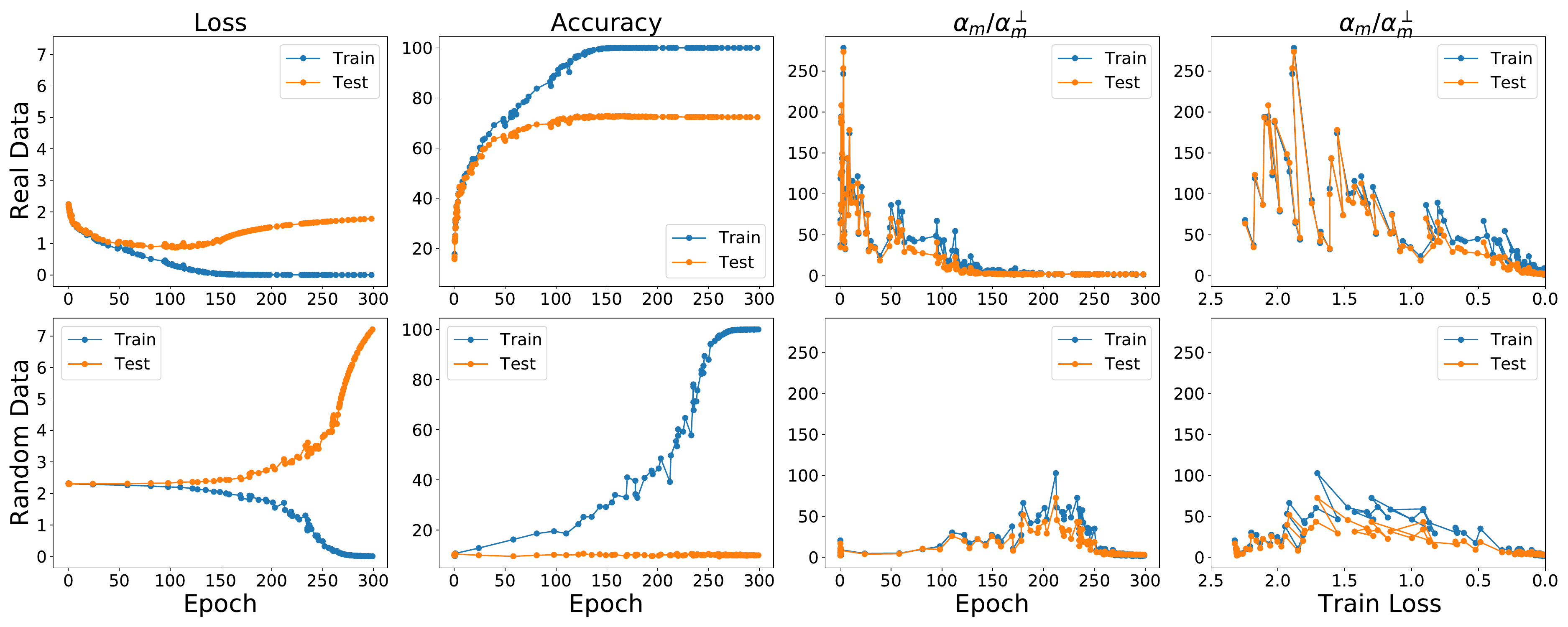}
\caption{The evolution of alignment of per-example gradients during training as measured with $\ralpha$ on samples of size $m=10{,}000$ on {\sc cifar10} dataset. The model is a AlexNet architecture without dropout layer. We train for 300 epochs with SGD using constant learning rate of 0.005. No regularizers were were used. Input images were rescaled to 224x224. Additional runs can be found in \Cref{fig:figure1_cifar_alexnet_direct_c}.}
\label{fig:figure1_cifar_alexnet_direct}
\end{figure}

\begin{figure}
\centering
\includegraphics[width=\textwidth]{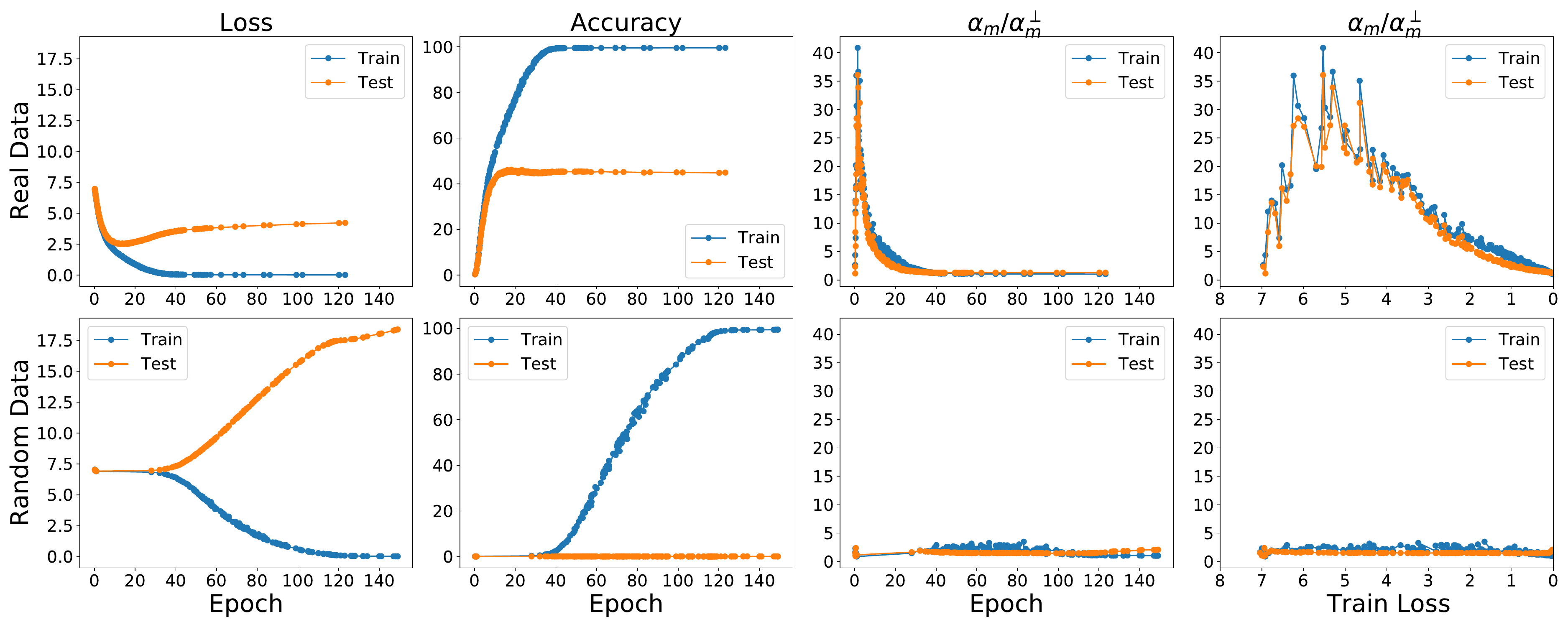}
\caption{The evolution of alignment of per-example gradients during training as measured with $\ralpha$ on samples of size $m=50{,}000$ on ImageNet dataset. Noise was added through labels randomization. The model is a Resnet-50. Additional runs can be found in \Cref{fig:figure1_resnet50_imagenet_random_labels_c}.}
\label{fig:figure1_resnet50_imagenet_random_labels_a}
\end{figure}

\begin{figure*}
\centering
\includegraphics[width=\textwidth]{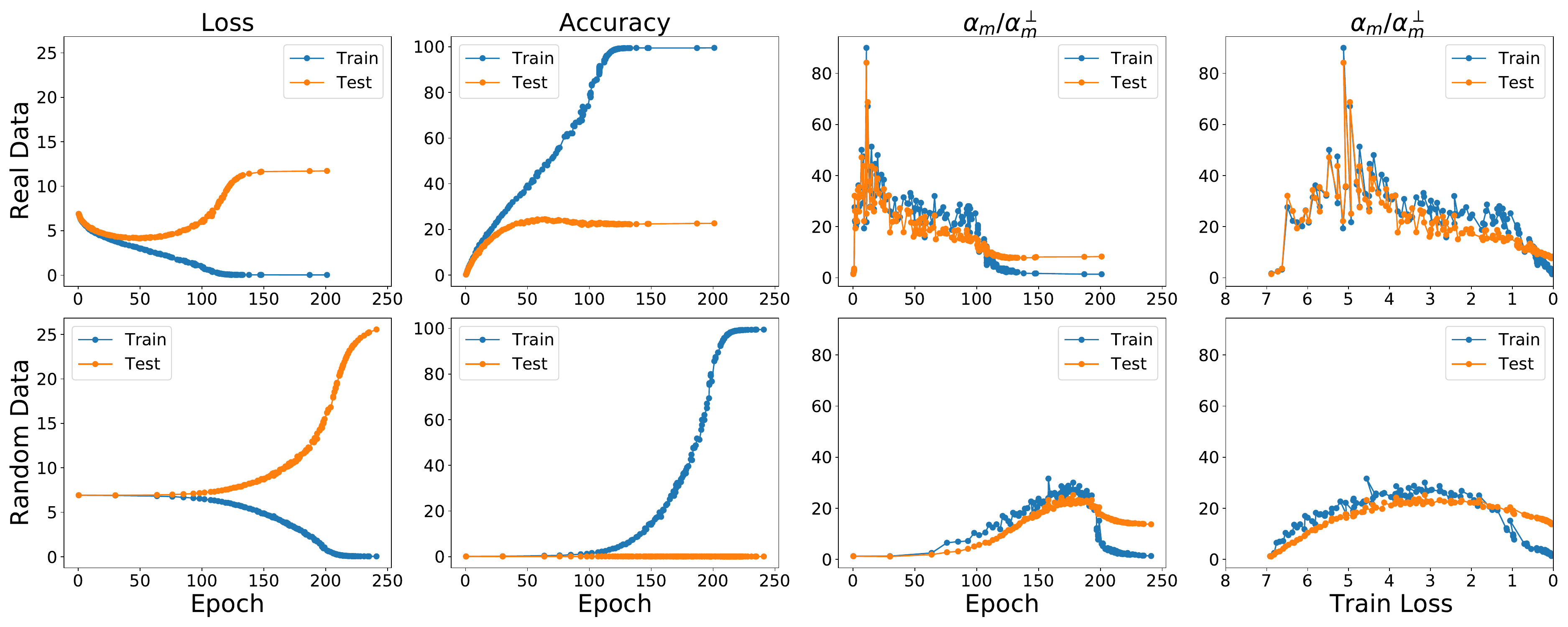}
\caption{The evolution of alignment of per-example gradients during training as measured with $\ralpha$ on samples of size $m=50{,}000$ on ImageNet dataset. Noise was added through labels randomization. The model is an AlexNet architecture without dropout layer. Additional runs can be found in \Cref{fig:figure1_alexnet_imagenet_random_labels_c}.}
\label{fig:figure1_alexnet_imagenet_random_labels_a}
\end{figure*}

\clearpage
\section{The Evolution of Coherence}
\label{app:evolution}

\begin{figure*}
\centering
\includegraphics[width=\textwidth]{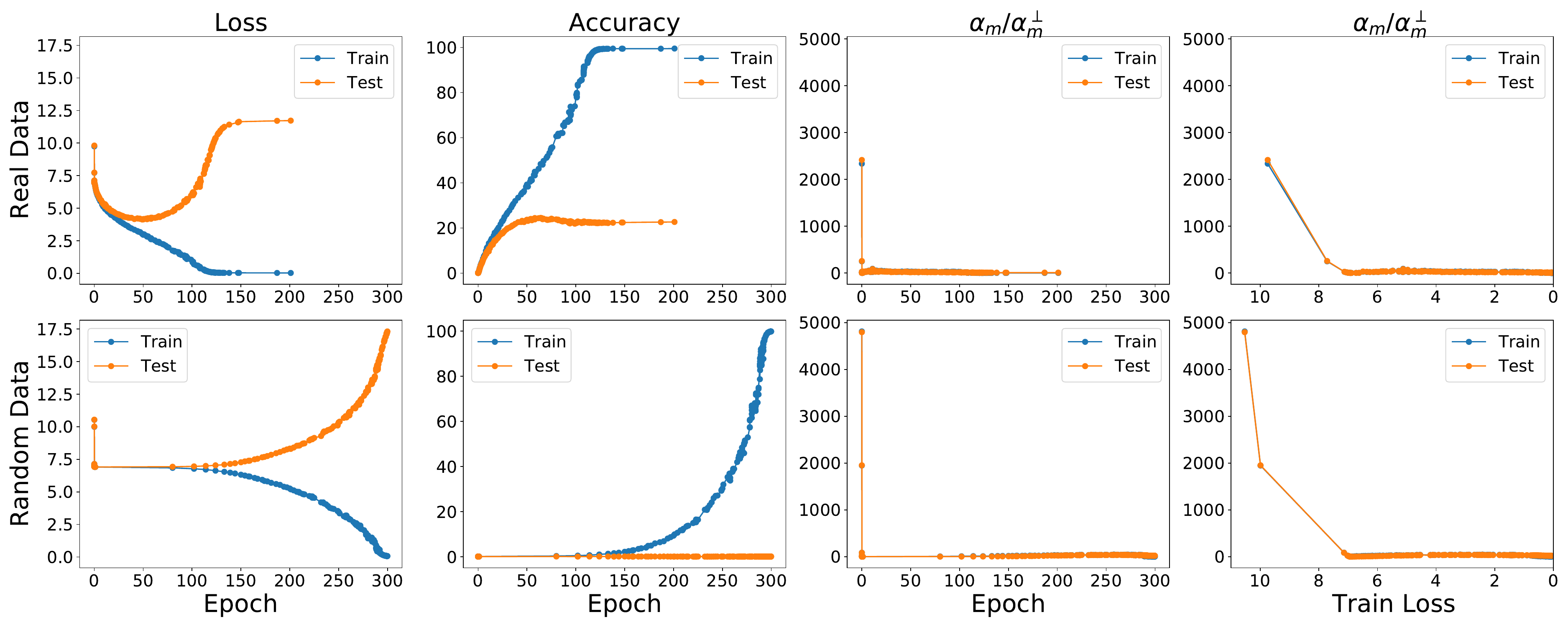}
\caption{In the first few steps of training $\ralpha$ can be very high due to imperfect initialization and in a sample of $50,000$ training examples, each example helps thousands of others, even in case of 100\% noise. In our usual plots of $\ralpha$, we omit this initial transient. This plot shows the transient when training AlexNet on ImageNet (\Cref{fig:chap1:alexnetimagenet_overall}).}
\label{fig:coherence_first_steps}
\end{figure*}

As mentioned in \Cref{sec:overview:measurement}, experiments across several architectures and datasets show a common pattern in how coherence as measured by $\ralpha$ (or equivalently, $\alpha$) changes during training.
Ignoring the initial transient in the first few steps of training (which we discuss a little later), coherence follows a roughly parabolic pattern:
It starts off at a low value, rises to a peak, and then comes back down to the orthogonal limit.%
\footnote{In rare cases, such as a fully connected network on {\sc mnist} where the signal is strong and easy to find, coherence starts off high.}
This happens regardless of whether the dataset is random or real, indicating that this is an optimization (as opposed to a generalization) effect.

This parabolic evolution of coherence is noisy with many minor fluctuations and deviations from the broad pattern. This suggests that the dynamics of coherence can be understood as being governed by two competing forces during training: once that causes coherence to increase (``creation") and one that causes coherence to decrease (``consumption").

Of the two, coherence consumption is easier to explain. If we assume that 
as examples get fitted during training, their gradients either become negligible, or orthogonal to the remaining (unfitted) examples, then
it can be shown that as the fraction of ``fitted examples" increases, $\alpha$ must decrease
For more details see \Cref{lem:zero} and \Cref{lem:two_orthogonal_sets} and associated examples in \Cref{app:alphafacts}.

Coherence creation, on the other hand, is harder to explain.
One way to reason about it is through the decomposition of the overall coherence in terms of the coherence of the constituent parts. 
As shown in \Cref{thm:weightedmean}, $\alpha$ as measured over the entire network (that is, over the entire gradient vector) is a convex combination of the $\alpha_i$ of its constituent parts (projections of the gradient vector). 
In particular, we have that
\[
\alpha = \sum_i f_i\,\alpha_i
\]
where $\alpha_i$ is the coherence of the $i$th parameter of the network, and 
$0 \leq f_i \leq 1$, and $\sum_i f_i = 1$.
Given this decomposition, there are two (non mutually exclusive) ways in which the overall coherence $\alpha$ can increase during training: 
(1) there is a shift in weight from a parameter $i$ with low $\alpha_i$ to one with higher $\alpha_i$, or 
(2) $\alpha_i$ for some parameter $i$ increases.

\Cref{ex:two-deep} provides a good example for first. Observe that the $\alpha$ for each $u^{(i)}$ and $w^{(i)}$ is constant during training, but the overall $\alpha$ increases. This happens due to shift in weight from the other idiosyncratic components to $u^{(6)}$ and $w^{(6)}$ due to amplification.
The following example demonstrates the second, and in fact demonstrates how $\alpha$ can follow a parabolic trajectory.

\begin{example}[Single parameter trajectory]
Consider the task of fitting a linear model $y = w \cdot x$ with a single parameter $w \in {\bb R}$ to the $m = 2$ points $(x_1, y_1)$ and $(x_2, y_2)$ under the usual mean squared loss.
It is easy to check that the point of minimum $\alpha_m$ is when the per-example gradients are equal in magnitude and opposite in direction, that is, at $w^{*} = (y_1 x_1 + y_2 x_2) / (x_1^2 + x_2^2)$. (Note that this is also the $w$ that optimizes the average loss.)
The point of maximum $\alpha_m$ is when the per-example gradients are equal in both magnitude and direction, that is, at $w^{\dagger} = (y_1 x_1 - y_2 x_2) / (x_1^2 - x_2^2)$.

Now let $(x_1, y_1) = (2, 3)$ and $(x_2, y_2) = (1, 9)$. Therefore, $w^{*} = 3$ and $w^{\dagger} = -1$. If we plot $\alpha_m$ as a function of $w$ for this training set, we see the following:
\begin{center}
\phantom{M}\\
\includegraphics[width=0.45\textwidth]{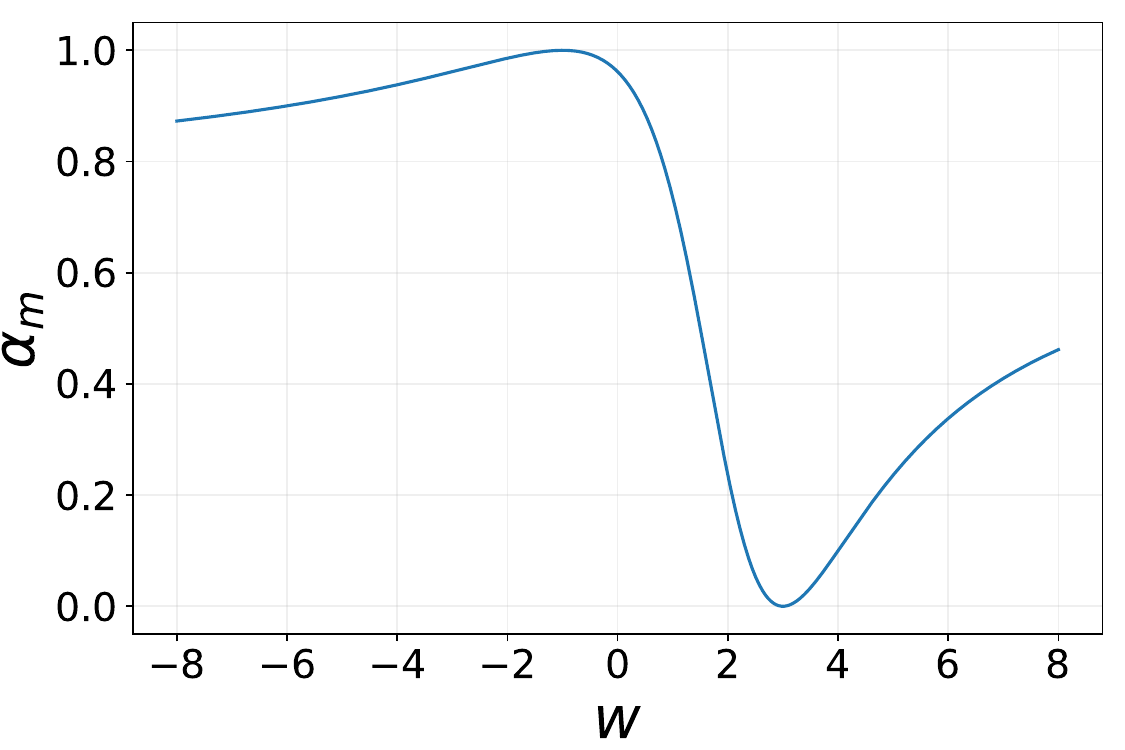}
\phantom{M}\\
\end{center}

If gradient descent started is initialized at a point to the left of $w^{\dagger}$, along its trajectory, we would observe an initial increase in $\alpha_m$ as the gradients of the two examples equalized in magnitude.
Once they are equal and $\alpha_m$ is 1, coherence begins to decrease until the two gradients are equal and opposite at the final solution $w^{*}$ and $\alpha_m$ is 0.
Note that the overall trajectory of $\alpha_m$ is roughly parabolic.

\qed
\end{example}

\newbold{Coherence in first few steps of training.} 
As mentioned in \Cref{sec:overview:measurement}, very early in training,
$\ralpha$ can be very high even for random data due to imperfect initialization. All the training examples are coordinated in moving the network to a more reasonable point in parameter space. As may be expected from our theory, this movement generalizes well: the test loss decreases in concert with training loss in this period. Rapid changes to the network early in training is well documented (see for example, the need for learning rate warmup in \citet{He16} and \citet{Goyal17}).

\Cref{fig:coherence_first_steps} shows an example of high coherence in the first few steps of training for AlexNet (in the same experiment as \Cref{fig:chap1:alexnetimagenet_overall}). We observe that transient in the first few steps when the loss is above what might be expected from an uniform distribution, that is, when the loss is above ${\rm log}(1000) \approx 6.9$ (recall that ImageNet has 1000 labels).

\clearpage
\section{Experimental Details of Easy and Hard Examples}
\label{app:easyhard}

\newbold{Hardness (difficulty) of examples.}
We estimate the difficulty of each of the 1,281,167 examples from the ImageNet training dataset by training a ResNet-50 to 50\% training accuracy, and recording for each example, whether it was correctly classified or not. We repeat this 8 times from different random initializations,
and assign a {\em hardness score} to each example  
by counting how many times it failed to be correctly classified.
A score of 0 indicates that the example was correctly classified on all 8 runs (a ``super easy" example) whereas a score of 8 indicates it was never correctly classified (``super hard"). 
The distribution of the hardness score is as follows:

\thoughtsep

\begin{center}
\begin{tabular}{l|r}
hardness score & \multicolumn{1}{c}{count} \\
\hline
8 (super hard) & 248,551 \\
7 & 121,290 \\
6 & 101,469 \\
5 & 95,780 \\
4 & 97,069 \\
3 & 103,182 \\
2 & 118,701 \\
1 & 142,939 \\
0 (super easy) & 252,186 \\
\end{tabular}
\end{center}

\thoughtsep

We construct the Easy ImageNet dataset described in Section \ref{sec:overview:easyhard} by randomly picking 550K examples that have a score of 3 or lower, and randomly allocating them into  train (500K examples) and test (50K examples) subsets. Hard ImageNet is constructed similarly by randomly picking 550K examples that have a score of 5 or greater. (Examples with a score of 4 are discarded.)

\newbold{In situ measurements.}
We can track coherence on examples of different levels of difficulty during regular ImageNet training.%
\footnote{One issue with tracking coherence on two different subsets of the training set is that the batch normalization statistics in each subset set may be different than those in the overall training set since the distribution of examples is different, which may impact the coherence measurement.}
\Cref{fig:chap1:supereasyfigure1} shows the coherence as measured by $\ralpha$ for 50K training samples of super easy and super hard examples during regular ImageNet training. As expected from our theory, the peak coherence of super easy examples is much higher than those of super hard examples.

\begin{figure*}[t]
\centering
\includegraphics[width=\textwidth]{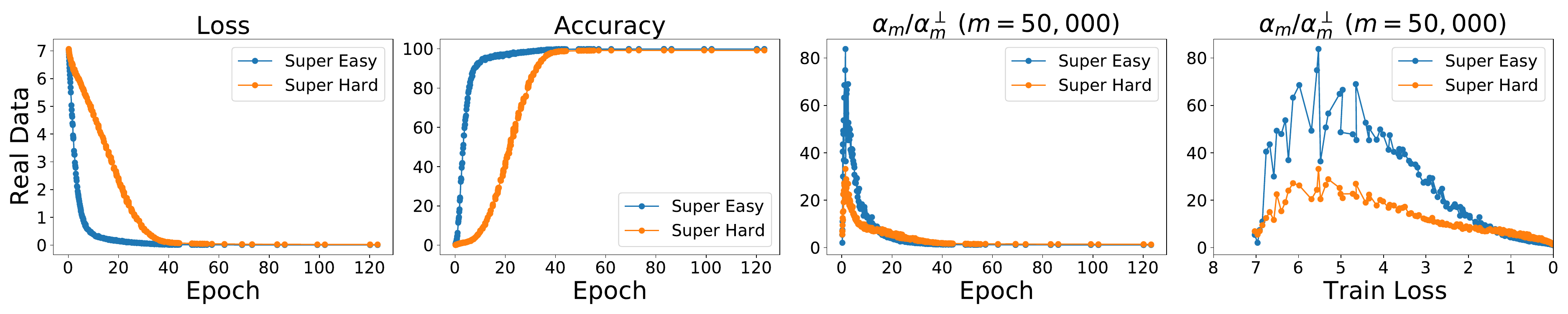}
\vspace{0.1cm}
\caption{Loss, accuracy, and $\ralpha$ on 50K samples of super easy and super hard examples measured during (normal) ImageNet training. As expected, the super easy examples have significantly higher coherence than super hard examples. The network used is a ResNet-50.}
\vspace{1cm}
\label{fig:chap1:supereasyfigure1}
\end{figure*}

\begin{figure*}[ht]
\centering
\includegraphics[width=\textwidth]{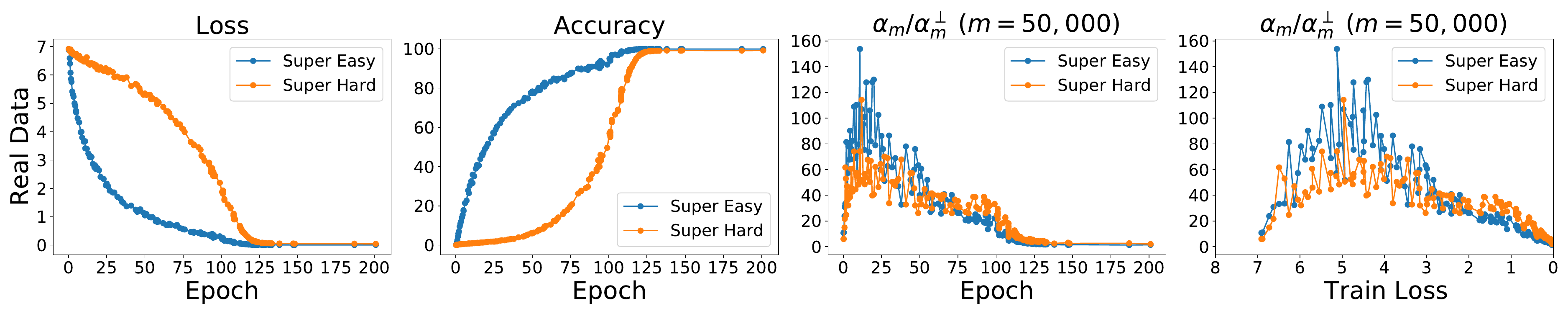}
\vspace{0.1cm}
\caption{Loss, accuracy, and $\ralpha$ on 50K samples of super easy and super hard examples measured during (normal) ImageNet training. This time the network being measured is an AlexNet although the hardness of the examples was measured on ResNet-50. The examples that are easy for ResNet-50 are also easy for AlexNet, suggesting that hardness is not very closely tied to the specific architecture used.}
\vspace{1cm}
\label{fig:chap1:supereasyfigure1alexnet}
\end{figure*}

Whether an ImageNet example is learned early in training or not does not appear to be architecture specific.
\Cref{fig:chap1:supereasyfigure1alexnet} tracks the loss, accuracy, and $\ralpha$ for (samples of) the super easy and super hard examples (from ResNet-50) during {\em AlexNet} training on ImageNet. 
We see that the super hard examples are learned later than the super easy examples by AlexNet as well, and the coherence of super easy examples is higher than that of the super hard examples. But the difference in $\ralpha$ is less stark for AlexNet than for ResNet-50. This mirrors the situation with real and random data analyzed in \Cref{sec:overview:measurement}, and a more granular per-layer plot of $\ralpha$ may reveal more significant per-layer differences.

\newbold{Do intrinsic easy patterns exist?}
Since AlexNet seems to find the same examples to be easy (or hard) as ResNet-50, one may wonder if there is something intrinsic to an example that controls its difficulty. 

To test this we train AlexNet on a training set containing either a single super easy example or a single super hard example, and count how many steps it takes for the training loss to fall below 0.05.%
\footnote{Since AlexNet does not use batch normalization, we can actually train with a single example.}
We repeat this for 500 super easy and 500 super hard examples picked at random.
Learning a single super hard example took on average 40.76 training steps%
\footnote{We use a very low learning rate of 0.00001 for this experiment.}
which was less than the average for a super easy example which was 43.12 steps! However, the difference between the two classes was not statistically significant.

This appears to suggest that:
\boldcenter{Hardness is not an intrinsic property of examples, but \underline{only} emerges via interaction with other examples in the training set.
}
This is a stronger statement than the one starting \Cref{sec:overview:easyhard} that  
``easy  examples  are  those  examples  that  have  a  lot  in  common  with  other examples where commonality is measured by the alignment of the gradients of the examples" since the latter does not exclude the additional possibility of something intrinsic to examples that controls their hardness.
The above experiment seems to rule out this possibility.

\section{The Under-Parameterized Case: A Preliminary Look}
\label{app:linear-regression}

\newbold{Coherence.} 
Although our theory has been motivated by the over-parameterized case, one may wonder what coherence looks like in an under-parameterized setting.%
\footnote{Note that under-parameterized systems may have good generalization, that is small gap between test and training loss, even if coherence is low, and this is qualitatively reflected in the $1/m$ dependence of the bound in \Cref{thm:cogen}.}
To that end, in this section, we take a look at coherence in the case of a 2D linear regression example.
Consider a sample of $m = 100$ points $(x, y)$ where $x$ is chosen uniformly at random from $[-2.0, 2.0]$, 
$\epsilon$ from $\mathcal{N}(0, \sigma^2)$,
and $y$ is given by
\[
y = 2 x + 3 + \epsilon.
\]
We fit the linear model $y = \beta_1 x + \beta_0$ to this sample using full-batch gradient descent and measure the alignment of per-example gradients $\ralpha$ on the sample:
\begin{center}
\phantom{M}\\
\includegraphics[width=0.9\textwidth]{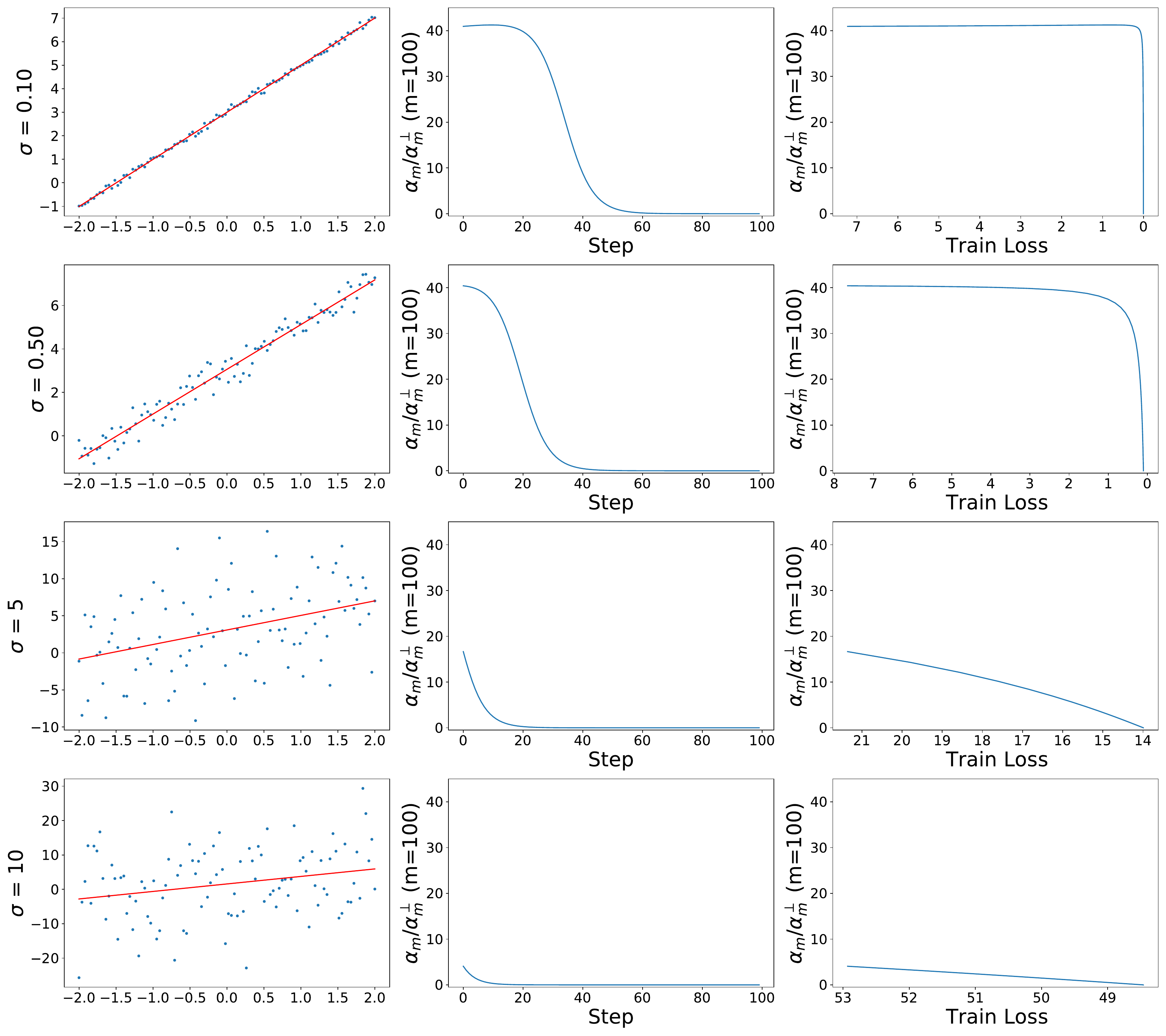}
\phantom{M}\\
\end{center}

Each row above corresponds to a different value of $\sigma$. The first panel shows a plot of the training sample and the line learned by gradient descent at the end of training. The second and third panel show $\ralpha$ as a function of the number of training steps taken and the training loss respectively.
Our main observation is that coherence depends on the amount of noise, and for low noise levels can be surprisingly high. 
For example, for $\sigma = 0.10$, each training example in our sample of 100 helps 40 other examples whereas for $\sigma = 10$, each example helps at most 5 examples.

\newbold{Suppressing Weak Directions.}
The techniques described in \Cref{sec:overview:suppress} can be useful in  under-parameterized situations as well. 
To illustrate this, we introduce outliers to our dataset by choosing 10 points at random from the points that have $x \ge 1.0$ and setting their corresponding $y$ coordinates to -1.0.
We train two models on this new dataset: one with gradient descent (GD) and another one with winsorized gradient descent (WGD) with $c = 10$. 
The following plots show the dataset and the regression line at the end of GD and WGD (the outliers are shown in orange):
\begin{center}
\phantom{M}\\    
\includegraphics[width=0.9\textwidth]{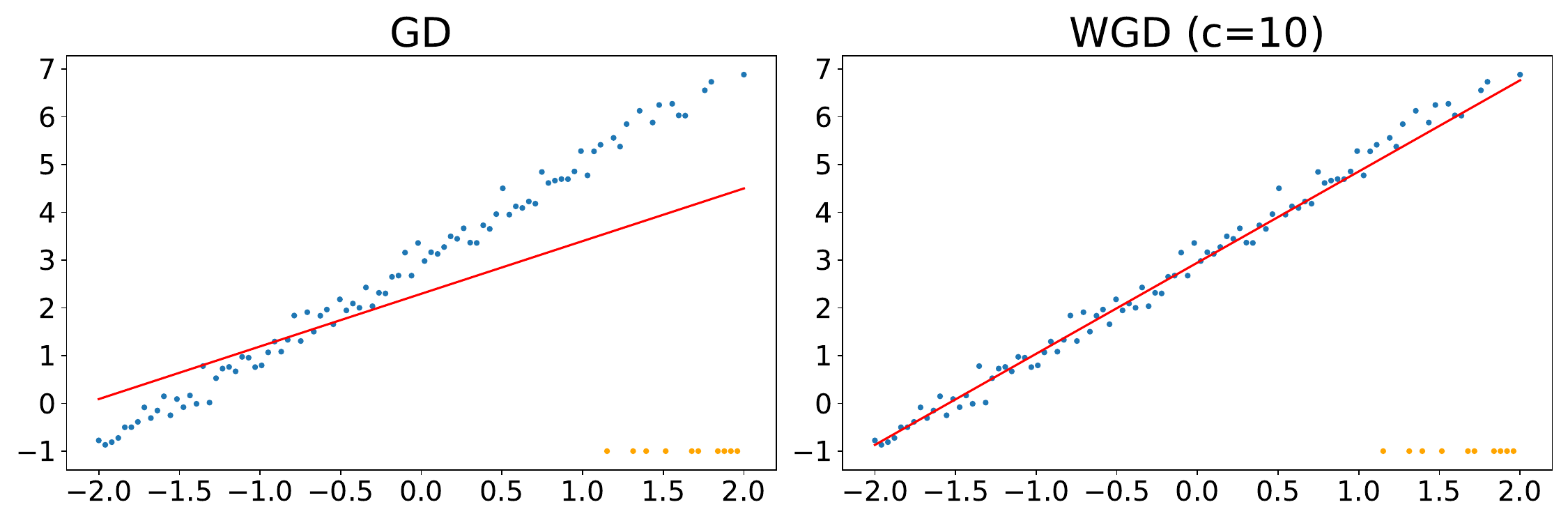}
\phantom{M}\\    
\end{center}
We see that the model trained with WGD effectively ignores the outliers, producing a solution better fitting the rest of the data. 

\section{Additional Data}
\label{app:additional-data}

This section collects additional plots for variations or multiple runs of experiments previously discussed.

\begin{figure}
\centering
\includegraphics[width=\textwidth]{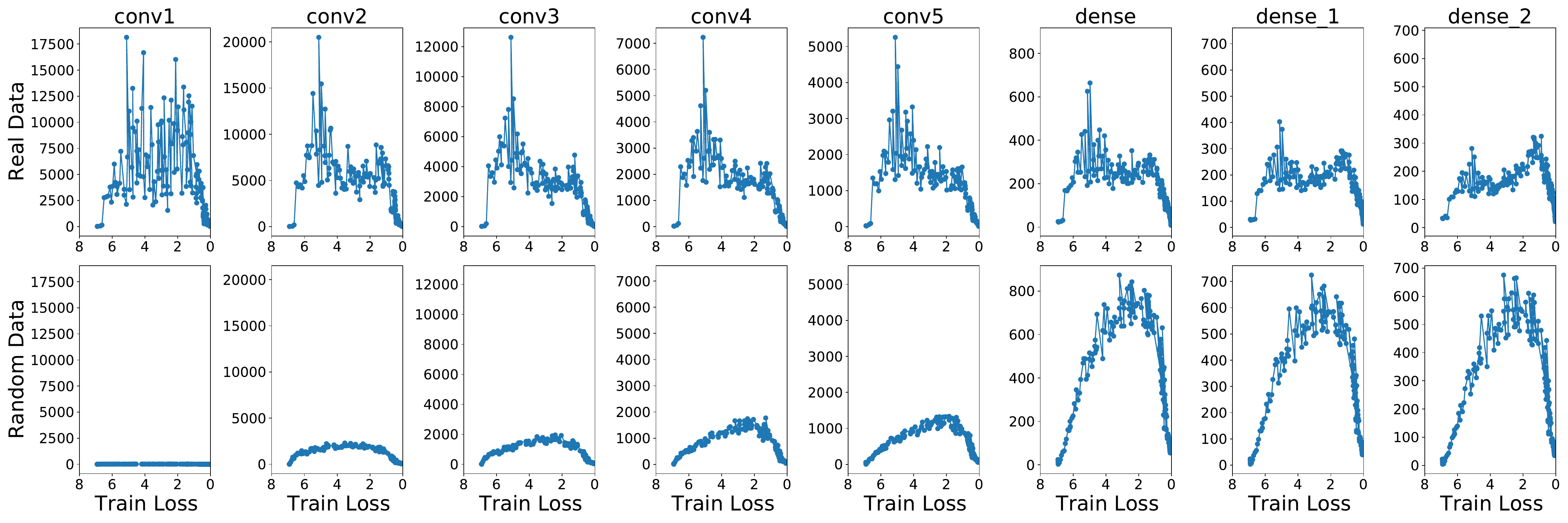}
\caption{\Cref{fig:chap1:alexnetimagenet_layer} but with $\ralpha$ measured on the {\em full} ImageNet training set ($m = $ 1,281,167) using the imputed method.}
\label{fig:alexnet_full_data_layers}
\end{figure}

\begin{figure}
\centering
\includegraphics[width=\textwidth]{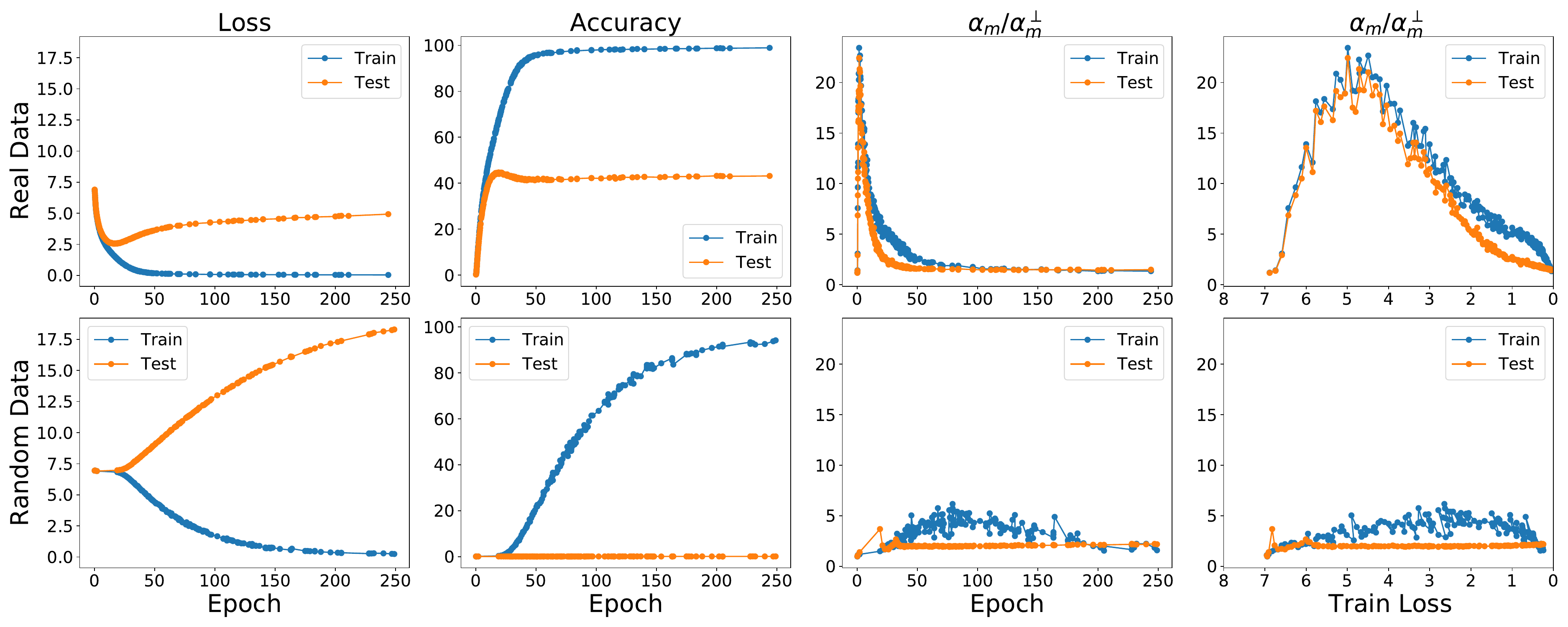}
\caption{Training setup identical with \Cref{fig:figure1_resnet50_imagenet_random_labels_a}, but minibatch gradients were normalized during training and (imputed) $\ralpha$ calculation to have the length of 1.0. Observe that $\ralpha$ on real data is still much higher than on random.}
\label{fig:ResNet-50_unit_gradient}
\end{figure}

\begin{figure}
\centering
\includegraphics[width=\textwidth]{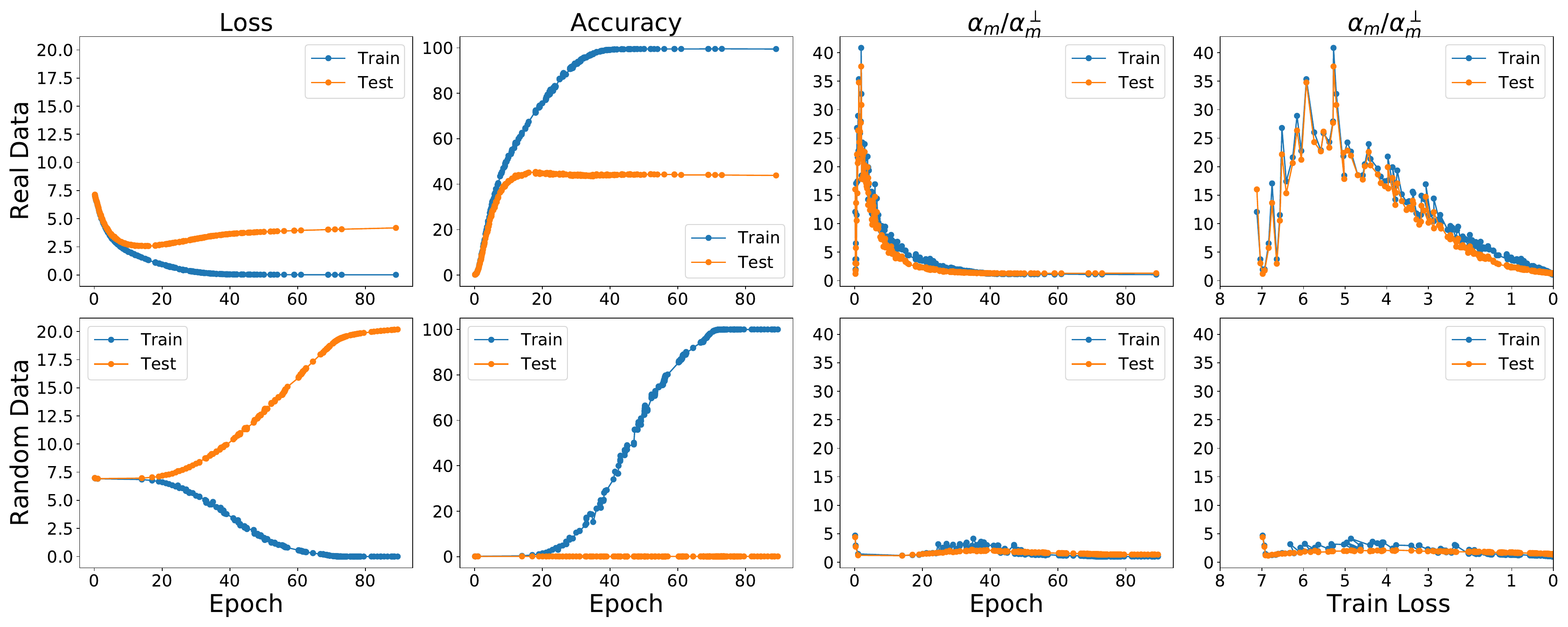}
(a)
\includegraphics[width=\textwidth]{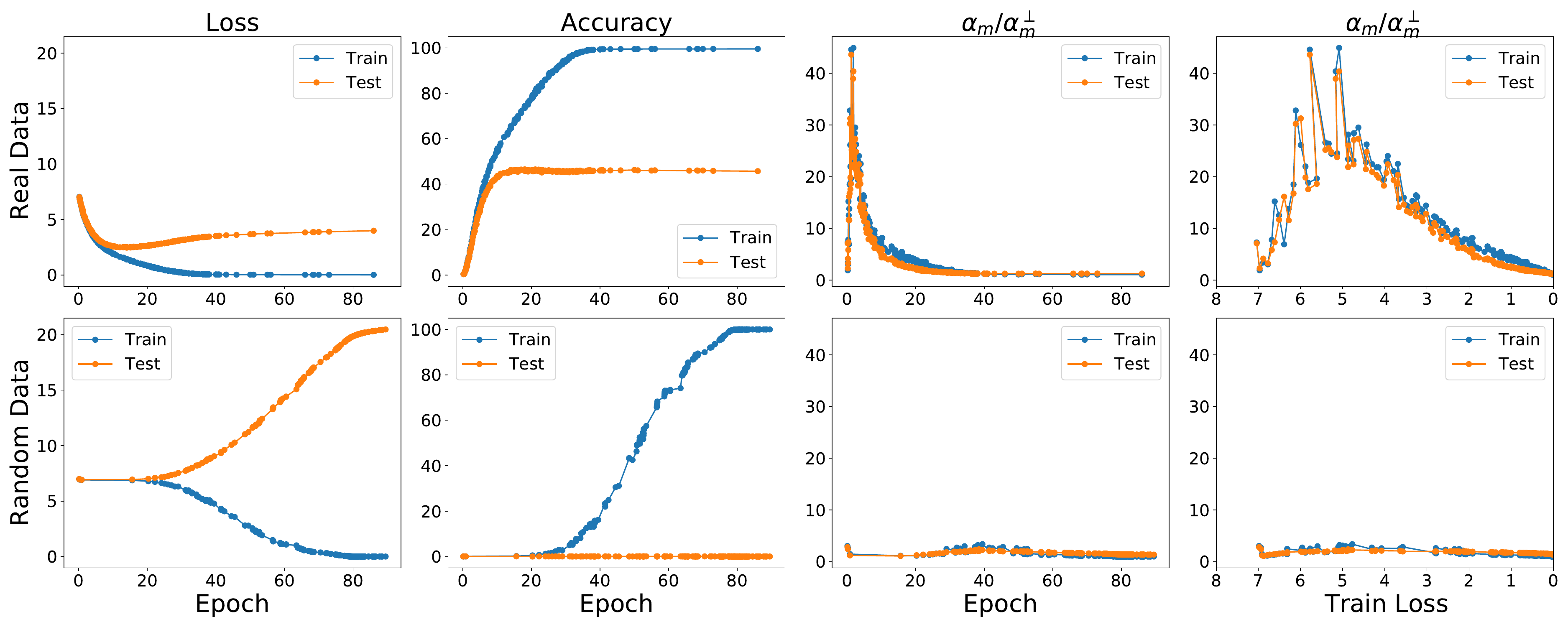}
(b)
\caption{Two additional runs of the experiment in \Cref{fig:chap1:ResNet-50imagenet}.}
\label{fig:figure1_resnet50_run_c}
\end{figure}

\begin{figure}
\centering
\includegraphics[width=\textwidth]{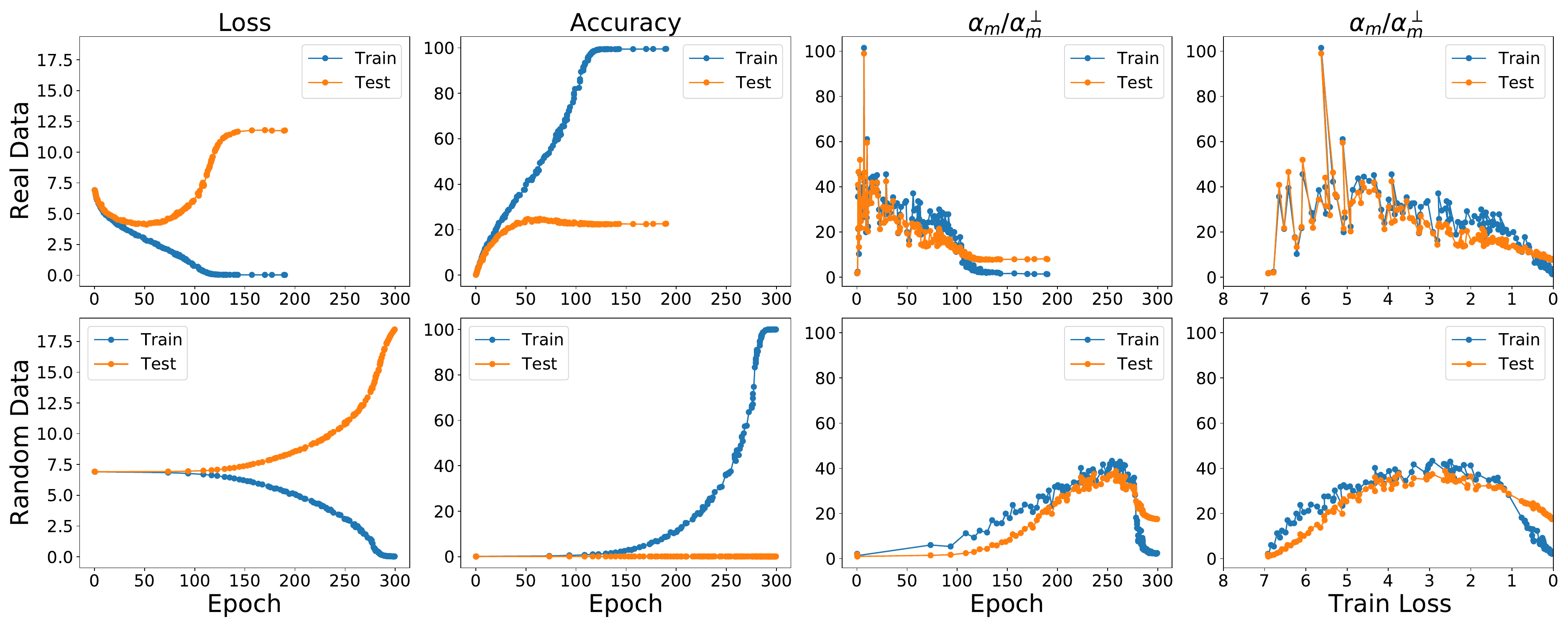}
(a)

\includegraphics[width=\textwidth]{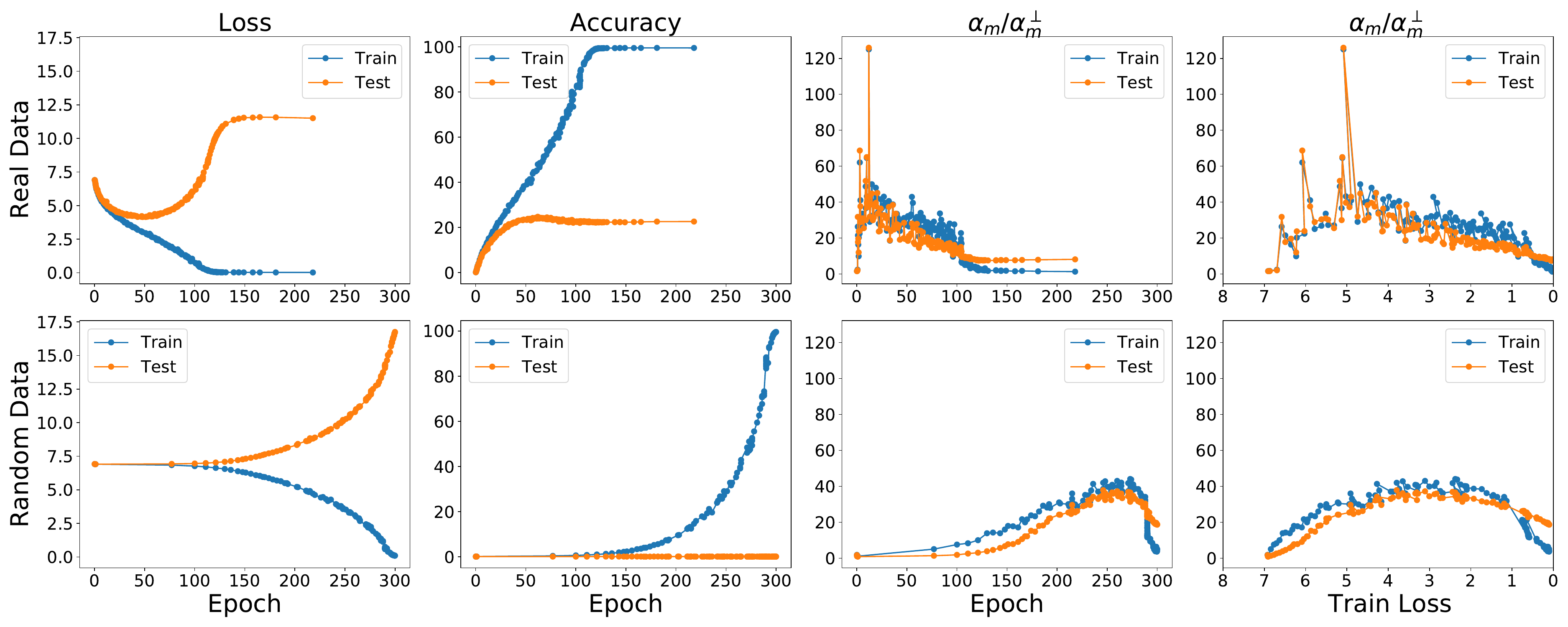}
(b)
\caption{Two additional runs of the experiment in \Cref{fig:chap1:alexnetimagenet_overall}.}
\label{fig:figure1_alexnet_run_c}
\end{figure}

\begin{figure*}
\centering
\includegraphics[width=\textwidth]{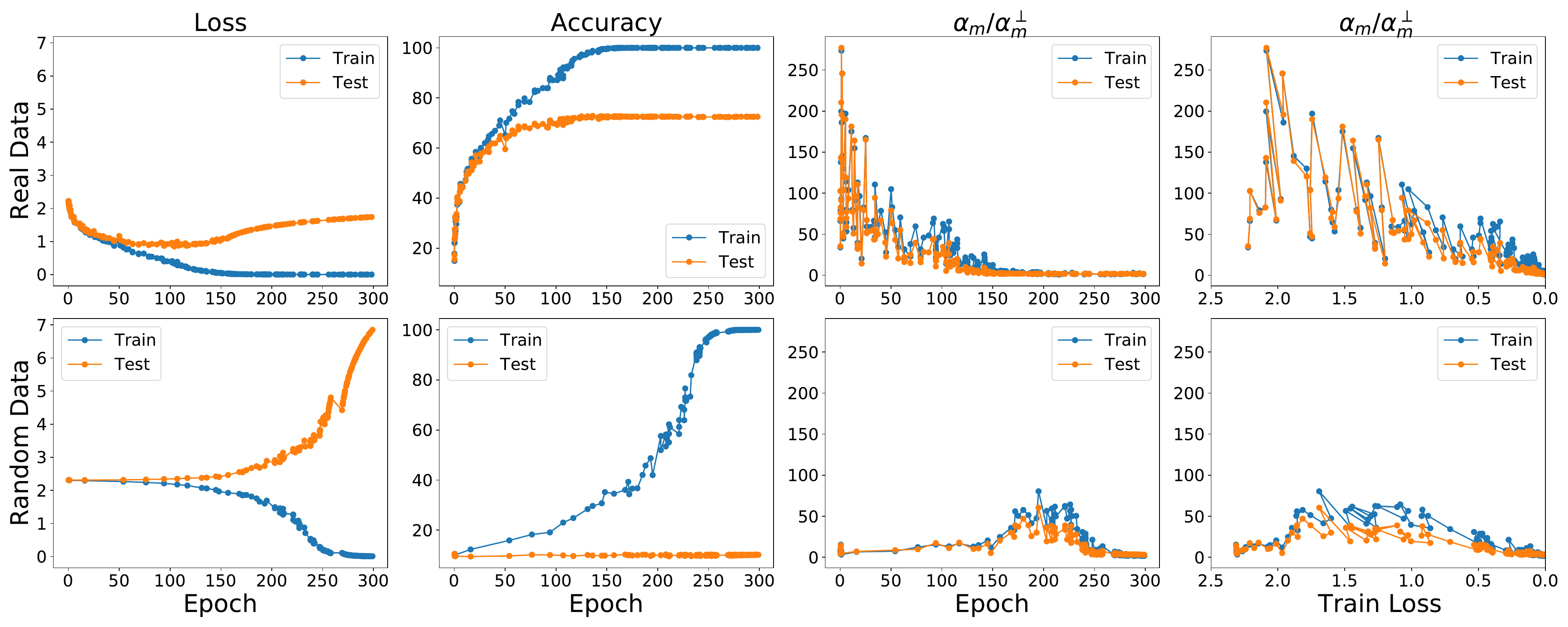}
(a)
\includegraphics[width=\textwidth]{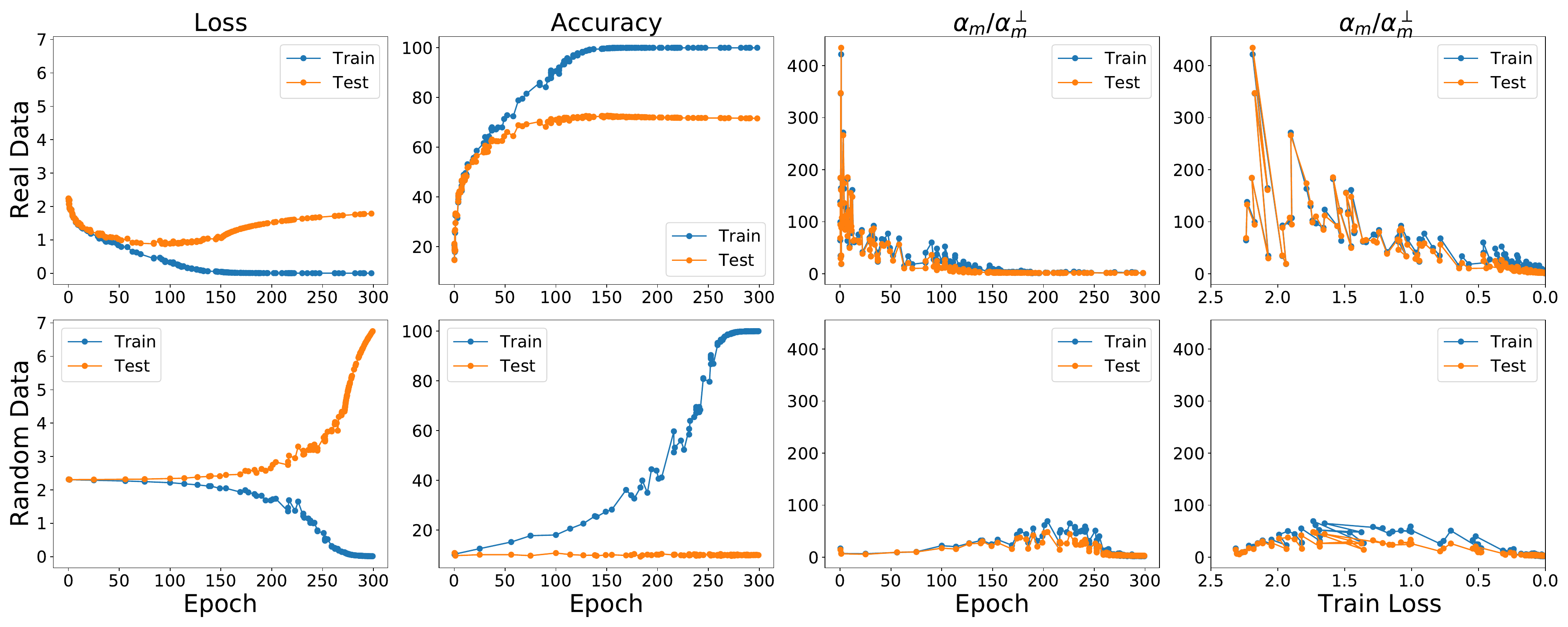}
(b)
\caption{Two additional runs of the experiment in \Cref{fig:figure1_cifar_alexnet_direct}.}
\label{fig:figure1_cifar_alexnet_direct_c}
\end{figure*}

\begin{figure*}
\centering
\includegraphics[width=\textwidth]{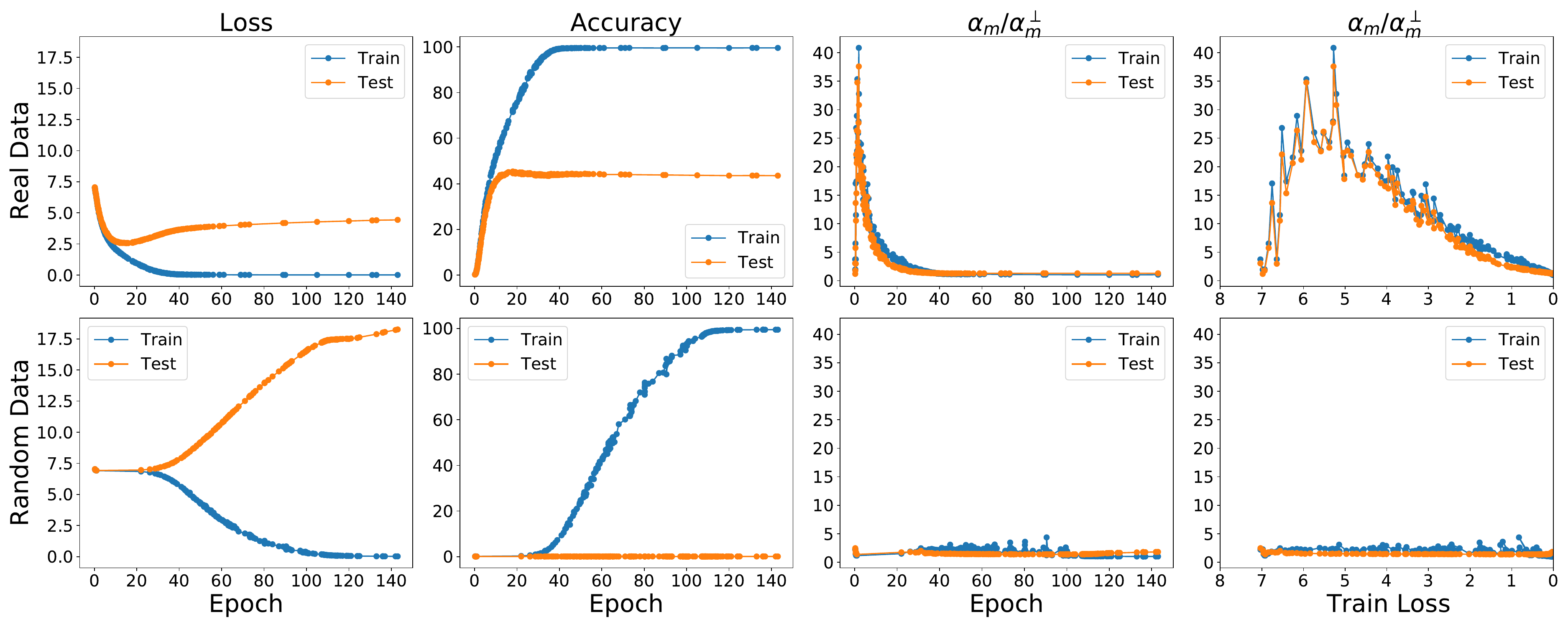}
(a)
\includegraphics[width=\textwidth]{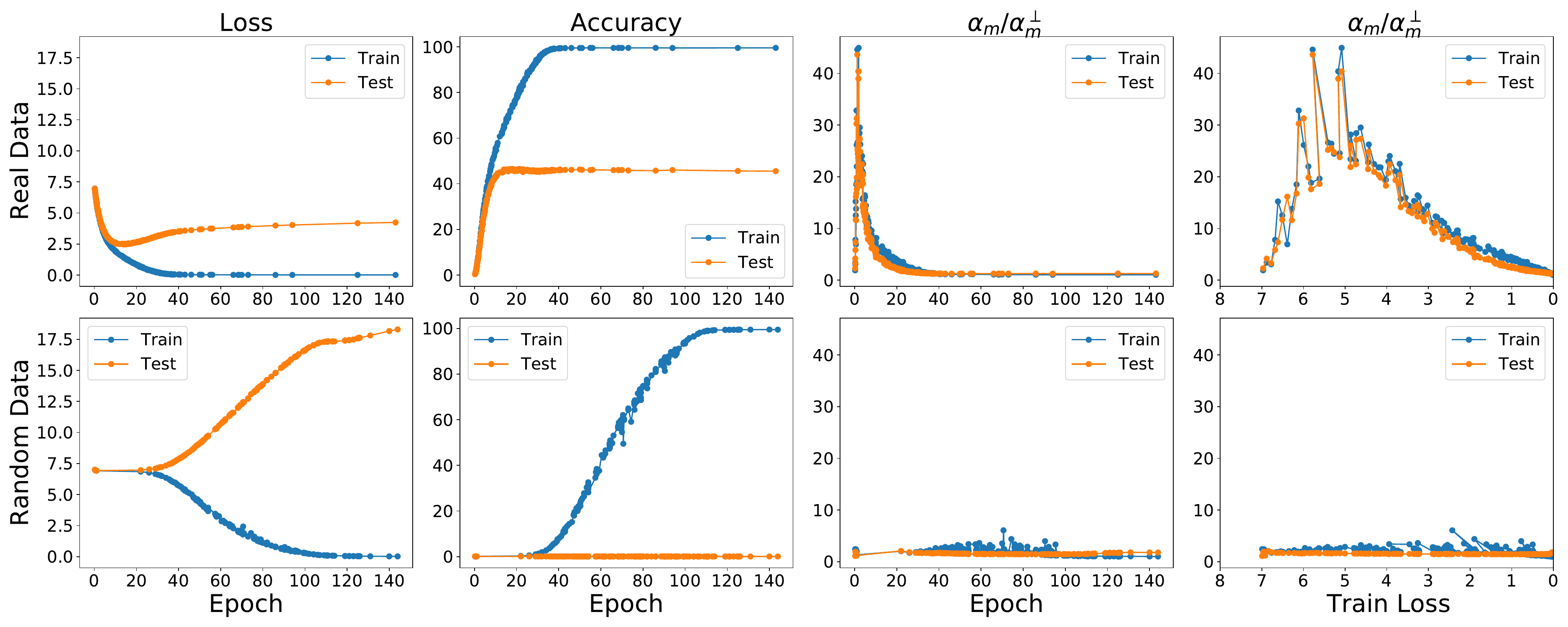}
(b)
\caption{Two additional runs of the experiment in \Cref{fig:figure1_resnet50_imagenet_random_labels_a}.}
\label{fig:figure1_resnet50_imagenet_random_labels_c}
\end{figure*}

\begin{figure*}
\centering
\includegraphics[width=\textwidth]{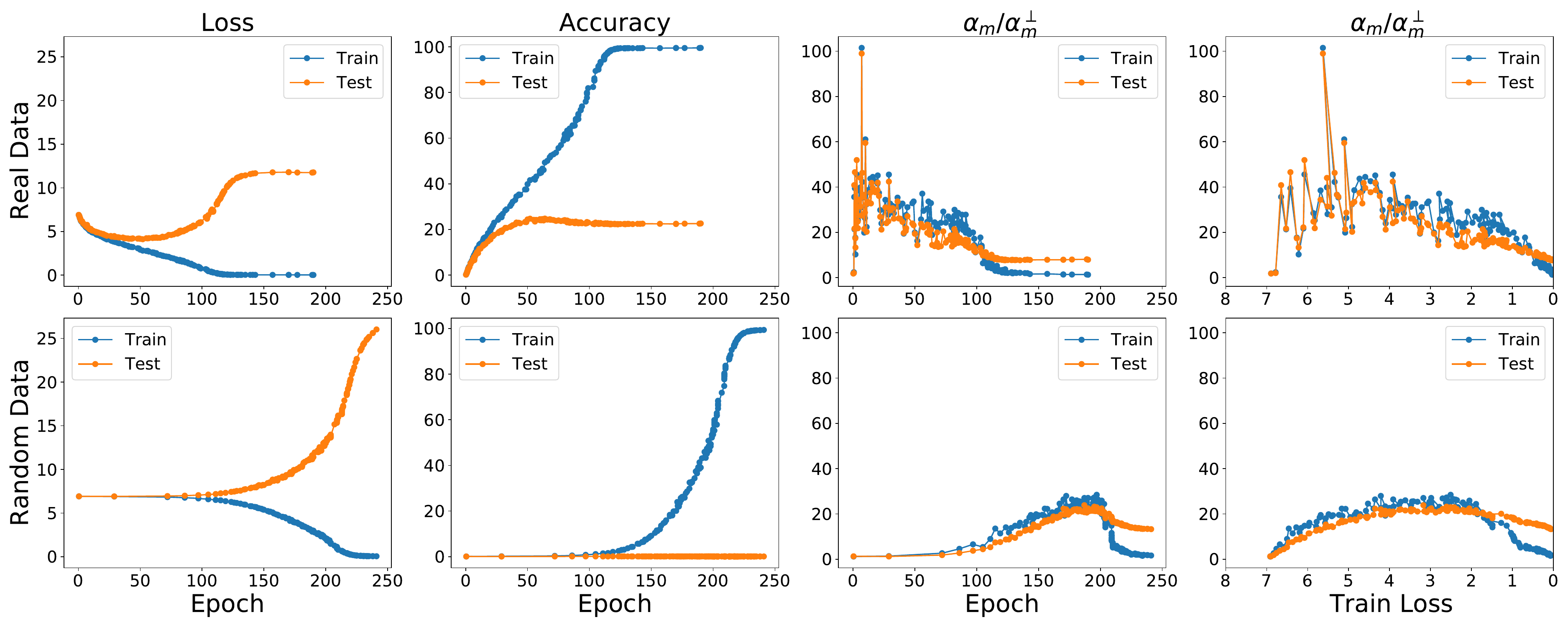}
(a)
\includegraphics[width=\textwidth]{plots/additional_runs/figure1_resnet50_imagenet_random_labels_c.pdf}
(b)
\caption{Two additional runs of the experiment in \Cref{fig:figure1_alexnet_imagenet_random_labels_a}.}
\label{fig:figure1_alexnet_imagenet_random_labels_c}
\end{figure*}

\begin{figure*}
\centering
\includegraphics[width=\textwidth]{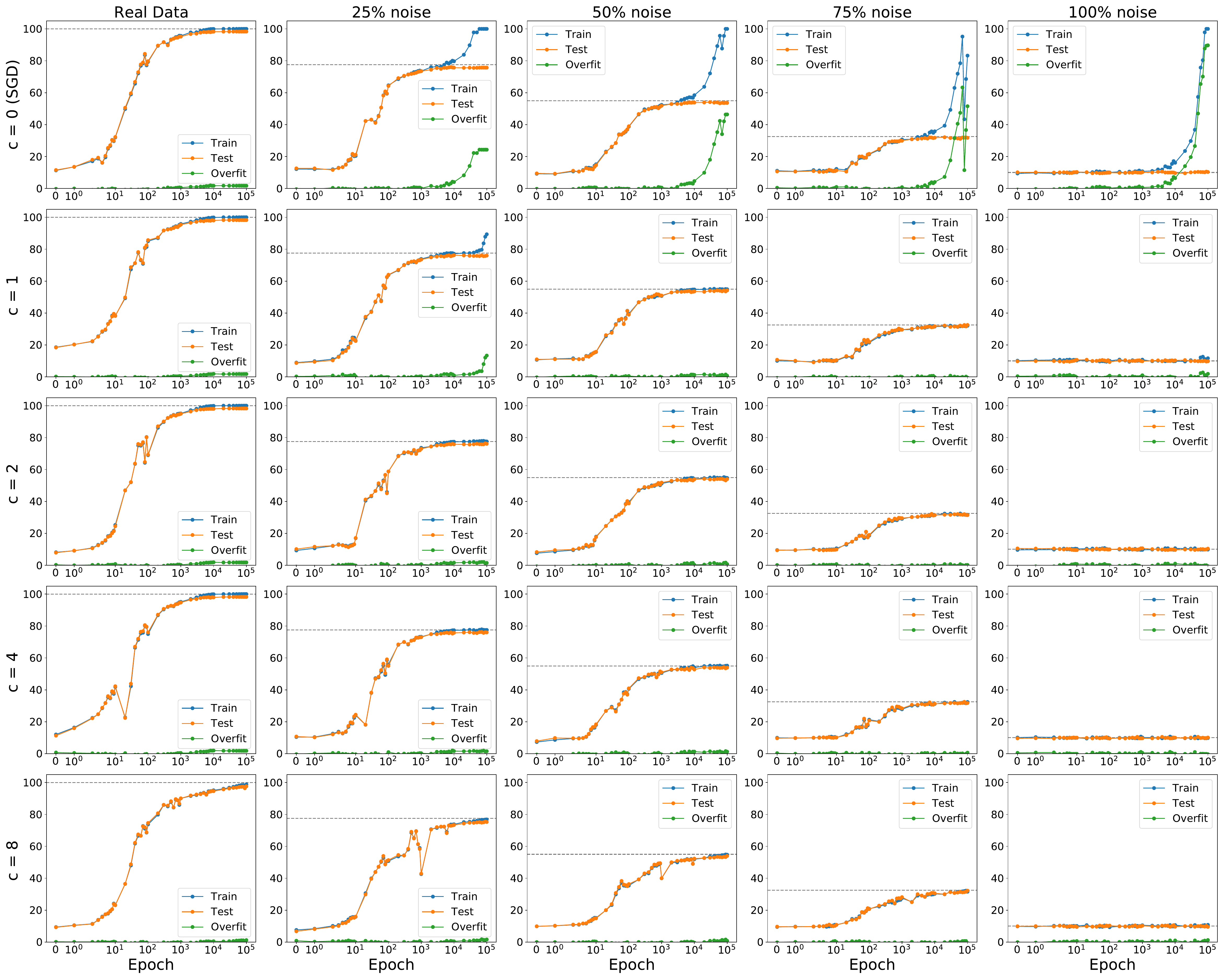}
\caption{Winsorization on {\sc mnist} with random pixels. Each column represents a dataset with different noise level, e.g. the third column shows dataset with half of the examples replaced with Gaussian noise. See \Cref{fig:chap1:iclrwinsorizationmnist} for experiments with random labels.}
\vspace{1cm}
\label{fig:overview:winsorization_gaussian}
\end{figure*}

\begin{figure*}
\centering
\includegraphics[width=\textwidth]{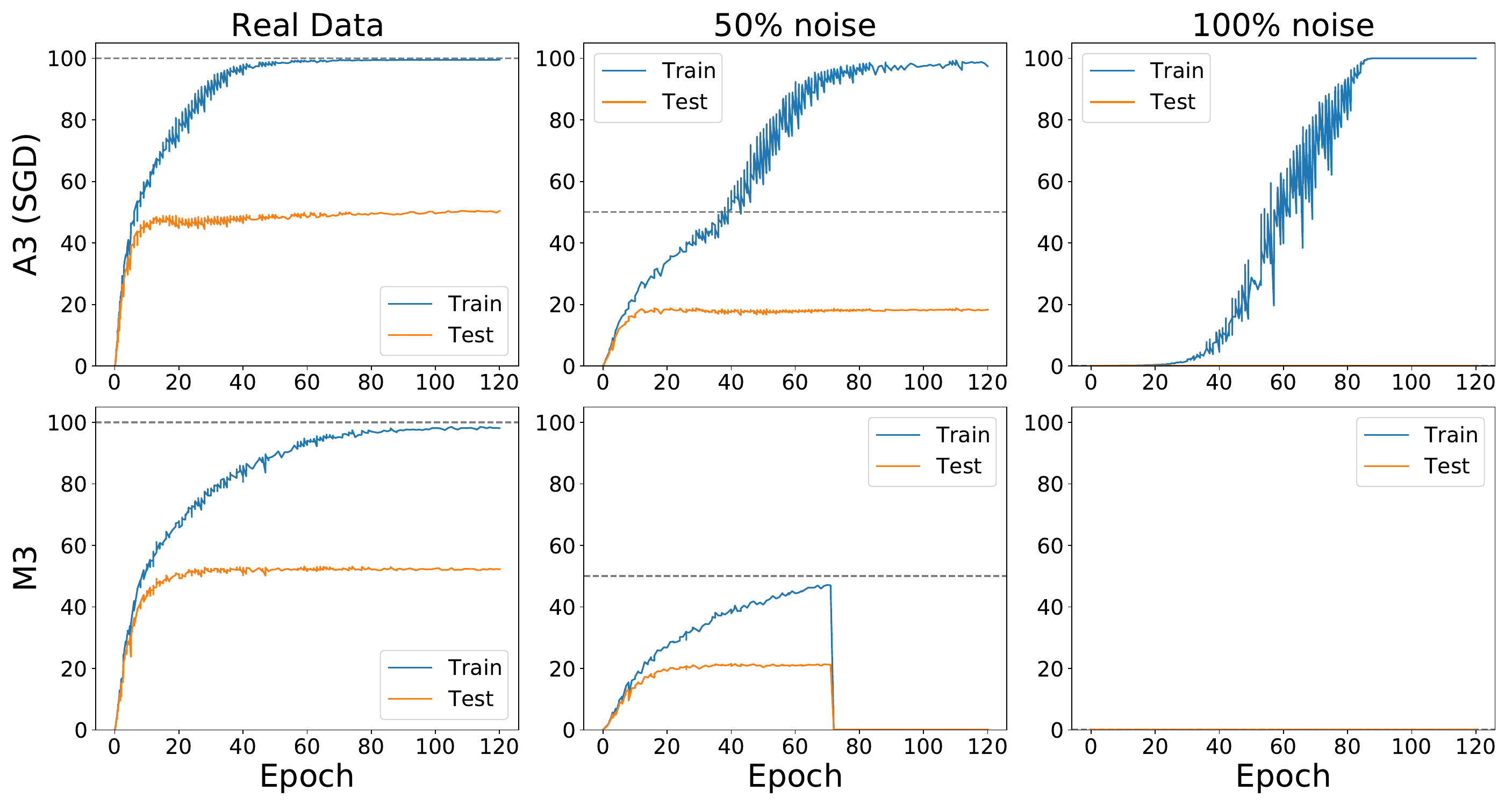}
\caption{Unstable M3 in 50\% noise case, learning rate of 0.1. See \Cref{fig:overview:resnet_imagenet_a3_vs_m3} for a stable run with a different learning rate.}
\vspace{1cm}
\label{fig:overview:resnet_imagenet_a3_vs_m3_unstable}
\end{figure*}

\begin{figure*}
\centering
\includegraphics[width=\textwidth]{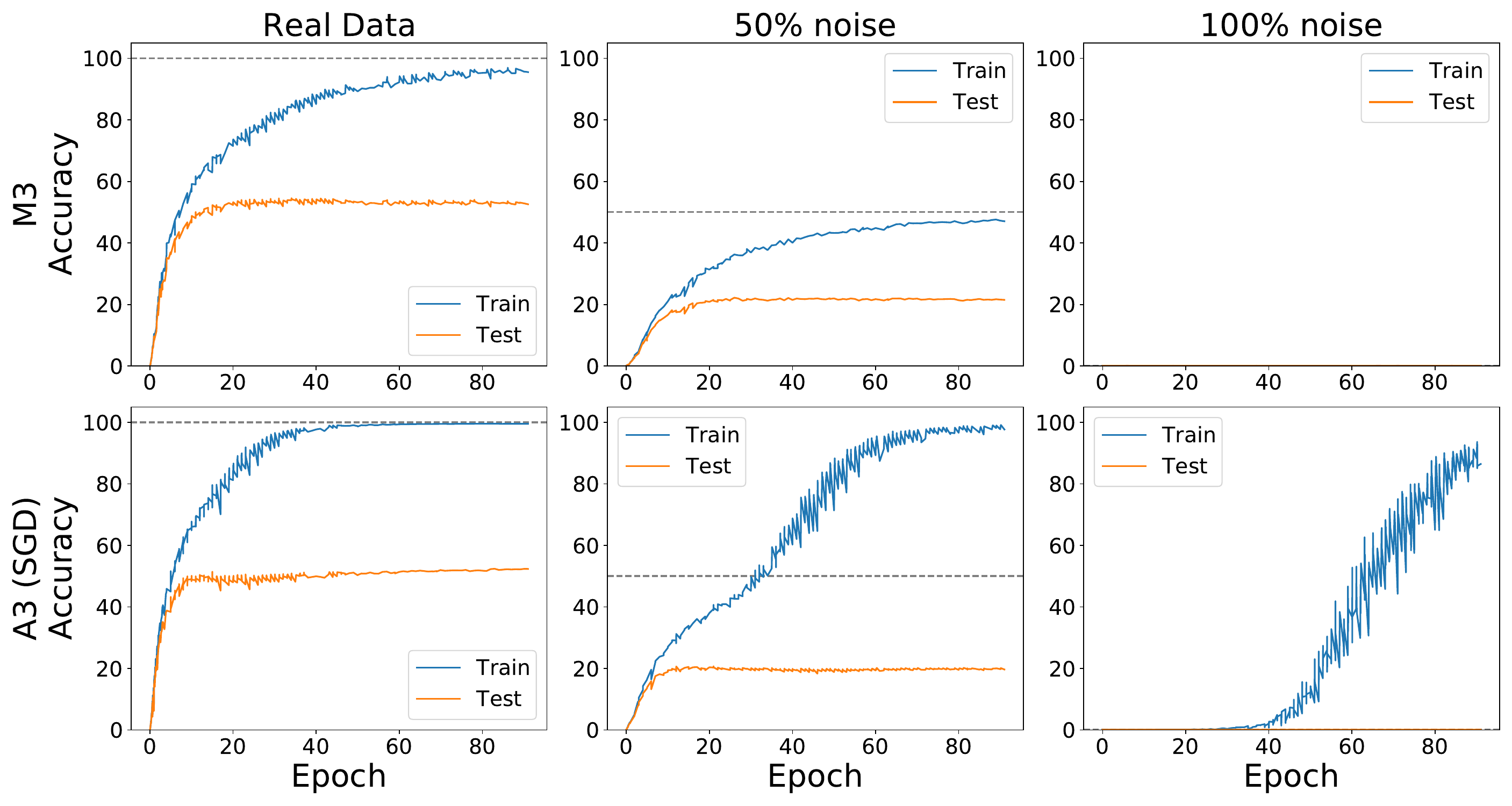}
(a)
\includegraphics[width=\textwidth]{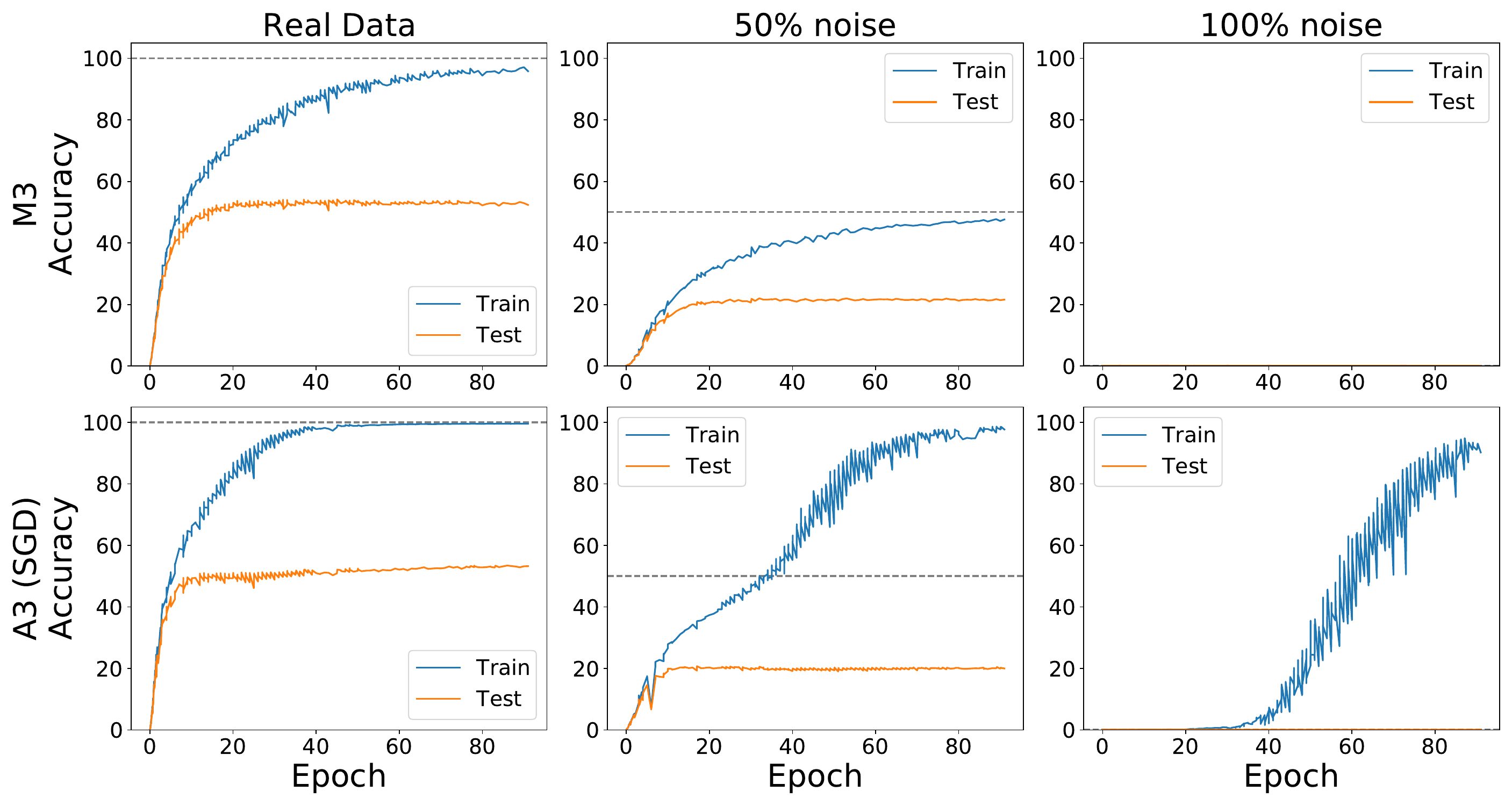}
(b)
\caption{Two additional runs of the experiment in \Cref{fig:overview:resnet_imagenet_a3_vs_m3}.}
\label{fig:a3_vs_m3_run_c}
\end{figure*}

\begin{figure*}
\centering
\includegraphics[width=\textwidth]{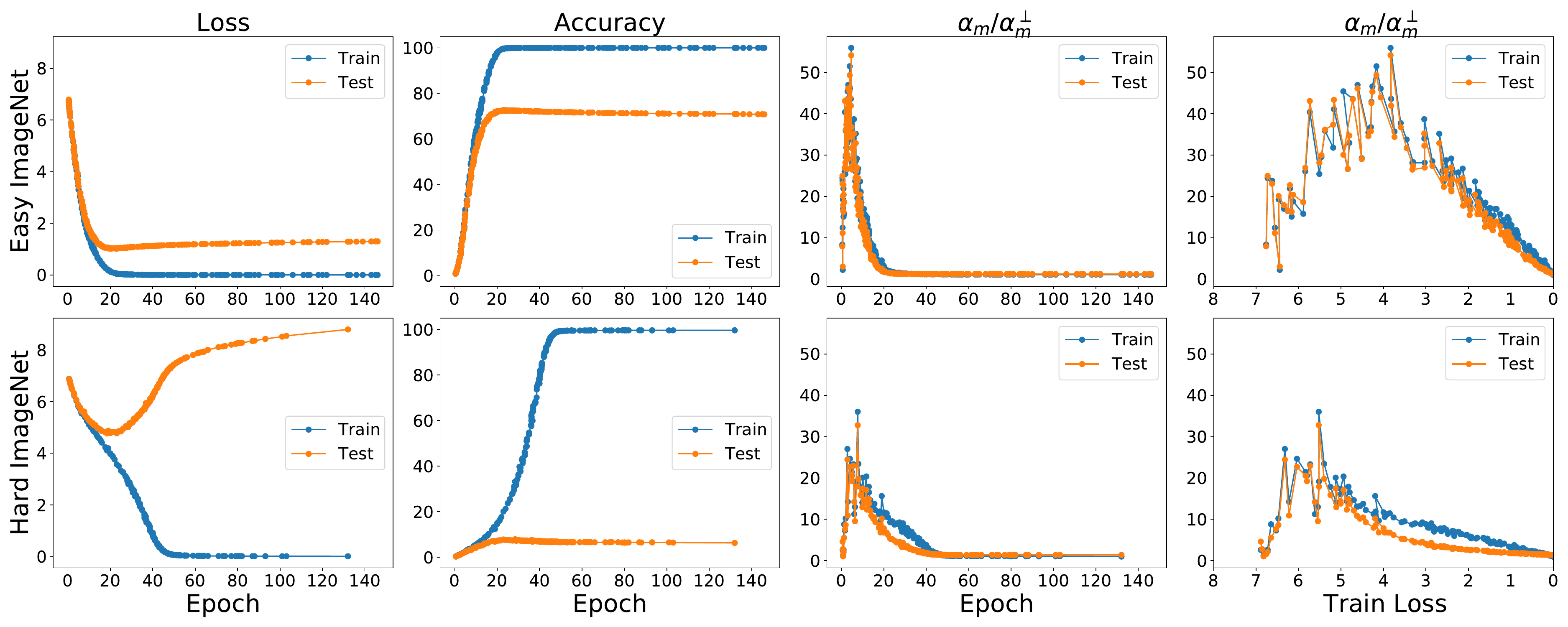}
(a)
\includegraphics[width=\textwidth]{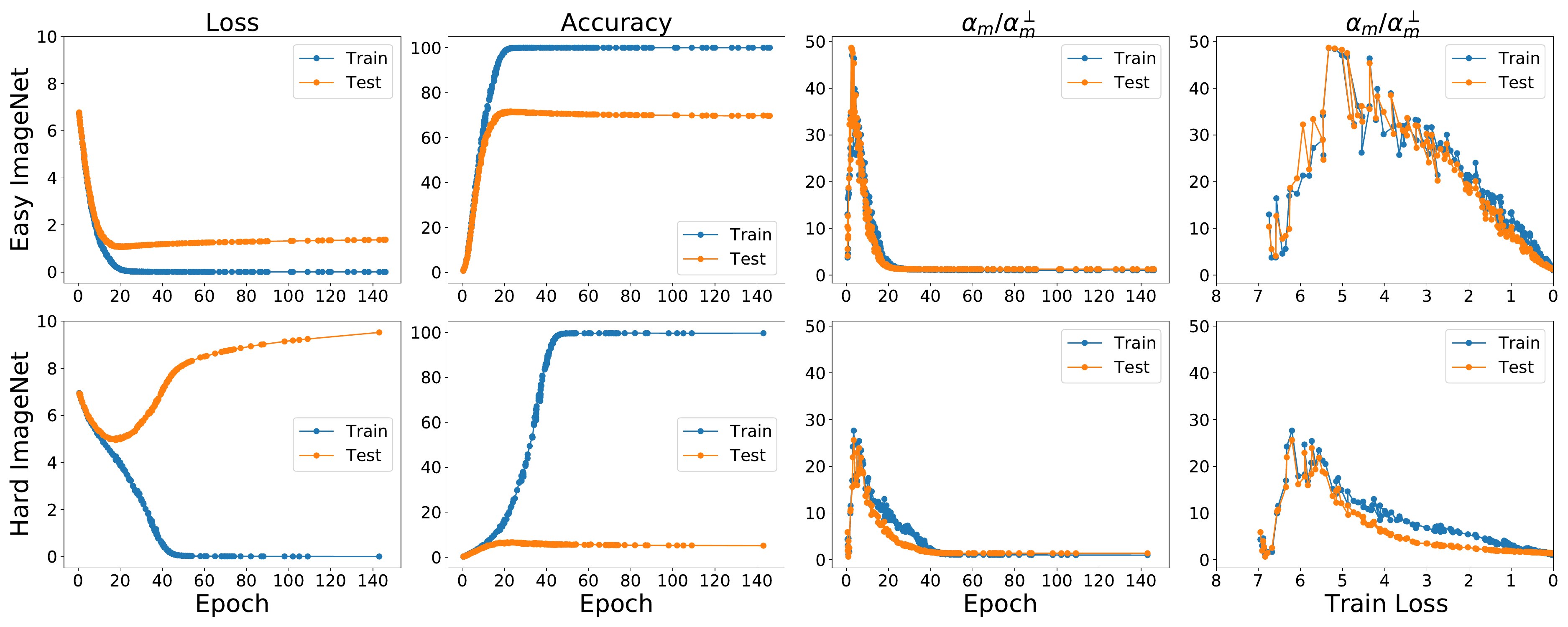}
(b)
\caption{Two additional runs of the experiment in \Cref{fig:chap1:easyhardfigure1}.}
\label{fig:easy_imagenet_hard_imagenet_c}
\end{figure*}










\end{appendices}

\addtocontents{toc}{\phantom{M}}

\bibliographystyle{plainnat} 
\bibliography{paper}
 
\end{document}